\documentclass[letterpaper,journal]{IEEEtran}
\usepackage[url=false]{biblatex}
\usepackage{algorithmic}
\usepackage{algorithm}
\usepackage{array}
\usepackage[caption=false,font=footnotesize]{subfig}
\usepackage{textcomp}
\usepackage{stfloats}
\usepackage{url}
\usepackage{verbatim}
\usepackage{graphicx}
\usepackage{capt-of}
\usepackage{amsmath,amsthm,amssymb, amsfonts}
\usepackage{bm}
\usepackage{xcolor}
\usepackage{booktabs,tabularx}
\usepackage{enumitem}
\usepackage{multirow}
\usepackage[most]{tcolorbox}
\usepackage[hidelinks]{hyperref}
\newcolumntype{L}{>{\raggedright\arraybackslash}p{0.16\textwidth}}
\newcolumntype{C}{>{\raggedright\arraybackslash}X}
\newcolumntype{R}{>{\raggedright\arraybackslash}p{0.16\textwidth}}

\newtheorem{theorem}{Theorem}[section]
\newtheorem{corollary}{Corollary}[section]
\newtheorem{lemma}[theorem]{Lemma}
\theoremstyle{definition}
\newtheorem{definition}{Definition}[section]
\theoremstyle{remark}
\newtheorem{remark}{Remark}
\theoremstyle{definition}
\newtheorem{example}{Example}
\newcommand{\examplecont}{\addtocounter{example}{-1}}
\tcolorboxenvironment{example}{
    enhanced,
    breakable,
    colback=blue!4,
    colframe=blue!55!black,
    boxrule=0pt,
    leftrule=2pt,
    sharp corners,
    boxsep=1pt,
    left=6pt,
    right=6pt,
    top=2pt,
    bottom=2pt,
    before skip=4pt,
    after skip=4pt,
}

\makeatletter
\renewcommand\subsection{\@startsection{subsection}{2}{\z@}%
  {1.5ex plus 0.5ex minus 0.5ex}%
  {0.4ex plus 0.2ex minus 0.1ex}%
  {\normalfont\normalsize\itshape}}
\makeatother
\DeclareMathOperator*{\argmin}{arg\,min}

\newtheorem{assumption}{Assumption}
\newtheorem{proposition}{Proposition}

\newcommand{\methodname}{\texttt{SafePBDS}}
\newcommand{\methodnamelong}{Safe Pullback Bundle Dynamical Systems}

\begin{document}

\title{Safe and Steerable Geometric Motion Policies for\\ Robotic Dexterous Manipulation}

\author{Albert Wu$^{1}$, Riccardo Bonalli$^{2}$, Thomas Lew$^{3}$, C. Karen Liu$^{1}$%
\thanks{\raggedright$^{1}$Computer Science Department, Stanford University, Stanford, CA 94305, USA.
\texttt{\{amhwu,~karenliu\}@cs.stanford.edu}}%
\thanks{\raggedright$^{2}$Laboratory of Signals and Systems, University of Paris-Saclay,
CNRS, CentraleSupélec, France. \texttt{riccardo.bonalli@l2s.centralesupelec.fr}}%
\thanks{\raggedright$^{3}$Toyota Research Institute, Los Altos, CA 94022, USA. \texttt{thomas.lew@tri.global}}%
}
\maketitle

\begin{abstract}
Robotic dexterous manipulation requires continuously reconciling objectives and constraints defined on heterogeneous geometric spaces: a robot controlled on a $\mathbb{R}^{7}$ configuration manifold may need to track end effector poses on $\mathrm{SE}(3)$ while satisfying obstacle avoidance margins in $\mathbb{R}$. We present \methodnamelong{} (\methodname{}), a geometrically consistent motion generation framework that computes optimal, certifiably safe configuration manifold accelerations from objectives and safety requirements defined on arbitrary task manifolds. \methodname{} builds on prior work that combines predefined task manifold dynamical systems to produce autonomous motion. Its first innovation is a pullback control barrier function construction, which converts task manifold safety conditions into linear constraints on configuration manifold accelerations. The second innovation is a task manifold action interface that allows a high-level policy to inject low dimensional residual motions; zero input recovers the autonomous behavior, while safety is preserved under arbitrary inputs. This enables high-level policies to efficiently steer exploration while leaving precise motion generation to the autonomous behavior. We validate \methodname{} in simulation and on a 23-DOF Franka Panda--Allegro Hand platform. On dexterous grasping, \methodname{} achieves a $92.5\%$ success rate across 20 household objects and 120 trials. By leveraging the action interface, the method can exclude any one of the four fingers during grasping using a one-dimensional action, achieving $94.4\%$ 3-finger grasp success across 3 objects and 36 trials. The efficient planning and safety guarantee of \methodname{} also enables the first model-based, fully actuated palm-down in-hand reorientation, exceeding $360^\circ$ of yaw rotation in both directions under varying object weight and wrist motion. A demo video and additional details are available at \href{https://tml.stanford.edu/safe-pbds}{\texttt{tml.stanford.edu/safe-pbds}}.
\end{abstract}


\section{Introduction}

Robotic dexterous manipulation requires reasoning over quantities defined on different geometric spaces, including positions in Cartesian space, orientations on $\mathrm{SO}(3)$, and joint limits in generalized coordinates. Some of these quantities correspond to soft objectives, such as null-space redundancy resolution and task prioritization, while others impose hard constraints that must never be violated, such as collision avoidance, closed-loop kinematic constraints, and force closure. Moreover, real-world perturbations and modeling errors require the system to respond quickly via feedback during execution. Successful manipulation therefore demands continuously resolving objectives and constraints defined on heterogeneous geometric spaces in real time.

\begin{figure}[!t]
    \centering
    \includegraphics[width=\columnwidth]{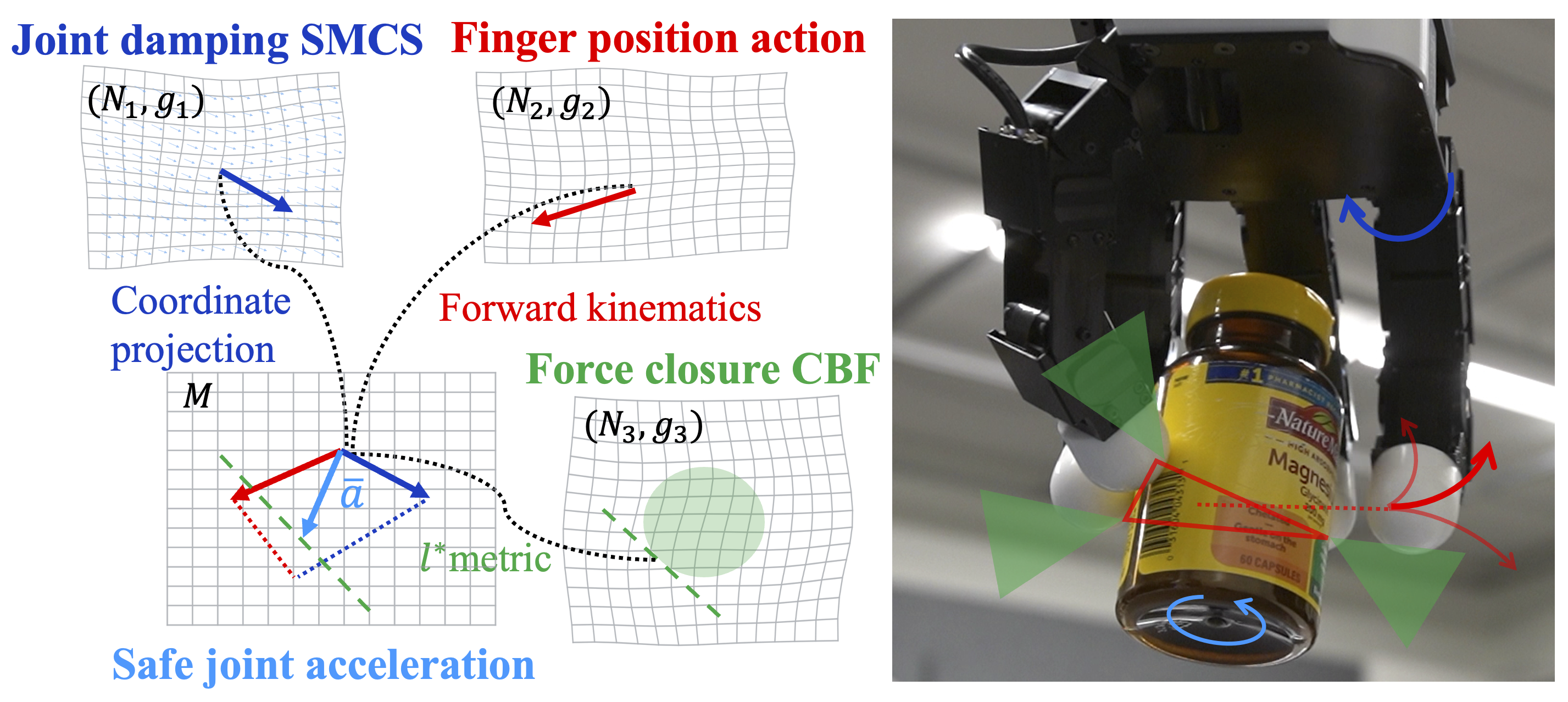}
    \caption{Overview of \methodname{}. At each control step, tasks on heterogeneous manifolds $(N_i,g_i)$ are pulled back to the configuration manifold~$M$ and composed into a safe acceleration~$\bar{a}$ via quadratic programs. (Left) Representative tasks include an autonomous joint-damping SMCS (blue), a high-level finger-position action (red), and a force-closure CBF safety constraint (green), each connected to~$M$ by a task map. (Right) On a 16-DOF Allegro Hand, these components combine to achieve palm-down in-hand reorientation (Section~\ref{sec:in_hand_reorientation}).}
    \label{fig:headline}
\end{figure}

Contemporary data driven methods attempt to sidestep these challenges by learning end-to-end policies that implicitly encode all objectives and constraints. Deep reinforcement learning (DRL)~\cite{andrychowicz2020learning,handa2023dextreme,chen2023visual,qi2023general} and behavior cloning~\cite{zhu2023viola,chi2023diffusion,zhao2023learning} can capture globally coordinated behaviors and are often straightforward to deploy. However, the resulting policies often remain specialized to a single task or embodiment~\cite{weinberg2024survey}. Vision-language-action models (VLAs)~\cite{geminirobotics2025,physicalintelligence2026pi07,tri2025lbm} have shown impressive generalization on arm-gripper manipulation, but their extension to multi-fingered hands remains limited by the complexity of dexterous motion and the difficulty of collecting multi-fingered manipulation data. 
Additionally, all aforementioned black-box policies can fail catastrophically outside the training distribution, motivating the need to incorporate certifiable structure for safety guarantees.

At the other end of the spectrum, trajectory optimization~\cite{posa2014direct,shirai2022robust,pang2023global,morgan2022complex} can explicitly reason about constraints and dynamics to produce certifiable motion plans. However, solving these optimizations is computationally expensive, prone to getting stuck in local minima, require specific engineering tricks such as fixed contact pairs, and the resulting plans are essentially one-shot. Adapting them to online perturbations requires expensive replanning at rates that are typically incompatible with high-frequency dexterous manipulation control. 

A third family of approaches is the \emph{vector field policy}, which maps each state to a desired velocity or acceleration through a globally defined vector field~\cite{ratliff2018riemannian,bylard2021composable}. Because these policies are defined analytically, they enable fast reactive behavior, and properties such as stability can often be certified by construction. Several frameworks compose per-task vector fields using Riemannian geometry~\cite{ratliff2018riemannian,cheng2021rmpflow,ratliff2020optimization,van2022geometric,bylard2021composable}, but existing formulations either lack geometric consistency or treat safety as a soft property~\cite{ratliff2020optimization,van2022geometric}. Moreover, because most vector field policies are inherently local, they can become trapped in local minima. A recent example, DextrAH-G~\cite{lum2024dextrahg}, combines vector field policies with learning-based methods to mitigate this lack of global planning, but the resulting policies remain limited by the vector field formulation.

Among vector field policies, Pullback Bundle Dynamical Systems (PBDS)~\cite{bylard2021composable} stands out for achieving geometric consistency. By building on the coordinate invariant simple mechanical control system (SMCS), PBDS provides a fully geometric formulation that works on arbitrary Riemannian manifolds. However, PBDS has two important limitations. First, \emph{safety constraints are soft}: metric-based constraints enter the optimization as additional weighted objectives and can therefore be violated when they conflict with other tasks; moreover, these objectives may become ill-defined in unsafe regions. Second, \emph{there is no action interface}: the system provides no mechanism for a higher-level policy to inject task manifold commands, limiting its use to autonomous policies fixed at design time.

We present \methodname{} (\methodnamelong{}), a framework that addresses both limitations of PBDS while preserving the geometric and compositional properties. \methodname{} extends PBDS with hard safety guarantees and external controllability while retaining its geometric consistency. At each control step, a constrained quadratic program takes in real-time observations and composes desired autonomous behaviors, safety constraints, and action inputs into a single acceleration target. This is enabled by the following:
\begin{itemize}[leftmargin=*, noitemsep, topsep=2pt]
    \item \emph{Pullback control barrier functions (CBFs)} that enforce task manifold safety as hard constraints on the configuration manifold acceleration by pulling back the constraints through smooth task maps. We derive formulations for two prominent higher-order CBF variants, exponential (ECBF) and backstepping (BCBF), and show their respective dependencies on task manifold geometries.
    \item A \emph{task manifold action interface} that lets a higher-level policy inject inputs on selected task manifolds. Our formulation guarantees that zero input recovers the autonomous behavior and that safety is preserved under arbitrary inputs.
\end{itemize}
We also provide extensive evaluations of \methodname{}:
\begin{itemize}[leftmargin=*, noitemsep, topsep=2pt]
    \item \emph{Simulation experiments} on an $\mathbb{S}^2$ double integrator and a 7-DOF robot arm validate the theoretical properties of our framework: chart invariance, task manifold metric effects, recovery from unsafe states, and the action interface.
    \item \emph{Hardware experiments} on a 23-DOF arm--hand system (Franka Panda + Allegro Hand) demonstrate
        (a)~autonomous dexterous grasping of 20 household objects at a 92.5\% success rate (111/120 trials), including a 3-finger ablation at 94.4\% (34/36); and
        (b)~robust in-hand reorientation under palm-down and variable wrist orientations, achieving over $360^\circ$ rotation in both directions.
\end{itemize}

The remainder of this paper is organized as follows. Section~\ref{sec:related_work} reviews related literature. Section~\ref{sec:preliminaries} introduces the necessary theoretical background. Section~\ref{sec:task_manifold_cbf} derives the pullback ECBF and BCBF formulations. Section~\ref{sec:task_manifold_actions} presents the task manifold action interface. Sections~\ref{sec:7dof_simulation} and~\ref{sec:hardware_experiments} present simulation and hardware experiments, respectively.
\section{Related work}
\label{sec:related_work}
Given the breadth of the motion planning and robotic manipulation literature, we restrict our attention to topics directly related to this work. For the adjacent literature, we refer the reader to recent surveys on optimization-based task and motion planning~\cite{zhao2024optimization}, robot manipulation in contact~\cite{suomalainen2022survey}, motion planning for manipulators in dynamic environments~\cite{liu2024review}, deep reinforcement learning for robotics~\cite{tang2025deep}, and imitation learning for contact-rich manipulation~\cite{tsuji2026imitation}.

\subsection{Multi-task motion planning on manifolds}
\label{sec:rw_ds}
Reactive multi-objective control in Euclidean task manifolds is classically addressed by operational space control~\cite{khatib1987unified,sentis2005synthesis,siciliano1991general}, which composes Cartesian objectives through null space projection. This line of work establishes key ingredients but does not provide a geometrically consistent or coordinate free framework for composing manifold valued tasks.

A major step toward manifold-valued policy composition is the introduction of \emph{Riemannian Motion Policies} (RMPs)~\cite{ratliff2018riemannian}, which compose acceleration-level policies through metric-weighted pullback; RMPflow~\cite{cheng2021rmpflow} later extends this construction to tree-structured task hierarchies. However, the original RMP formulation is not chart-invariant, so the resulting policy depends on the choice of local coordinates~\cite{bylard2021composable}. A related line of work is the \emph{Optimization Fabrics} and \emph{Geometric Fabrics} framework~\cite{ratliff2020optimization,van2022geometric}, which provides a Finsler-geometric foundation for stable reactive policy design. Across this family of methods, however, safety-related behaviors are typically encoded in the reactive policy itself rather than enforced as hard constraints at runtime. Unresolved geometric consistency issues also leave them short of a complete solution to composing tasks defined on heterogeneous manifolds.

\emph{Pullback Bundle Dynamical Systems} (PBDS)~\cite{bylard2021composable} resolves the geometric consistency issue by realizing simple mechanical control systems (SMCSs)~\cite{bullo2004geometric} on task manifolds and combining them through a metric-weighted least-squares problem. In doing so, PBDS establishes a principled foundation for geometrically consistent control synthesis on manifolds. However, PBDS inherits the broader limitation above: safety-related tasks enter the least-squares objective as soft costs and may therefore be violated under competing task pressures. PBDS also remains purely autonomous, with no mechanism for a higher-level planner or learned policy to steer the system at runtime. Our work addresses both limitations.

\subsection{Control barrier functions on manifolds}
\label{sec:rw_cbf}

Control barrier functions (CBFs) certify forward invariance of a safe set by imposing an affine inequality on the control input at runtime~\cite{ames2019controlbarrierfunctionstheory}. The earliest formulations assume the safety function has relative degree one with respect to the input. This assumption fails for systems where position-level safety must be enforced through acceleration-level control. Two families of methods address this higher-relative-degree setting. \emph{Exponential CBFs} (ECBFs)~\cite{nguyen2016exponential} apply pole-placement design to the chain of Lie derivatives of the safety function, collapsing the higher-order condition into a single linear constraint on the input. A related generalization, high-order CBFs (HOCBFs)~\cite{xiao2021high}, relaxes the linear pole-placement structure of ECBFs to a sequence of class-$\mathcal{K}$ comparison functions. Backstepping CBFs (BCBFs) instead build a CBF from a safe virtual controller for the lower-order subsystem and lift it to the full system~\cite{taylor2022safe}; for second-order systems with position constraints, this construction reduces to lifting a safe velocity field to an acceleration-level constraint.

Recent work has generalized the CBF methodology beyond Euclidean spaces to manifold-valued states. Wu and Sreenath~\cite{wu2015safety} extend CBF synthesis to mechanical systems evolving on Riemannian configuration manifolds, with demonstrations on the spherical pendulum ($\mathbb{S}^2$) and the 3D pendulum ($\mathrm{SO}(3)$). More recently, De Sa et al.~\cite{de2025bundles} develop a general theory of geometric CBFs on control systems defined over bundles, and use it to generalize kinetic-energy CBF backstepping to SMCSs. These works establish geometric CBFs on individual \textit{configuration} manifolds; \methodname{} instead defines safety on \emph{task} manifolds so that hard constraints naturally expressed in different geometric spaces (e.g. end effector poses and joint angle limits) can be composed with the PBDS task structure.
\subsection{Dexterous manipulation}
\label{sec:rw_manipulation}
We focus on two dexterous manipulation problems targeted in our hardware evaluation: multi-fingered \emph{grasping} and \emph{in-hand reorientation}. Both require coordinating objectives and constraints defined on heterogeneous spaces, including force closure and friction-cone conditions in contact space, fingertip and link clearance in Cartesian space, joint limits in the configuration manifold, and reachability in $\mathrm{SE}(3)$.

\subsubsection{Grasping}
The theory of static multi-fingered grasping is well established~\cite{murray1994mathematical}. We refer the reader to~\cite{bohg2013data,newbury2023deep,leen2025deep} for broader surveys of grasp synthesis and focus on methods that provide explicit physical-feasibility guarantees, as in our work. A central challenge in certifiable grasp synthesis is that precise local constraints must be satisfied while searching over multiple global grasp modalities (see~\cite{feix2016grasp} for a grasp taxonomy). This challenge has motivated optimization-based pipelines that incorporate analytical grasp-quality objectives, sometimes used to refine learned grasp predictions~\cite{liu2021synthesizing,turpin2023fastgraspd,zimmermann2020multi}. In most such methods, however, the relevant physical conditions enter as relaxed scalar penalties and therefore do not provide strict guarantees on the final grasp. Wu et al.~\cite{wu2022learning} address this limitation by formulating grasp refinement as a bilevel optimization in which force closure is imposed as an exact inner constraint, yielding grasps with certified physical feasibility at the solution. Li et al.~\cite{li2023frogger} further introduce the min-weight metric as a surrogate objective for force closure certification. However, both~\cite{wu2022learning} and~\cite{li2023frogger} certify grasp feasibility offline and do not maintain it online under execution-time perturbations. Lum et al.~\cite{lum2024dextrahg} combine reinforcement learning with a Geometric Fabric to enforce arm joint limits and collision avoidance during grasp execution, but do not certify grasp conditions such as force closure. Shaw Cortez et al.~\cite{cortez2021control} use a CBF-based safety filter to preserve grasp validity (e.g., preventing slip and singularities) during in-hand adjustments, but do not discuss how the grasp is initially achieved.

\subsubsection{In-hand reorientation}
In-hand reorientation can be broadly divided by palm orientation. In the \emph{palm-up} setting, gravity helps retain the object, allowing brief release-and-regrasp maneuvers. In the \emph{palm-down} setting, the object must remain securely grasped throughout the motion. Both have proven challenging, and most prior work has relied on deep reinforcement learning with sim-to-real transfer, beginning with palm-up~\cite{andrychowicz2020learning,handa2023dextreme,qi2023general,yin2023rotating} and extending to palm-down more recently~\cite{chen2022system,sievers2022learning,chen2023visual,pitz2024learning,yang2024anyrotate}; see~\cite{weinberg2024survey} for a broader survey. While these learned methods demonstrate impressive dexterity, they do not provide formal safety guarantees and typically require substantial tuning and reward engineering to work on hardware. Analytical approaches are much less common; existing work often exploits mechanical compliance and underactuation to reduce the planning burden~\cite{morgan2022complex,bhatt2022surprisingly}. Suh et al.~\cite{suh2026contact} propose a contact-trust-region formulation for contact-rich MPC and demonstrate palm-up cube reorientation with a fully actuated Allegro hand. To the best of our knowledge, our work is the first to achieve palm-down and variable-wrist-pose in-hand reorientation on a fully actuated general-purpose dexterous hand without relying on either machine learning or mechanical compliance.
\section{Preliminaries}
\label{sec:preliminaries}
Our method draws on differential geometry, pullback bundle dynamical systems, and control barrier functions, which we introduce in this section. Before formalizing each construct, we anchor the exposition with the following running example.

\begin{example}[Reaching with obstacle avoidance]
    \label{ex:7dof}
    A 7-DOF Franka Emika Panda arm is to reach a goal end effector pose while avoiding a fixed workspace obstacle. The configuration is given by seven joint angles $\sigma = (\sigma^1,\ldots,\sigma^7)$ with joint limits $\sigma^j_- \leq \sigma^j \leq \sigma^j_+$. Forward kinematics yields the end effector pose and the pose of each link for collision checking. The motion policy must reconcile objectives on heterogeneous spaces:
    \begin{itemize}[noitemsep, topsep=2pt]
        \item \emph{End effector tracking} in end effector pose.
        \item \emph{Obstacle avoidance} in link-obstacle distances.
        \item \emph{Joint limits} in the configuration manifold.
    \end{itemize}
    This system is implemented in simulation in Section~\ref{sec:7dof_simulation}.
\end{example}

\begin{table}
    \caption{Notation Summary.}
    \label{tab:pbds_notation_core}
    \centering
    \footnotesize
    \renewcommand{\arraystretch}{0.95}
    \begin{tabularx}{\columnwidth}{L C}
        \toprule
        \textbf{Symbol}                                                                          & \textbf{Meaning / Definition}                                                                                        \\
        \midrule
        \multicolumn{2}{l}{\textit{Manifolds and Maps}}                                                                                                                                                                 \\
        $M$                                                                                      & Configuration manifold (dimension $m$)                                                                               \\
        $N_i$                                                                                    & Task manifold for task $i$ (dimension $n_i$)                                                                         \\
        $TM$, $T^*M$                                                                             & Tangent and cotangent bundles of $M$                                                                                 \\
        $(p,v) \in TM$                                                                           & Tangent-bundle point; $v \in T_pM$                                                                                   \\
        $f_i : M \!\to\! N_i$                                                                    & Task map                                                                                                             \\
        $df_i,\,Jf_i$                                                                            & Differential and Jacobian of $f_i$                                                                                   \\
        $\dot{J}f_i$                                                                             & Time derivative of $Jf_i$, i.e.\ $v^\ell\,\partial^2 f_i/\partial x^\ell \partial x^j$                               \\
        $F_i : TM \!\to\! TN_i$                                                                  & Higher-order task map, $(p,v) \mapsto (f_i(p), (df_i)_p(v))$                                                         \\
        $(\nabla df)$                                                                            & Second fundamental form of $f$ \\
        \midrule
        \multicolumn{2}{l}{\textit{Riemannian Geometry}}                                                                                                                                                                \\
        $g_M$, $g_i$                                                                             & Riemannian metrics on $M$ and $N_i$                                                                                  \\
        $\langle \cdot, \cdot \rangle_g$                                                         & Inner product of metric $g$; on $N$ also $\langle \cdot, \cdot \rangle_N$                                            \\
        ${}^M\!\nabla$, ${}^N\!\nabla$                                                           & Levi-Civita connections on $M$ and $N$                                                                               \\
        ${}^M\!\Gamma^k_{ij}$, ${}^N\!\Gamma^\alpha_{\beta\gamma}$                               & Christoffel symbols on $M$ and $N$                                                                                   \\
        $\sharp$                                                                                 & Sharp operator, $g^\sharp: T^*M \to TM$ (and task manifold analogue)                                                    \\
        $\operatorname{grad} h$                                                                  & Riemannian gradient: $\langle \operatorname{grad} h, v \rangle_g = dh(v)$                                            \\
        \midrule
        \multicolumn{2}{l}{\textit{PBDS Framework}}                                                                                                                                                                     \\
        $\sigma(t),\,\dot{\sigma},\,\ddot{\sigma}$                                               & Configuration curve on $M$; velocity and coordinate acceleration                                                     \\
        $x(t)$                                                                                   & task manifold curve on $N$                                                                                              \\
        $\Phi_i$                                                                                 & Potential function on $N_i$                                                                                          \\
        $\mathcal{F}_{D,i}$                                                                      & Dissipative force map $TN_i \!\to\! T^*N_i$                                                                          \\
        $\mathcal{F} \subset T^*M$                                                               & SMCS input codistribution                                                                                            \\
        $w_i$                                                                                    & Weighting pseudometric (task priority)                                                                               \\
        $w_i^a \in \mathbb{R}^{n_i\times n_i}$                                                   & Lower-right block of $w_i$ in local coordinates                                                                      \\
        $u \in T^*M$                                                                             & Control input (generalized force one-form)                                                                           \\
        $\mathcal{D}_{(p,v)}$                                                                    & Affine distribution of second-order vectors at $(p,v)$                                                               \\
        $\mathcal{A}_{\mathrm{safe}}$                                                            & Safe acceleration set (pullback CBFs)                                                                                \\
        $\pi_{\mathrm{VB}},\,\pi_{\mathcal{A}}$                                                  & Vertical-bundle and actuation projections                                                                            \\
        $\bar{a}$                                                                                & Optimal autonomous acceleration~\eqref{eq:autonomous_safe_pbds_def}                                                  \\
        $f_l,\,u_l,\,g_l,\,w_l$                                                                  & Control task map, input $u_l\in T^*N_l$, behavior metric, weighting pseudometric                                     \\
        \midrule
        \multicolumn{2}{l}{\textit{Control Barrier Functions (ECBF)}}                                                                                                                                                   \\
        $h_0 : N \!\to\! \mathbb{R}$                                                             & Safety function (indexed $h_{0,j}$ for multiple constraints)                                                         \\
        $\mathcal{C}_0 \subseteq \mathcal{S}$                                                    & Safe set $\{ x : h_0(x) \geq 0 \}$, contained in informal safe region $\mathcal{S}\subseteq N$                       \\
        $\alpha \in \mathcal{K}_\infty^e$                                                        & Extended class-$\mathcal{K}_\infty$ function                                                                         \\
        $\eta_b$                                                                                 & Barrier state $[h_0,\dot{h}_0,\ldots,h_0^{(r-1)}]^\top$                                                              \\
        $\bm{\kappa}$                                                                            & ECBF gain vector; $\bm{\kappa}^\top \in \mathbb{R}^{1\times r}$ (scalars $\kappa_1,\kappa_2$ for $r{=}2$)            \\
        $\nu_i,\ \mathcal{C}^{\nu}_i$                                                            & Auxiliary functions and sets in the ECBF recursion; $p_i$ are negated eigenvalues of $F-G\bm{\kappa}^\top$           \\
        \midrule
        \multicolumn{2}{l}{\textit{Backstepping CBF (BCBF)}}                                                                                                                                                            \\
        $h(x,\dot{x})$                                                                           & Lifted BCBF candidate on $TN$                                                                                        \\
        $\xi,\,\tilde{\xi}$                                                                      & Nominal and safe velocity fields on $N$ (pre/post-filter)                                                            \\
        $\xi_{\mathrm{HS}},\,\alpha_{\mathrm{HS}},\,\beta_{\mathrm{HS}},\,\lambda_{\mathrm{HS}}$ & Half-Sontag safety filter quantities                                                                                 \\
        $\Omega_0 \supset \mathcal{C}_0$                                                         & Open domain on which $\tilde{\xi}$ is a strict CBF                                                                   \\
        $\delta,\,\varepsilon$                                                                   & Strict-margin augmentation and BCBF scaling                                                                          \\
        $\epsilon_{\mathrm{pad}}$                                                                & Numerical padding for $h_0$ (distinct from $\varepsilon$)                                                            \\
        $e_x$                                                                                    & Velocity error $\dot{x} - \tilde{\xi}$                                                                               \\
        \bottomrule
    \end{tabularx}
\end{table}

\subsection{Differential geometry and geometric mechanical systems}

We present a brief introduction to Riemannian geometry; for further details, we refer the reader to~\cite{bullo2004geometric,lee2003smooth}. Let $M$ be a smooth $m$-dimensional manifold with tangent bundle~$TM$ and cotangent bundle~$T^*M$. A \emph{Riemannian metric}~$g$ is a smooth assignment of an inner product~$g_p\colon T_pM \times T_pM \to \mathbb{R}$ to each point~$p \in M$; the pair~$(M,g)$ is called a \emph{Riemannian manifold}. The metric induces the sharp isomorphism \(g^\sharp \colon T^*M \to TM\), which maps each covector \(\alpha \in T_x^*M\) to the unique vector \(g^\sharp(\alpha) \in T_xM\) satisfying
\[
g_x\!\bigl(g^\sharp(\alpha), v\bigr) = \alpha(v)
\qquad \text{for all } v \in T_xM.
\]
The \emph{Riemannian gradient} of a smooth function~$h$ is defined by~$\operatorname{grad} h = g^\sharp(dh)$, or equivalently, $\langle \operatorname{grad} h, v \rangle_g = dh(v)$ for all tangent vectors~$v$. The Levi-Civita connection induced by $g$ is denoted by $\nabla$, so that the covariant acceleration along a curve~$\sigma(t)$ is given by~$\nabla_{\dot{\sigma}}\dot{\sigma}$. When two manifolds $M$ and $N$ must be disambiguated, we write ${}^M\!\nabla$, ${}^N\!\nabla$ for their Levi-Civita connections and ${}^M\!\Gamma$, ${}^N\!\Gamma$ for the corresponding Christoffel symbols.

The Euler-Lagrange equations for a mechanical system admit a coordinate-free formulation on Riemannian manifolds, known as a \emph{simple mechanical control system}:
\begin{definition}[Simple Mechanical Control System]
    \label{def:smcs}
    A \emph{simple mechanical control system} (SMCS)~\cite{bullo2004geometric} is a tuple~$(M, g, \Phi, \mathcal{F})$, where $M$ is an $m$-dimensional smooth manifold (the \emph{configuration manifold}), $g$ is a Riemannian metric on~$M$ (the \emph{kinetic energy metric}), $\Phi \in C^\infty(M)$ is the \emph{potential function}, and $\mathcal{F} \subset T^{*} M$ is a rank-$p$ codistribution on~$M$ (the \emph{input codistribution}). The equations of motion are
    \begin{equation}
        \nabla_{\dot{\sigma}}\dot{\sigma} = -\operatorname{grad}\Phi \circ \sigma + u^\sharp,
        \label{eq:dynamics}
    \end{equation}
    where $u \in \mathcal{E}(\sigma,\mathcal{F})$, the space of sections of $T^{*} M$ along~$\sigma$ such that $u(t) \in \mathcal{F}_{\sigma(t)}$, is the control one-form. In coordinates, \eqref{eq:dynamics} becomes
    \begin{equation}
    \label{eq:dynamics_coords}
        \begin{split}
            &\ddot{\sigma}^k(t) + \Gamma^k_{ij}(\sigma(t))\,\dot{\sigma}^i(t)\dot{\sigma}^j(t) \\
            &\quad = -g^{kl}(\sigma(t))\,\tfrac{\partial \Phi}{\partial \sigma^l}(\sigma(t)) + g^{kl}(\sigma(t))\,u_l(t),
        \end{split}
    \end{equation}
    where $\Gamma^k_{ij}$ are the Christoffel symbols of~$g$ and $u_l$ are the components of~$u$.
\end{definition}
\begin{remark}
    SMCS is control affine since the sharp map is a linear operator.
\end{remark}
\begin{remark}
    $g^{kl}$ converts the force covectors~$u$ and~$d\Phi$ into accelerations. When~$g$ is the kinetic energy metric, $g^{kl}$ plays the role of an inverse mass matrix. $\Gamma^k_{ij}\dot{\sigma}^i\dot{\sigma}^j$ corrects the coordinate acceleration $\ddot{\sigma}^k$ into a true acceleration on the curved manifold similar to the Coriolis and centripetal terms.
\end{remark}

\examplecont
\begin{example}[continued]
    For the 7-DOF arm reaching task, the configuration manifold is the open box of joint limits
    \begin{equation*}
        M = \prod_{j=1}^{7} (\sigma^j_-,\, \sigma^j_+) \subset \mathbb{R}^7,
    \end{equation*}
    equipped with the flat product metric $g_M = \delta_{ij}$. Although the physical joint limits are closed, we use strict inequalities so that $M$ is a smooth manifold without boundary; the distinction makes no difference in practice.
\end{example}

\subsection{Pullback bundle dynamical system (PBDS)}
\label{sec:pbds}

In order to construct a motion policy that combines multiple desired task behaviors,~\cite{bylard2021composable} proposes designing a stable mechanical control system (SMCS) on each task manifold of interest and combining them via a weighted least-squares optimization over the resulting task accelerations. Consider smooth \textit{task maps} $f_i : M \rightarrow N_i$, $i = 1,\ldots,K$, that map the configuration manifold~$M$ (of dimension~$m$) to task manifolds~$N_i$ (of dimension~$n_i$). On each~$N_i$ one designs an SMCS that exhibits  a desired behavior by choosing a potential function $\Phi_i : N_i \rightarrow \mathbb{R}$, a Riemannian metric $g_i$ on $N_i$ (together with its Levi-Civita connection~$\nabla_i$), and a dissipative force map $\mathcal{F}_{D,i} : TN_i \rightarrow T^*\!N_i$. We illustrate the role of $f_i$, $g_i$, $\Phi_i$, and $\mathcal{F}_{D,i}$ on the running 7-DOF arm scenario (Example~\ref{ex:7dof}) in the continuation block following Definition~\ref{def:multi_task_pbds}. Together, these determine a second-order dynamical system on~$TN_i$, which can be pulled back to a dynamical system on the \textit{pullback bundle}
\begin{equation}
    f_i^{*}TN_i = \coprod_{p\in M}\pi_{N_i}^{-1}(f_i(p)),
\end{equation}
where $\coprod$ denotes the disjoint union and $\pi_{N_i} : TN_i \rightarrow N_i$ is the tangent bundle projection. Concretely, the fiber of~$f_i^*TN_i$ over a point~$p \in M$ is the tangent space~$T_{f_i(p)}N_i$. The framework of~\cite{bylard2021composable} constructs on~$f_i^*TN_i$ a pullback connection $f_i^*\nabla_i$, a compatible pullback metric $f_i^*g_i$, pullback dissipative forces~$f_i^*\mathcal{F}_{D,i}$, and a pullback gradient operator~$f_i^*\operatorname{grad}\Phi_i$, all of which are globally well-defined.

Given these pullback constructions, two operators are introduced that capture, respectively, the \emph{desired} task manifold acceleration produced by each local PBDS and the \emph{resulting} task manifold acceleration induced by a candidate configuration manifold acceleration. The first operator,
\begin{equation}
    S_i : TM \longrightarrow T(f_i^{*}TN_i),
\end{equation}
maps a state $(p,v)\in TM$ to the \emph{desired} pullback task acceleration. It is defined as the vertical-bundle projection of the velocity of the local PBDS curve at~$(p,v)$:
\begin{equation}
    S_i(p,v) = \pi_{\mathrm{VB}}\bigl(G_i(p,v)\bigr) = \Bigl(\,(f_i^*df_i)_p(v),\;\bigl(0,\,\dot{\gamma}^{a}_{v_p,i}(0)\bigr)\Bigr),
\end{equation}
where $G_i$ encodes the dynamics of the local PBDS defined by $(f_i, g_i, \Phi_i, \mathcal{F}_{D,i})$, and $\dot{\gamma}^{a}_{v_p,i}(0)$ is the resulting task manifold acceleration. The pullback construction ensures that $S_i$ is globally well-defined; see~\cite{bylard2021composable} for a detailed derivation. In local coordinates, the desired acceleration evaluates to
\begin{multline}
    \label{eq:pbds_desired_accel}
    (\dot{\gamma}^a_{v_p,i})^\alpha(0) = g_i^{\alpha\beta}(f_i(p))\,\mathcal{F}_{D,i}^\beta\bigl((df_i)_p(v)\bigr) \\
    - \frac{\partial \Phi_i}{\partial y_i^\beta}(f_i(p))\,g_i^{\alpha\beta}(f_i(p)) \\
    - v^\ell\,(Jf_i)^\beta_h(p)\,v^h\,(Jf_i)^\gamma_\ell(p)\,(\Gamma_i)^\alpha_{\gamma\beta}(f_i(p)).
\end{multline}

The second operator is defined as
\begin{equation}
    Z_i : TTM \longrightarrow T(f_i^{*}TN_i),
\end{equation}
maps a candidate configuration manifold acceleration to its \emph{resulting} task manifold acceleration. It is defined as the vertical-bundle projection of the differential of the pullback differential~$f_i^*df_i$:
\begin{equation}
    Z_i\bigl((p,v),\,a\bigr) = \pi_{\mathrm{VB}}\Bigl(d(f_i^*df_i)_{(p,v)}(a)\Bigr).
\end{equation}
For $a = \bigl((p,v),(v,a^a)\bigr) \in \mathcal{D}_{(p,v)}$ (defined below), the local expression is
\begin{equation}
    Z_i(a) = \Bigl(\,(f_i^*df_i)_p(v),\;\bigl(0,\, Jf_i(p)\,a^a + \dot{J}f_i(p,v)\,v\bigr)\Bigr),
\end{equation}
where $(\dot{J}f_i)^\alpha_j(p,v) = v^\ell\,\partial^2 f_i^\alpha / \partial x^\ell \partial x^j$.

Let $\mathcal{D}_{\dot\sigma(t)} \subseteq T_{\dot\sigma(t)}TM$ denote the affine distribution of second-order vectors at $\dot\sigma(t) = (p,v) \in TM$, i.e.\ $\mathcal{D}_{(p,v)} = \{((p,v),(v,a^a)) : a^a \in \mathbb{R}^n\}$. Additionally, let $w_i$ be a Riemannian \emph{weighting pseudometric} on~$TN_i$ that controls the priority of task~$i$ relative to the other tasks, and let $F_i : TM \rightarrow TN_i$, $(p,v) \mapsto (f_i(p), (df_i)_p(v))$ be the higher-order task map. A \textit{best-compromise} acceleration is then obtained by minimizing the weighted sum of task-acceleration tracking errors:

\begin{definition}[Multi-Task PBDS, {\cite[Definition~III.2]{bylard2021composable}}]
    \label{def:multi_task_pbds}
    Let $\{f_i : M \longrightarrow N_i\}_{i=1,\ldots,K}$ be smooth task maps for Riemannian task manifolds $(N_i, g_i)$, with corresponding smooth potential functions $\Phi_i : N_i \longrightarrow \mathbb{R}$, dissipative forces $\mathcal{F}_{D,i} : TN_i \longrightarrow T^*\!N_i$, and weighting pseudometrics (positive semidefinite Riemannian metric)~$w_i$ on~$TN_i$. Then the set $\{(f_i, g_i, \Phi_i, \mathcal{F}_{D,i}, w_i)\}_{i=1,\ldots,K}$ forms a \emph{multi-task PBDS} with curves $\sigma : [0,\infty) \longrightarrow M$ satisfying
    \begin{equation}
        \label{eq:multi_pbds_def}
        \begin{cases}
            \ddot{\sigma}(t) = \displaystyle\underset{a \in \mathcal{D}_{\dot\sigma(t)}}{\arg\min} \ \sum_{i=1}^{K} \frac{1}{2} \lVert Z_i(a) - S_i(\dot\sigma(t)) \rVert^2_{F_i^{*}w_i}, \\[6pt]
            \dot\sigma(0) = (p_0, v_0).
        \end{cases}
    \end{equation}
\end{definition}

\noindent The dynamics~\eqref{eq:multi_pbds_def} are globally well-defined and smooth on~$TM$; see~\cite{bylard2021composable} for details. Since~\eqref{eq:multi_pbds_def} is a least-squares problem over the free variable~$a^a \in T_pM$, it admits the local closed-form solution
\begin{equation}
    \label{eq:pbds_closed_form}
    \ddot{\sigma} = \biggl(\sum_{i=1}^{K} Jf_i^\top\, (w_i^a \circ F_i)\, Jf_i\biggr)^{\!\dagger} \biggl(\sum_{i=1}^{K} Jf_i^\top\, (w_i^a \circ F_i)\, A_i\biggr),
\end{equation}
where $(\cdot)^\dagger$ is the pseudoinverse, $w_i^a \in \mathbb{R}^{n_i \times n_i}$ is the lower-right block of the local matrix representation of~$w_i$ (the only block that contributes after the vertical-bundle projections in $S_i$ and~$Z_i$), and
\begin{equation}
    \label{eq:pbds_Ai}
    A_i = \dot{\gamma}^a_{\dot\sigma,i}(0) - \dot{J}f_i(\sigma, \dot\sigma)\,\dot\sigma,
\end{equation}
with $\dot{\gamma}^a_{\dot\sigma,i}(0)$ given by~\eqref{eq:pbds_desired_accel} evaluated at $p = \sigma(t)$, $v = \dot\sigma(t)$.
\begin{remark}
    \label{rem:coordinate_acceleration}
    The multi-task PBDS~\eqref{eq:multi_pbds_def} outputs the coordinate acceleration $\ddot{\sigma}$. The optimization variable $a^a \in T_pM$ in~\eqref{eq:multi_pbds_def} represents the coordinate second derivative $d^2\sigma^k/dt^2$, and the operators $S_i$ and $Z_i$ involve only the task manifold metrics $g_i$ and weights $w_i$. Consequently, the PBDS framework is oblivious to the geometry of $M$. Intuitively, PBDS produces the \textit{task manifold} desired behavior, which is independent of the robot's \textit{configuration manifold} dynamics. In practice, an acceleration tracking controller must be used to track $\ddot{\sigma}$ and account for the configuration manifold dynamics, which stems from the physical system dynamics.
\end{remark}

\examplecont
\begin{example}[continued]
    The end effector tracking objective decomposes into orientation tracking on $\mathrm{SO}(3)$ and position tracking on $\mathbb{R}^3$; per~\cite{bylard2021composable}, an additional joint space damping task is required for stability. This yields 3 SMCSs on a different manifold:
    \begin{itemize}[noitemsep, topsep=2pt]
        \item \emph{Orientation tracking} on $N_{\mathrm{ori}} = \mathrm{SO}(3)$, task map $f_{\mathrm{ori}} \colon M \to \mathrm{SO}(3)$ given by the end effector orientation from forward kinematics, attractor potential $\Phi_{\mathrm{ori}}$ centered at a goal orientation, and linear damping $\mathcal{F}_{D,\mathrm{ori}}$. In the implementation we use the unit quaternion parameterization (i.e.\ the universal double cover $\mathbb{S}^3 \to \mathrm{SO}(3)$); see Appendix~\ref{app:7dof_experiment_details}.
        \item \emph{Position tracking} on $N_{\mathrm{pos}} = \mathbb{R}^3$ with the identity metric, task map $f_{\mathrm{pos}} \colon M \to \mathbb{R}^3$ given by the end effector position, attractor potential $\Phi_{\mathrm{pos}}$ centered at a goal position, and linear damping $\mathcal{F}_{D,\mathrm{pos}}$.
        \item \emph{Joint space damping} with identity task map $f_{\mathrm{jd}} = \mathrm{id}_M$, no potential, and linear damping $\mathcal{F}_{D,\mathrm{jd}}$.
    \end{itemize}
    In the implementation, the attractor and damping of each tracking SMCS are realized as separate PBDS tasks on the same manifold, so the multi-task PBDS comprises 5 tasks total (orientation attractor, orientation damping, position attractor, position damping, and joint space damping). These tasks compose into a multi-task PBDS through Definition~\ref{def:multi_task_pbds} with weighting pseudometrics~$w_i$ tuned by task priority.
\end{example}

\subsection{Control barrier functions (CBFs)}
\label{sec:cbfs}
While multi-task PBDS~\eqref{eq:pbds_closed_form} allows encouraging safe behavior through metric-based constraints~\cite[Section~III.C]{bylard2021composable}, it does not provide a hard guarantee. Additionally, metric-based constraints are ill-defined in unsafe regions. We now review control barrier functions, which provide the tools to overcome this limitation.

Consider a safety problem on a task manifold $N$ (e.g. obstacle avoidance constraints for the end effector), where $\mathcal{S} \subseteq N$ denotes a safe region: the system is safe if the state $x \in \mathcal{S}$.
\begin{definition}[Safety function]
    \label{def:h0}
    A smooth function $h_0 : N \rightarrow \mathbb{R}$ with $0$ as a regular value and whose superlevel set satisfies
    \begin{equation}
        \mathcal{C}_0 \triangleq \{ x \in N \mid h_0(x) \geq 0 \} \subseteq \mathcal{S},
        \label{eq:safe_set_general}
    \end{equation}
    is called a \emph{safety function}.
\end{definition}

Forward invariance of the set $\mathcal{C}_0$ implies that any trajectory initialized in $\mathcal{C}_0$ remains in $\mathcal{C}_0$ for all future times, ensuring that the system always satisfies safety constraints. Safety may be enforced by selecting appropriate control inputs~$u$ that govern the evolution of $h_0$. To characterize how $u$ relates to $h_0$, we present two families of techniques from the literature. Throughout, we let $\alpha \in \mathcal{K}_\infty^e$ denote an extended class-$\mathcal{K}_\infty$ function, i.e., a continuous, strictly increasing function $\alpha : \mathbb{R} \to \mathbb{R}$ satisfying $\alpha(0) = 0$ and $\lim_{r \to \pm\infty} \alpha(r) = \pm\infty$. For a task manifold CBF we write $(x, \dot{x}) \in TN$ for a tangent-bundle point, where $\dot{x} \in T_xN$; we fix a Riemannian metric $g$ on $N$, and the norm $\|\cdot\|_N$, gradient $\operatorname{grad}$, and inner product $\langle\cdot,\cdot\rangle_N$ below are all taken with respect to~$g$.

We first introduce the class of systems under consideration. A \emph{control-affine system} on a manifold~$N$ takes the form
\begin{equation}
    \label{eq:control_affine}
    \dot{x} = f_0(x) + B(x)\,u, \quad x \in N,\; u \in U \subseteq \mathbb{R}^{m_u},
\end{equation}
where $f_0$ is a drift vector field on~$N$, $B(x) = (B_1(x),\ldots,B_{m_u}(x))$ collects $m_u$ input vector fields with $B(x)\,u = \sum_{k=1}^{m_u} B_k(x)\,u_k \in T_xN$, and $U$ is the set of admissible inputs. In the context of the SMCS~\eqref{eq:dynamics}, the state is $(x, \dot{x}) \in TN$ and the control input $u$ enters through the acceleration, so the system has relative degree two with respect to a safety function $h_0(x)$.

\subsubsection{Exponential CBF (ECBF)}

When the safety constraint $h_0$ has relative degree $r \geq 2$ with respect to the control input, i.e. the first $r{-}1$ time derivatives of $h_0$ are independent of $u$ and $u$ first appears explicitly in $h_0^{(r)}$, the standard relative-degree-one CBF condition cannot be applied directly. The exponential CBF (ECBF) framework of~\cite{nguyen2016exponential} addresses this by enforcing a linear constraint on the $r$th derivative of $h_0$.

Define
\begin{equation}
    \eta_b(x) = \bigl[\,h_0(x),\; \dot{h}_0(x),\; \ldots,\; h_0^{(r-1)}(x)\,\bigr]^\top.
\end{equation}
The input-output linearized dynamics of $h_0$ under the virtual input $\mu = h_0^{(r)}$ take the linear form $\dot{\eta}_b = F\,\eta_b + G\,\mu$, $h_0 = C\,\eta_b$, where $(F,G,C)$ is in controllable canonical form. Consider the input constraint $\mu \geq -\bm{\kappa}^\top \eta_b(x)$ with gains $\bm{\kappa}^\top = [\kappa_1, \ldots, \kappa_r] \in \mathbb{R}^{1 \times r}$. By the comparison lemma, under this constraint, $h_0(x(t)) \geq C\,e^{(F - G\bm{\kappa}^\top)t}\,\eta_b(x_0)$. Therefore, safety ($h_0 \geq 0$) is guaranteed if the right-hand side remains nonnegative. To this end, \cite{ames2019controlbarrierfunctionstheory} introduces a family of auxiliary functions $\nu_i$ and corresponding sets $\mathcal{C}^{\nu}_i$:
\begin{equation}
    \nu_0(x) = h_0(x), \quad \nu_i(x) = \dot{\nu}_{i-1}(x) + p_i\,\nu_{i-1}(x), \  i = 1,\ldots,r,
    \label{eq:ecbf_nu_chain}
\end{equation}
\begin{equation}
    \mathcal{C}^{\nu}_i = \{x : \nu_i(x) \geq 0\},
\end{equation}
where $p_1, \ldots, p_r$ are the roots of the characteristic polynomial $F - G\bm{\kappa}^\top$. Forward invariance of $\mathcal{C}_0$ then requires the following conditions on $\bm{\kappa}$~\cite{ames2019controlbarrierfunctionstheory,nguyen2016exponential}:

\begin{definition}[Exponential CBF~\cite{nguyen2016exponential,ames2019controlbarrierfunctionstheory}]
    \label{def:ecbf}
    A safety function $h_0$ of relative degree $r$ with respect to~\eqref{eq:control_affine} is an \emph{exponential CBF} if there exists a row vector $\bm{\kappa} \in \mathbb{R}^{r}$ such that for all $x \in \mathrm{Int}(\mathcal{C}_0)$,
    \begin{equation}
        \sup_{u \in U}\, h_0^{(r)}(x, u) \geq -\bm{\kappa}^\top \eta_b(x).
        \label{eq:ecbf_constraint}
    \end{equation}
    results in $h_0(x(t)) \geq C\,e^{(F - G\bm{\kappa}^\top)t}\,\eta_b(x_0) \geq 0$ whenever $h_0(x_0) \geq 0$.
\end{definition}
\begin{theorem}[{\cite[Theorem~8]{ames2019controlbarrierfunctionstheory}}]
    Suppose $\bm{\kappa}$ is chosen such that $F - G\bm{\kappa}^\top$ is Hurwitz and totally negative (resulting in negative real poles), and its eigenvalues satisfy
    \begin{equation}
        \lambda_i(F - G\bm{\kappa}^\top) \geq -\frac{\dot{\nu}_{i-1}(x_0)}{\nu_{i-1}(x_0)}, \quad i = 1,\ldots,r,
        \label{eq:ecbf_initial_condition}
    \end{equation}
    whenever $\nu_{i-1}(x_0) > 0$. Then $\mu \geq -\bm{\kappa}^\top \eta_b(x)$ guarantees that $h_0(x)$ is an exponential CBF.
\end{theorem}

\begin{remark}
    The totally-negative requirement is stronger than the standard Hurwitz condition, which permits complex eigenvalues with negative real parts. This is because the recursive argument in~\cite[Proposition~6 and Theorem~7]{ames2019controlbarrierfunctionstheory} establishes forward invariance of each $\mathcal{C}^{\nu}_i$ by showing that $\dot{\nu}_{i-1} \geq 0$ on $\partial \mathcal{C}^{\nu}_{i-1}$ only when $p_i > 0$. The eigenvalue bound~\eqref{eq:ecbf_initial_condition} couples the pole locations to the initial state: the poles cannot be placed arbitrarily fast without violating the requirement that $\nu_i(x_0) \geq 0$. In practice, the ECBF can be designed via pole placement, choosing $p_i$ large enough for rapid convergence while satisfying~\eqref{eq:ecbf_initial_condition}. The constraint~\eqref{eq:ecbf_constraint} is then enforced pointwise via a quadratic program (QP) at each time step.
\end{remark}

\begin{remark}
    Evaluating~\eqref{eq:ecbf_constraint} requires computing time derivatives of $h_0$ up to order $r$, which in turn depend on time derivatives of the state up to order $r$. For the SMCS~\eqref{eq:dynamics}, which is second order, $r = 2$, and the two conditions reduce to: 1)~$p_1, p_2 > 0$ with $\kappa_1 = p_1 p_2$ and $\kappa_2 = p_1 + p_2$, and 2)~$p_1 \geq -\dot{h}_0(x_0)/h_0(x_0)$.
\end{remark}

\subsubsection{Backstepping CBF (BCBF)}

An alternative to the ECBF that avoids computing higher-order derivatives, which amplify measurement noise and incur additional autograd cost, is backstepping CBF (BCBF)~\cite{taylor2022safe}. In BCBF, the safety specification on $N$ is lifted to a CBF on $TN$ by introducing an auxiliary safe velocity field on~$N$. Below we recite the key geometric generalization results in~\cite{de2025bundles}. For simplicity, we assume the system is fully actuated ($\pi_{\mathcal{A}} = \mathrm{Id}$) and refer the reader to~\cite{de2025bundles} for the general (including underactuated) case and proof details. In practice, most manipulators are fully- or over-actuated.

\paragraph*{Step 1: Construct a safe velocity field.} Given a nominal velocity field $\xi$ on $N$ (so $\xi(x) \in T_xN$), one applies a smooth safety filter to obtain a \emph{safe velocity field} $\tilde{\xi} : N \to TN$, $x \mapsto \tilde{\xi}(x) \in T_xN$, that renders $\mathcal{C}_0$ forward invariant under first-order dynamics $\dot{x} = \tilde{\xi}(x)$. Following~\cite{de2025bundles}, one uses the \emph{half-Sontag formula}:
\begin{equation}
    \xi_{\mathrm{HS}}(x) = \xi(x) + \lambda_{\mathrm{HS}}(\alpha_{\mathrm{HS}}(x),\, \beta_{\mathrm{HS}}(x))\, \operatorname{grad} h_0,
    \label{eq:half_sontag_vector_field}
\end{equation}
where
\begin{equation}
    \lambda_{\mathrm{HS}}(\alpha_{\mathrm{HS}},\, \beta_{\mathrm{HS}}) =
    \begin{cases}
        0                                                                                                             & \text{if } \beta_{\mathrm{HS}} = 0,    \\[4pt]
        \dfrac{-\alpha_{\mathrm{HS}} + \sqrt{\alpha_{\mathrm{HS}}^2 + \beta_{\mathrm{HS}}^2}}{2\,\beta_{\mathrm{HS}}} & \text{if } \beta_{\mathrm{HS}} \neq 0,
    \end{cases}
    \label{eq:half_sontag}
\end{equation}
with
\begin{align}
    \alpha_{\mathrm{HS}}(x) & = \alpha(h_0(x)) + \langle \operatorname{grad} h_0,\, \xi(x) \rangle_N, \nonumber \\
    \beta_{\mathrm{HS}}(x)  & = \left\| \operatorname{grad} h_0 \right\|_N^2.
\end{align}
The half-Sontag formula is a smooth, closed-form safety filter: it adds a correction along $\operatorname{grad} h_0$ that is just large enough to enforce the CBF condition $d(h_0)_x \xi_{\mathrm{HS}} \geq -\alpha(h_0(x))$, and vanishes when $\xi$ already satisfies it. The nominal velocity field $\xi$ can be chosen freely based on the application; a natural choice within the PBDS framework is $\xi(x) = -\operatorname{grad} \Phi|_x$, which drives the system toward the minimum of the task potential.

To provide a strict margin for the backstepping step, one further augments $\xi_{\mathrm{HS}}$ with an additional gradient term:
\begin{equation}
    \tilde{\xi} = \xi_{\mathrm{HS}}(x) + \delta\, \operatorname{grad} h_0,
    \label{eq:safe_velocity}
\end{equation}
for some $\delta > 0$.

\begin{lemma}[Lemma~2 in~\cite{de2025bundles}]
    \label{lem:half_sontag_safety_filter}
    Let $h_0$ be the smooth safety function and $\mathcal{C}_0$ be the safe set as in Definition~\ref{def:h0}, and let $\alpha \in \mathcal{K}_\infty^e$. There exists an open set $\Omega_0 \supset \mathcal{C}_0$ on which the safe velocity field~\eqref{eq:safe_velocity} satisfies the strict CBF condition
    \begin{equation}
        d(h_0)_x\, \mu_x > -\alpha(h_0(x)), \quad \forall\, x \in \Omega_0.
    \end{equation}
\end{lemma}

\paragraph*{Step 2: Lift to a CBF on $TN$.} Using the safe velocity field $\mu_x$ from Step~1, one defines a candidate CBF on the tangent bundle:
\begin{equation}
    h(x, \dot{x}) = h_0(x) - \frac{\varepsilon}{2} \left\| \dot{x} - \mu_x \right\|_N^2, \quad (x, \dot{x}) \in TN,
    \label{eq:backstepping_cbf}
\end{equation}
where $\varepsilon > 0$ is a design parameter controlling the tradeoff between the safety margin and the allowable velocity deviation from $\mu_x$.

\paragraph*{Step 3: Enforce the BCBF constraint.}

\begin{theorem}[Backstepping CBF~\cite{taylor2022safe, de2025bundles}]
    \label{thm:bcbf}
    The function $h$ in~\eqref{eq:backstepping_cbf} is a CBF for the SMCS on $T\Omega_0$, i.e., there exist control inputs satisfying
    \begin{equation}
        \sup_{u \in U}\, \dot{h}(x, \dot{x}, u) > -\alpha(h(x, \dot{x})),
        \label{eq:backstepping_cbf_constraint}
    \end{equation}
    for all $(x, \dot{x}) \in T\Omega_0$.
\end{theorem}

Since $h(x, \dot{x}) \leq h_0(x)$ for all $(x, \dot{x}) \in TN$ by construction~\eqref{eq:backstepping_cbf}, forward invariance of the set $\{(x, \dot{x}) \in TN : h(x, \dot{x}) \geq 0\}$ immediately implies safety:

\begin{corollary}[Safety~{\cite[Proposition~2]{de2025bundles}}]
    \label{cor:safety_bcbf}
    Let $u$ be any locally Lipschitz control input satisfying $\dot{h}(x, \dot{x}, u) \geq -\alpha(h(x, \dot{x}))$ for all $(x, \dot{x}) \in T\Omega_0$. If $h(x(0), \dot{x}(0)) \geq 0$, then $x(t) \in \mathcal{C}_0$ for all $t \geq 0$.
\end{corollary}

The BCBF approach has two practical advantages over the ECBF: it avoids
computing second (or higher) derivatives of $h_0$ along the system dynamics,
and the resulting CBF constraint~\eqref{eq:backstepping_cbf_constraint} is
affine in $u$, yielding a standard QP. Furthermore, the backstepping
construction extends naturally to the Riemannian manifold setting via the
geometric framework of~\cite{de2025bundles}. Conversely, the ECBF has the
advantage that its constraint~\eqref{eq:ecbf_constraint} depends only on
Lie derivatives of $h_0$ along the system vector fields and is therefore
independent of any choice of Riemannian metric on~$N$, whereas the BCBF
relies on metric-dependent quantities ($\operatorname{grad} h_0$,
$\|\cdot\|_N$, $\nabla_{\dot{x}}\mu_x$, $u^\sharp$).

\section{Pullback CBF for Task Manifold Safety}
\label{sec:task_manifold_cbf}

While metric-based tasks for safety were proposed in~\cite{bylard2021composable}, they present several limitations:
\begin{enumerate}[noitemsep, topsep=2pt]
    \item \emph{Exit behavior:} due to the symmetry of the metric-based velocity field, the constraint task must be deactivated when attempting to leave the safe set.
    \item \emph{Undefined behavior upon violation:} the metric $g = \exp(1/(2x^2))$ used in~\cite{bylard2021composable} is undefined on the constraint boundary $x = 0$, so no recovery mechanism exists once the constraint is violated.
    \item \emph{No guaranteed constraint satisfaction:} the constraint is ultimately enforced as one task among many in the multi-task PBDS QP~\eqref{eq:multi_pbds_def}, and the desired acceleration produced by the constraint task may not be fully realized.
\end{enumerate}

We address these limitations by incorporating \textit{pullback control barrier functions}, which produce sufficient conditions for a configuration-controlled robot to enforce task manifold safety. For brevity, we drop the task manifold index in this section. Consider a safety set $\mathcal{S} \subseteq N$ and a safety function $h_0: N \rightarrow \mathbb{R}$ that defines the safety of the system. We additionally make the following assumption, which is typically satsified in fully actuated and redundant robotic manipulators:
\begin{assumption}[Surjective submersive task map]
    \label{assum:surjective_submersion_task_map}
The task map $f : M \rightarrow N$ is a surjective submersion, so the system is fully actuated on $N$.
\end{assumption}
\begin{remark}
    The safe set in the configuration manifold is the preimage $f^{-1}(\mathcal{C}_0) = \{ p \in M \mid (h_0 \circ f)(p) \geq 0 \}$, which may not be connected. Since trajectories are continuous, forward invariance of $f^{-1}(\mathcal{C}_0)$ implies forward invariance of each connected component individually.
\end{remark}

Since the PBDS framework builds on the SMCS~\eqref{eq:dynamics}, the system is second order. We present adaptations of both the ECBF from~\cite{nguyen2016exponential,ames2019controlbarrierfunctionstheory} and the BCBF from~\cite{taylor2022safe,de2025bundles} enforce task manifold safety constraints on configuration manifold acceleration. By constraining the configuration manifold acceleration to lie in a safe set $\mathcal{A}_{\mathrm{safe}}$, we arrive at the following:
\begin{definition}[Autonomous \methodname{}]
    \label{def:autonomous_safe_pbds}
    The union of $K$ PBDS behavior tasks and $J$ pullback-CBF safety tasks,
    \begin{equation*}
        \{(f_i, g_i, \Phi_i, \mathcal F_{D,i}, w_i)\}_{i=1,\ldots,K} \;\cup\; \{(f_j, h_{0,j}, \bm{\kappa}_j)\}_{j=1,\ldots,J},
    \end{equation*}
    forms a multi-task PBDS with curves $\sigma : [0,\infty) \longrightarrow M$ satisfying
    \begin{equation}
        \label{eq:autonomous_safe_pbds_def}
        \ddot{\sigma}(t) = \underset{a \in \mathcal D_{\dot\sigma(t)}\cap \mathcal{A}_{\mathrm{safe}}}{\arg \min} \ \sum^K_{i=1} \frac{1}{2} \lVert Z_i(a) - S_i(\dot\sigma(t))   \rVert^2_{F_i^*w_i}
    \end{equation}
    where $\mathcal{A}_{\mathrm{safe}}$ is the intersection of the pullback-CBF constraints induced by the $J$ safety tasks (Theorems~\ref{thm:task_ecbf} and~\ref{thm:task_bcbf}). For a BCBF safety task, $\bm{\kappa}_j$ is replaced by its BCBF parameters $(\xi_j, \varepsilon_j)$.
\end{definition}

The remainder of this section derives pullback CBFs under the ECBF and BCBF formulations. The key idea is to pullback the task manifold CBF conditions through the task map to obtain safety constraints expressed entirely in configuration manifold inputs and the robot's state.

For the purpose of~\methodname{}, which is controlled by configuration manifold coordinate acceleration $\ddot{\sigma}$, we show that both the pullback ECBF and pullback BCBF yield linear constraints on $\ddot{\sigma}$. Therefore, safety can be enforced by adding linear constraints to the PBDS QP \eqref{eq:multi_pbds_def}. By substituting in Equation~\eqref{eq:dynamics_coords}, our derivation can be extended to general systems with non-flat configuration manifolds and SMCS dynamics~\eqref{eq:dynamics}.

\subsection{Geometric Kinematics}
While the safety constraint is defined in the task manifold $N$, the control input in PBDS is the configuration manifold coordinate acceleration $\ddot{\sigma}$. Thus, we begin by deriving the geometric kinematics relating the configuration and task manifold accelerations.
\subsubsection{Velocity Relationship}
Let $x(t)$ be a curve on $N$, $\sigma(t)$ be a curve on $M$ where $f(\sigma(t)) = x(t)$. In local coordinates, differentiating $x^\alpha = f^\alpha(\sigma)$ with respect to time yields
\begin{equation}
    \dot{x}^\alpha = \frac{\partial f^\alpha}{\partial \sigma^i} \dot{\sigma}^i = (Jf\, \dot{\sigma})^\alpha
\end{equation}
where $Jf$ is the Jacobian of the task map. In coordinate-free notation, this is $\dot{x} = df\, \dot{\sigma}$.

\subsubsection{Acceleration Relationship}
Differentiating the velocity relation gives the coordinate acceleration on $N$:
\begin{align}
    \ddot{x}^\alpha & =  \frac{d}{dt}\left( \frac{\partial f^\alpha}{\partial \sigma^j} \dot{\sigma}^j \right)                                                                      \\
                    & = \frac{\partial f^\alpha}{\partial \sigma^j} \ddot{\sigma}^j + \frac{\partial^2 f^\alpha}{\partial \sigma^k \partial \sigma^j} \dot{\sigma}^k \dot{\sigma}^j
    \label{eq:coord_acc_expansion}
\end{align}
The covariant accelerations on $N$ and $M$ are given in coordinates by
\begin{equation}
    \left( {}^{N}\nabla_{\dot{x}}\dot{x} \right)^\alpha = \ddot{x}^\alpha + {}^{N}\Gamma^\alpha_{\mu\nu} \dot{x}^\mu \dot{x}^\nu
    \label{eq:cov_acc_N}
\end{equation}
\begin{equation}
    \left( {}^{M}\nabla_{\dot{\sigma}}\dot{\sigma} \right)^k = \ddot{\sigma}^k + {}^{M}\Gamma^k_{ij} \dot{\sigma}^i \dot{\sigma}^j
    \label{eq:cov_acc_M}
\end{equation}
Rearranging~\eqref{eq:cov_acc_M} for $\ddot{\sigma}^k$ and substituting into~\eqref{eq:coord_acc_expansion} yields
\begin{equation}
    \ddot{x}^\alpha = \frac{\partial f^\alpha}{\partial \sigma^k} \left( {}^M\nabla_{\dot{\sigma}}\dot{\sigma} \right)^k + \left( \frac{\partial^2 f^\alpha}{\partial \sigma^i \partial \sigma^j} - \frac{\partial f^\alpha}{\partial \sigma^k} {}^M\Gamma^k_{ij} \right) \dot{\sigma}^i \dot{\sigma}^j
    \label{eq:coord_acc_substituted}
\end{equation}
Substituting~\eqref{eq:coord_acc_substituted} into~\eqref{eq:cov_acc_N} and using $\dot{x}^\mu = \frac{\partial f^\mu}{\partial \sigma^i}\dot{\sigma}^i$ yields
\begin{equation}
    \begin{split}
        &\left( {}^{N}\nabla_{\dot{x}}\dot{x} \right)^\alpha = \underbrace{\frac{\partial f^\alpha}{\partial \sigma^k} \left( {}^M\nabla_{\dot{\sigma}}\dot{\sigma} \right)^k}_{df\left( {}^M\nabla_{\dot{\sigma}}\dot{\sigma} \right)} \\
        &+ \underbrace{\biggl( \frac{\partial^2 f^\alpha}{\partial \sigma^i \partial \sigma^j} - \frac{\partial f^\alpha}{\partial \sigma^k} {}^M\Gamma^k_{ij}
            + {}^{N}\Gamma^\alpha_{\mu\nu} \frac{\partial f^\mu}{\partial \sigma^i} \frac{\partial f^\nu}{\partial \sigma^j} \biggr) \dot{\sigma}^i \dot{\sigma}^j}_{(\nabla df)(\dot{\sigma}, \dot{\sigma})}
    \end{split}
\end{equation}
In coordinate-free notation:
\begin{equation}
    {}^{N}\nabla_{\dot{x}}\dot{x} = df \left( {}^M\nabla_{\dot{\sigma}}\dot{\sigma} \right) + (\nabla df)(\dot{\sigma}, \dot{\sigma})
    \label{eq:geo_kinematics}
\end{equation}
where $(\nabla df)(\dot{\sigma}, \dot{\sigma})$ denotes the second fundamental form of the map $f$ as in~\cite{eells1978report}.

\subsection{Task manifold ECBF}
As~\eqref{eq:geo_kinematics} is a relative degree 2 system, we examine the first and second derivatives of the constraint function $h_0$ along the system trajectories.

The first derivative of $h_0(x(t))$ is computed by the chain rule:
\begin{equation}
    \dot{h}_0 = \frac{\partial h_0}{\partial x^\alpha} \dot{x}^\alpha.
    \label{eq:h_dot}
\end{equation}
For the second derivative, we differentiate~\eqref{eq:h_dot} and substitute $\dot{x}^\alpha = \frac{\partial f^\alpha}{\partial \sigma^i} \dot{\sigma}^i$:
\begin{equation}
    \ddot{h}_0 = \frac{\partial^2 h_0}{\partial x^\alpha \partial x^\beta} \dot{x}^\alpha \dot{x}^\beta + \frac{\partial h_0}{\partial x^\gamma} \ddot{x}^\gamma.
    \label{eq:h_double_dot_expansion}
\end{equation}
Substituting the coordinate acceleration expansion~\eqref{eq:coord_acc_expansion} into~\eqref{eq:h_double_dot_expansion} and applying the chain rule identities
\begin{align}
    \frac{\partial^2 (h_0 \circ f)}{\partial \sigma^i \partial \sigma^j} & = \frac{\partial^2 h_0}{\partial x^\alpha \partial x^\beta} \frac{\partial f^\alpha}{\partial \sigma^i} \frac{\partial f^\beta}{\partial \sigma^j} + \frac{\partial h_0}{\partial x^\gamma} \frac{\partial^2 f^\gamma}{\partial \sigma^i \partial \sigma^j}, \label{eq:chain_rule_hessian} \\
    \frac{\partial (h_0 \circ f)}{\partial \sigma^i}                     & = \frac{\partial h_0}{\partial x^\gamma} \frac{\partial f^\gamma}{\partial \sigma^i}, \label{eq:chain_rule_gradient}
\end{align}
yields
\begin{equation}
    \ddot{h}_0 = \underbrace{\frac{\partial^2 (h_0 \circ f)}{\partial \sigma^i \partial \sigma^j} \dot{\sigma}^i \dot{\sigma}^j}_{\text{Natural Dynamics / Drift}} + \underbrace{\frac{\partial (h_0 \circ f)}{\partial \sigma^i} \ddot{\sigma}^i}_{\text{Control Term}}.
    \label{eq:h_ddot_coordinate}
\end{equation}
Applying the ECBF constraint from Definition~\ref{def:ecbf} with $r = 2$ to~\eqref{eq:h_ddot_coordinate} yields the task manifold formulation:

\begin{theorem}[Task manifold ECBF]
    \label{thm:task_ecbf}
    A safety function $h_0$ (Definition~\ref{def:h0}) is a \textit{task manifold ECBF} for a PBDS system if there exists $\ddot{\sigma} \in T_\sigma M$ such that
    \begin{equation}
        \frac{\partial (h_0 \circ f)}{\partial \sigma^i} \ddot{\sigma}^i \geq -\frac{\partial^2 (h_0 \circ f)}{\partial \sigma^i \partial \sigma^j} \dot{\sigma}^i \dot{\sigma}^j - \kappa_2 \dot{h}_0 - \kappa_1 h_0,
        \label{eq:ecbf_coordinate_form}
    \end{equation}
    where the gains $\kappa_1, \kappa_2$ satisfy the pole placement and initial-condition requirements in Definition~\ref{def:ecbf}.
\end{theorem}
\begin{proof}
    Substituting~\eqref{eq:h_ddot_coordinate} into the ECBF condition $\ddot{h}_0 \geq -\kappa_2 \dot{h}_0 - \kappa_1 h_0$ and rearranging yields~\eqref{eq:ecbf_coordinate_form}. By the comparison lemma, the evolution of $h_0$ satisfies
    \begin{equation}
        h_0(x(t)) \geq [1 \quad 0] e^{\begin{bmatrix} 0 & 1 \\ -\kappa_1 & -\kappa_2 \end{bmatrix} t} \eta_b(x_0).
        \label{eq:feedback_linearized_h_dynamics}
    \end{equation}
    where $\eta_b = [h_0, \dot{h}_0]^\top$. With $\kappa_1 = p_1 p_2$ and $\kappa_2 = p_1 + p_2$ for $p_1, p_2 > 0$, the matrix exponential decays and the right-hand side remains nonnegative provided the initial state satisfies $x_0 \in \mathcal{C}_0 \cap \mathcal{C}^\nu_1$, i.e.\ $h_0(x_0) \geq 0$ and $\dot{h}_0(x_0) + p_1 h_0(x_0) \geq 0$.
\end{proof}

\begin{remark}[Metric independence]
    \label{rem:ecbf_metric_independence}
    The constraint~\eqref{eq:ecbf_coordinate_form} involves only coordinate partial derivatives of $h_0 \circ f$ and the coordinate acceleration $\ddot{\sigma}^i$. Intuitively, $h_0 \circ f$ has no knowledge of $N$, so no geometric quantities of $N$ appear. Since \methodname{} controls $\ddot{\sigma}$, the ECBF constraint contains no geometric quantities altogether.
\end{remark}

\begin{remark}[Covariant form for non-flat $M$]
    \label{rem:ecbf_covariant_form}
    For any choice of Riemannian metric $g_M$ on $M$, substituting $\ddot{\sigma}^k = ({}^M\!\nabla_{\dot{\sigma}}\dot{\sigma})^k - {}^M\!\Gamma^k_{ij}\dot{\sigma}^i\dot{\sigma}^j$ from~\eqref{eq:cov_acc_M} into~\eqref{eq:ecbf_coordinate_form} and using the Riemannian Hessian identity
    \begin{equation}
        \operatorname{Hess}_M(h_0 {\circ} f)(\dot{\sigma},\dot{\sigma}) {=} \tfrac{\partial^2 (h_0 \circ f)}{\partial \sigma^i \partial \sigma^j} \dot{\sigma}^i \dot{\sigma}^j {-} \tfrac{\partial (h_0 \circ f)}{\partial \sigma^k} {}^M\!\Gamma^k_{ij} \dot{\sigma}^i \dot{\sigma}^j,
        \label{eq:riemannian_hessian_M}
    \end{equation}
    yields the equivalent geometric form
    \begin{equation}
        d(h_0 \circ f)\!\left( {}^M\nabla_{\dot{\sigma}}\dot{\sigma} \right) \geq -\operatorname{Hess}_M(h_0 \circ f)(\dot{\sigma}, \dot{\sigma}) - \kappa_2 \dot{h}_0 - \kappa_1 h_0.
        \label{eq:geometric_ecbf_constraint}
    \end{equation}
\end{remark}

\begin{remark}[Behavior outside $\mathcal{C}_0$]
    As long as $h_0$ is smooth and well defined outside $\mathcal{C}_0$, the ECBF constraint guarantees the system will converge asymptotically to $\{x \mid h_0(x) \geq 0\}$ from any initial condition, with convergence rate governed by the eigenvalues $-p_1, -p_2$ via~\eqref{eq:feedback_linearized_h_dynamics}. This contrasts with the metric-based constraint enforcement in~\cite{bylard2021composable}, where the metric $g = \exp(1/(2x^2))$ is undefined on the constraint boundary $h_0(x) = 0$. In practice, one may offset $h_0$ so that $h_0(x) \geq -\epsilon$ represents safety; the ECBF then drives the system to $\{x \mid h_0(x) \geq -\epsilon\}$ in finite time.
\end{remark}

\examplecont
\begin{example}[continued]
    \emph{Workspace obstacle clearance} for robot link $i$ is encoded on the task manifold $N_{\mathrm{obs},i} = \mathbb{R}$ with task map $f_{\mathrm{obs},i}\colon M \to N_{\mathrm{obs},i}\colon \sigma \mapsto d(\mathrm{geom}_i, \mathrm{obs})$, the signed distance between the link $i$ collision geometry and the obstacle. The safety function $h_{0,\mathrm{obs},i}\colon N_{\mathrm{obs},i} \to \mathbb{R}$ is given by $h_{0,\mathrm{obs},i}(x) = x - d_{\min}$, where $d_{\min} \geq 0$ is a minimum clearance margin. 
    
    \emph{Joint limits} for joint $j$ are encoded on the task manifold $N_{\mathrm{lim},j} = \mathbb{R}$ with task map $f_{\mathrm{lim},j}\colon M \to N_{\mathrm{lim},j}\colon \sigma \mapsto \sigma^j$. The safety functions $h_{0,j}^{-}, h_{0,j}^{+} \colon N_{\mathrm{lim},j} \to \mathbb{R}$ are $h_{0,j}^{-}(x) = x - \sigma^j_-$ and $h_{0,j}^{+}(x) = \sigma^j_+ - x$, for $j = 1,\ldots,7$. Both families have relative degree 2 with respect to the task state. Each safety function yields a linear constraint on $\ddot{\sigma}$ via Theorem~\ref{thm:task_ecbf}.
\end{example}

\subsection{Task manifold BCBF}
Here we derive the BCBF safety constraint for the PBDS system. Note that by Assumption~\ref{assum:surjective_submersion_task_map}, the task map $f$ is a surjective submersion, so the system is fully actuated on $N$.
\begin{lemma}
    \label{lem:task_bcbf_hdot}
    In the task manifold PBDS setup, the BCBF candidate $h$ given in~\eqref{eq:backstepping_cbf} has time derivative
    \begin{equation}
        \begin{split}
            \dot{h} &= \langle \operatorname{grad} h_0, \dot{x} \rangle_N + \varepsilon \bigl\langle e_x, \nabla_{\dot{x}} \tilde{\xi} \bigr\rangle_N \\
            &\quad - \varepsilon \bigl\langle e_x, df \left( {}^M\nabla_{\dot{\sigma}}\dot{\sigma} \right) + (\nabla df)(\dot{\sigma}, \dot{\sigma}) \bigr\rangle_N.
        \end{split}
        \label{eq:pbds_bcbf_hdot}
    \end{equation}
\end{lemma}

The proof is given in Appendix~\ref{app:proof_task_bcbf_hdot}.

\begin{theorem}[Task manifold BCBF]
    \label{thm:task_bcbf}
    The function $h$ in~\eqref{eq:backstepping_cbf} is a \textit{task manifold
        BCBF} on $T\Omega_0$ for a PBDS system if
    \begin{equation}
        \sup_{\ddot{\sigma} \in T_\sigma M}
        \dot{h}(\sigma, \dot{\sigma}, \ddot{\sigma})
        > -\alpha(h(x, \dot{x})),
        \label{eq:task_bcbf_constraint}
    \end{equation}
    for all $(\sigma, \dot{\sigma}) \in TM$ with
    $(f(\sigma), df\,\dot{\sigma}) \in T\Omega_0$,
    where $x = f(\sigma)$, $\dot{x} = df\,\dot{\sigma}$, and $h$, $\dot{h}$ are given
    by~\eqref{eq:backstepping_cbf}
    and~\eqref{eq:pbds_bcbf_hdot}.
\end{theorem}
\begin{proof}
    We verify~\eqref{eq:task_bcbf_constraint} for all $(\sigma, \dot{\sigma}) \in TM$ with $(f(\sigma), df\,\dot{\sigma}) \in T\Omega_0$. Let $x = f(\sigma)$ and $\dot{x} = df\,\dot{\sigma}$. If $e_x = 0$ (i.e.\ $\dot{x} = \tilde{\xi}$), the $\varepsilon$-terms in~\eqref{eq:pbds_bcbf_hdot} vanish and
    \begin{equation}
        \dot{h} = \langle \operatorname{grad} h_0, \tilde{\xi} \rangle_N > -\alpha(h_0(x)) = -\alpha(h(x, \dot{x})),
    \end{equation}
    where the inequality follows from Lemma~\ref{lem:half_sontag_safety_filter}, and the last equality from $h(x, \dot{x}) = h_0(x)$ whenever $e_x = 0$ by~\eqref{eq:backstepping_cbf}. If $e_x \neq 0$, the term $-\varepsilon \langle e_x, df({}^M\!\nabla_{\dot{\sigma}}\dot{\sigma}) \rangle_N$ in~\eqref{eq:pbds_bcbf_hdot} is affine in $\ddot{\sigma}$ by~\eqref{eq:cov_acc_M}, with linear coefficient $-\varepsilon\, g_{\alpha\beta}\, e^\beta \frac{\partial f^\alpha}{\partial \sigma^k}$ (free index $k$ on $M$). Since $f$ is a submersion ($Jf$ has full row rank) and $e_x \neq 0$, this coefficient is nonzero, so $\dot{h}$ can be made arbitrarily large by choosing $\ddot{\sigma} \in T_\sigma M$, giving $\sup \dot{h} = +\infty$.
\end{proof}
\begin{remark}
    Assumption~\ref{assum:surjective_submersion_task_map} ensures the system can be controlled in any direction on the tangent space of the task manifold, so the supremum over accelerations is unbounded if $e_x \neq 0$. This significantly simplifies the CBF design process.
\end{remark}

\begin{proposition}[Task manifold safety under BCBF]
    \label{prop:task_bcbf_safety}
    Let $\tilde{\xi}$ be a safe velocity field as in Lemma~\ref{lem:half_sontag_safety_filter}, and let $\ddot{\sigma}$ satisfy the non-strict BCBF constraint $\dot{h} \geq -\alpha(h(x, \dot{x}))$ for all $(x, \dot{x}) \in T\Omega_0$. For any initial state with $h(x(0), \dot{x}(0)) \geq 0$, the closed-loop trajectory satisfies $x(t) \in \mathcal{C}_0$ for all $t \geq 0$.
\end{proposition}
\begin{proof}
    Since $\alpha \in \mathcal{K}_\infty^e$ satisfies $\alpha(0) = 0$ and is strictly increasing, the non-strict constraint $\dot{h} \geq -\alpha(h)$ implies forward invariance of $\{h \geq 0\}$ by the comparison lemma, so $h(x(t), \dot{x}(t)) \geq 0$ for all $t \geq 0$. Since $h_0(x) \geq h(x, \dot{x})$ by~\eqref{eq:backstepping_cbf}, we conclude $h_0(x(t)) \geq 0$, and hence $x(t) \in \mathcal{C}_0$ for all $t \geq 0$.
\end{proof}

\begin{remark}
    Unlike the ECBF formulation, the BCBF depends on the choice of Riemannian metric on~$N$: the safe velocity field $\tilde{\xi}$, the velocity error norm $\|e_x\|_N$, the covariant derivative $\nabla_{\dot{x}} \tilde{\xi}$, and the inner product in~\eqref{eq:pbds_bcbf_hdot} all involve the metric. Intuitively, this is because the backstepping construction compares the current velocity $\dot{x}$ against the safe velocity field $\tilde{\xi}$ using the metric on $N$.
\end{remark}

\begin{remark}[Strict vs.\ non-strict inequalities]
    The ECBF constraint~\eqref{eq:ecbf_coordinate_form} uses a non-strict
    inequality ($\geq$), following~\cite{ames2019controlbarrierfunctionstheory},
    with the regularity condition (that $0$ is a regular value of $h_0$) assumed
    in Definition~\ref{def:h0}. The BCBF
    constraint~\eqref{eq:task_bcbf_constraint} uses a strict inequality
    ($>$), following~\cite{de2025bundles}. The strict inequality is a stronger
    condition: it implies regularity~\cite[Remark~1]{de2025bundles}, but not
    vice versa, since it additionally requires the existence of a control input
    that makes $h$ strictly increase on $\{h = 0\}$. In the
    backstepping construction, this strict margin arises naturally from the
    $\delta\,\operatorname{grad} h_0$ augmentation
    in~\eqref{eq:safe_velocity}. Note that the actual enforcement constraint used in Proposition~\ref{prop:task_bcbf_safety} is the non-strict $\dot{h} \geq -\alpha(h)$; the strict inequality in Theorem~\ref{thm:task_bcbf} establishes that valid controls exist.
\end{remark}

\begin{remark}[Recovery outside $\mathcal{C}_0$]
    For the BCBF, if $h_0$ is smooth on all of $N$, the half-Sontag formula~\eqref{eq:half_sontag_vector_field} produces a globally smooth safe velocity field $\tilde{\xi}$, so the candidate $h$~\eqref{eq:backstepping_cbf} is well-defined on all of~$TN$. The non-strict enforcement constraint $\dot{h} \geq -\alpha(h)$ is then satisfiable everywhere: when $e_x \neq 0$, full actuation makes the supremum unbounded; when $e_x = 0$, the half-Sontag formula guarantees $d(h_0)_x\,\tilde{\xi} \geq -\alpha(h_0(x))$ globally. If $h < 0$, a comparison argument using $\alpha \in \mathcal{K}_\infty^e$ shows $h(t) \to 0$, recovering safety. The domain $\Omega_0$ from Lemma~\ref{lem:half_sontag_safety_filter} is only required for the strict inequality in Theorem~\ref{thm:task_bcbf}, which establishes $h$ as a CBF in the formal sense, but does not limit the practical recovery behavior.
\end{remark}
\begin{remark}[Practical considerations]
    In practice, finite differencing is often used to compute derivatives and smooth functions that are Lipschitz continuous but not everywhere differentiable. Additionally, due to numerical errors and latency, it is possible to get small CBF violations despite satisfying CBF constraints. Such issues may be mitigated by padding the safety function in~\eqref{eq:safe_set_general} by a small margin $\epsilon_{\mathrm{pad}} > 0$, which we apply in our implementation.
\end{remark}

\section{Task manifold actions}
\label{sec:task_manifold_actions}
In the autonomous \methodname{} framework (Definition~\ref{def:autonomous_safe_pbds}), runtime decision-making is limited to adjusting the weighting pseudometrics~$w_i$. However, one may wish to steer the system directly in a task manifold~$N$, for example to incorporate commands from a higher-level planner or a learned policy. We therefore propose a mechanism that exposes control inputs while preserving the safety and compositional structure of the PBDS framework. We seek the following desiderata:
\begin{enumerate}[noitemsep, topsep=2pt]
    \item The action space may be lower dimensional than the configuration dimension~$m$.
    \item If so desired, all~$m$ dimensions may be controlled.
    \item The action represents a residual force on top of the autonomous PBDS dynamics: when the action is zero, the system follows the autonomous behavior.
\end{enumerate}

To this end, consider $L$ \emph{control task maps} $f_l: M\rightarrow N_l$, $l=1,\ldots,L$, each equipped with a weighting pseudometric~$w_l$ on~$TN_l$ and a Riemannian behavior metric~$g_l$ on~$N_l$. These may coincide with some of the $K$ autonomous task maps, or they may be entirely separate. For each control task, the user supplies a force-like input $u_l \in T^*\!N_l$, which is converted to an acceleration-level quantity via the sharp (musical isomorphism) $u_l^\sharp = g_l^\sharp(u_l) \in TN_l$ (in coordinates, $g_l^{-1}\, u_l$).

\begin{definition}[Steerable \methodname{}]
    \label{def:steerable_pbds}
    Let $\bar{a}$ be the optimal autonomous acceleration~\eqref{eq:autonomous_safe_pbds_def}. Given $L$ control inputs $u_l \in T^*\!N_l$, each defined in control task manifolds $f_l: M\rightarrow N_l$, the acceleration under action input is
    \begin{equation}
        \begin{split}
            \label{eq:safe_steered_pbds_def}
            \ddot{\sigma}(t) = \underset{a \in \mathcal D_{\dot\sigma(t)}\cap \mathcal{A}_{\mathrm{safe}}}{\arg \min} \ \sum^{K}_{i=1} \frac{1}{2} \lVert Z_i(a) - S_i(\dot{\sigma}(t))   \rVert^2_{F_i^*w_i} \\ + \sum^{L}_{l=1} \frac{1}{2} \lVert Z_l(a) - Z_l(\bar{a}) - u_l^\sharp   \rVert^2_{F_l^*w_l}.
        \end{split}
    \end{equation}
    where $S_i$ and $Z_i$ are defined in Section~\ref{sec:pbds}.
\end{definition}
The first term is identical to the autonomous objective~\eqref{eq:multi_pbds_def}; the second term shifts the autonomous acceleration by the control input~$u_l^\sharp$. Since $Z_i$ is affine in the acceleration and the constraint set $\mathcal{D}_{\dot\sigma} \cap \mathcal{A}_{\mathrm{safe}}$ is characterized by linear inequalities (Section~\ref{sec:task_manifold_cbf}), the steered problem~\eqref{eq:safe_steered_pbds_def} remains a convex QP and can be solved with the same infrastructure as the autonomous problem. Algorithm~\ref{alg:safepbds_step} summarizes the per-step procedure with the safety constraints written out.

\begin{algorithm}[!htb]
    \caption{\methodname{} control step}
    \label{alg:safepbds_step}
    \small
    \begin{algorithmic}[1]
        \REQUIRE state $(\sigma, \dot{\sigma}) \in TM$; behavior tasks $\{(f_i, g_i, \Phi_i, \mathcal{F}_{D,i}, w_i)\}_{i=1}^{K}$; safety tasks indexed by $\mathcal{J}_E$ (ECBF, Theorem~\ref{thm:task_ecbf}) and $\mathcal{J}_B$ (BCBF, Theorem~\ref{thm:task_bcbf}); control tasks $\{(f_l, g_l, w_l, u_l)\}_{l=1}^{L}$.
        \ENSURE coordinate acceleration $\ddot{\sigma} \in T_\sigma M$.
        \STATE $\mathcal{A}_{\mathrm{safe}} \gets T_\sigma M$
        \FOR{each $j \in \mathcal{J}_E$}
        \STATE intersect $\mathcal{A}_{\mathrm{safe}}$ with the half-space
        \begin{multline*}
            \tfrac{\partial (h_{0,j} \circ f_j)}{\partial \sigma^k} a^k \geq -\tfrac{\partial^2 (h_{0,j} \circ f_j)}{\partial \sigma^k \partial \sigma^\ell} \dot{\sigma}^k \dot{\sigma}^\ell \\
            - \kappa_{2,j} \dot{h}_{0,j} - \kappa_{1,j} h_{0,j}
        \end{multline*}
        \ENDFOR
        \FOR{each $j \in \mathcal{J}_B$}
        \STATE intersect $\mathcal{A}_{\mathrm{safe}}$ with $\dot{h}_j(\sigma, \dot{\sigma}, a) \geq -\alpha_j\!\bigl(h_j(x_j, \dot{x}_j)\bigr)$, where $\dot{h}_j$ is given by~\eqref{eq:pbds_bcbf_hdot}
        \ENDFOR
        \STATE \textbf{QP~1} (autonomous safe acceleration):
        \[
            \bar{a} \gets \argmin_{a \in \mathcal{A}_{\mathrm{safe}}} \ \sum^{K}_{i=1} \tfrac{1}{2} \lVert Z_i(a) - S_i(\dot{\sigma}) \rVert^2_{F_i^*w_i}
        \]
        \STATE \textbf{QP~2} (steered acceleration): solve~\eqref{eq:safe_steered_pbds_def} for $\ddot{\sigma}$ given $\bar{a}$, with $a \in \mathcal{A}_{\mathrm{safe}}$
        \RETURN $\ddot{\sigma}$
    \end{algorithmic}
\end{algorithm}

\begin{theorem}[Autonomous behavior recovery]
    \label{thm:autonomous_recovery}
    If all control inputs $u_l=0$, $\bar{a}$ is a minimizer of~\eqref{eq:safe_steered_pbds_def}.
\end{theorem}

\begin{proof}
    Since $Z_i$ is affine in~$a$ and the feasible set $\mathcal{D}_{\dot\sigma(t)} \cap \mathcal{A}_{\mathrm{safe}}$ is characterized by linear constraints (Section~\ref{sec:task_manifold_cbf}), Equation~\eqref{eq:safe_steered_pbds_def} is a convex QP for which the KKT conditions are necessary and sufficient. Let $Aa \leq b$ denote the assembled linear constraints. Setting $u_l = 0$, the KKT stationarity condition of~\eqref{eq:safe_steered_pbds_def} is
    \begin{multline}
        \nabla_a \!\biggl[\sum^{K}_{i=1} \frac{1}{2} \lVert Z_i(a) \!-\! S_i \rVert^2_{F_i^*w_i} \\
        +\, \sum^{L}_{l=1} \frac{1}{2} \lVert Z_l(a) \!-\! Z_l(\bar{a}) \rVert^2_{F_l^*w_l}\biggr] + A^\top \lambda = 0,
    \end{multline}
    with $Aa \leq b$, $\lambda \geq 0$, and $\lambda^\top(Aa - b) = 0$.

    By construction, $\bar{a} \in \mathcal{D}_{\dot\sigma(t)}\cap \mathcal{A}_{\mathrm{safe}}$, so primal feasibility holds. Since $Z_l$ is affine in~$a$, the gradient of $\lVert Z_l(a) - Z_l(\bar{a}) \rVert^2_{F_l^*w_l}$ vanishes at $a = \bar{a}$. The stationarity condition thus reduces to that of the autonomous problem (Definition~\ref{def:autonomous_safe_pbds}), which $\bar{a}$ satisfies with optimal autonomous problem dual variables~$\bar{\lambda}$. As the feasible set is identical, all KKT conditions are satisfied and $\bar{a}$ minimizes~\eqref{eq:safe_steered_pbds_def}.
\end{proof}
\begin{remark}
    If one interprets $\bar{a}$ as the resulting acceleration of a task manifold SMCS (Equation~\eqref{eq:dynamics}), the second sum in~\eqref{eq:safe_steered_pbds_def} acts as an additional task manifold force.
\end{remark}
\begin{remark}
    $\bar{a}$ is not necessarily a unique minimizer. In practice, since $\bar{a}$ is used to warm start the numerical solver, it is typically selected when all $u_l=0$.
\end{remark}
\begin{remark}[Desiderata]
    \label{rem:desiderata}
    The three desiderata from the beginning of this section are satisfied by construction: (1)~each control task map $f_l : M \rightarrow N_l$ may have $n_l \leq m$, so the total action dimension $\sum_{l=1}^L n_l$ can be less than~$m$; (2)~if the control task maps collectively form a submersion (i.e., $\operatorname{rank}\bigl[\,Jf_1^\top \cdots Jf_L^\top\bigr] = m$), all $m$ degrees of freedom are controllable; and (3)~Theorem~\ref{thm:autonomous_recovery} guarantees that zero control input recovers the autonomous behavior.
\end{remark}

\examplecont
\begin{example}[continued]
    To expose a one-dimensional action space that selects among qualitatively distinct avoidance behaviors (e.g. routing the elbow on either side of a workspace obstacle), we use a control task map given by the coordinate projection onto joint 1,
    \begin{equation*}
        f\colon M \to \mathbb{R}\colon \sigma \mapsto \sigma^1,
    \end{equation*}
    with identity behavior metric $g = 1$ and a positive scalar weighting pseudometric $w$. The action input $u \in \mathbb{R}$ injects an acceleration on joint 1; flipping the sign of $u$ steers the arm into a different homotopy class of safe paths around the obstacle, while $u = 0$ recovers the autonomous behavior by Theorem~\ref{thm:autonomous_recovery}. This is demonstrated in Section~\ref{sec:7dof_steered}.
\end{example}

\section{Evaluation: Simulated $\mathbb{S}^2$ Double Integrator}
Using a point robot on $\mathbb{S}^2$, we first validate the theoretical properties of \methodname{}, in particular chart invariance, safety guarantees, multi-task composition, and effects of our action space design. All experiment configurations in this section are summarized in Table~\ref{tab:s2_run_configs}. Additional details are available in Appendix~\ref{app:s2_experiment_details}.

\subsection{System Setup}

\paragraph{System definition}
Consider a point robot with unit mass traveling on the surface of a unit sphere. The configuration manifold of the robot is $\mathbb{S}^2$, which has a natural embedding $\bar{\varphi}: \mathbb{S}^2\hookrightarrow \mathbb{R}^3$. We consider the atlas $\left\{(U_N, \varphi_N), (U_S, \varphi_S)\right\}$ formed by the north and south pole stereographic projection charts. We can define corresponding maps $\bar{\varphi}_N: \mathbb{R}^2\rightarrow \mathbb{R}^3$ and $\bar{\varphi}_S: \mathbb{R}^2\rightarrow \mathbb{R}^3$ from chart coordinates $(y_1, y_2)$ to embedding coordinates $(x_1, x_2, x_3)$.

\paragraph{Configuration Dynamics}
Since PBDS does not account for the configuration manifold geometry (Remark~\ref{rem:coordinate_acceleration}), we treat the chart coordinates as~$\mathbb{R}^2$ and integrate with a fourth-order Runge--Kutta scheme when simulating the system forward.

\paragraph{Task Dynamics}
Following~\cite{bylard2021composable}, we use the chart-to-embedding task map $f = \bar{\varphi}_\alpha\colon \mathbb{R}^2 \to \mathbb{R}^3$, with an attractor potential and dissipative force defined on the embedding~$\mathbb{R}^3$.

\paragraph{Safety constraint}
We choose the safety task as staying at least arclength $r$ from a point $x_o \in \mathbb{S}^2$. The safety function (Definition~\ref{def:h0}) is
$h_0(y) = \arccos\!\bigl(\bar{\varphi}_\alpha(y) \cdot x_o\bigr) - r,$
with $\alpha \in \{N, S\}$ the active chart. The safe set is $\mathcal{C}_0 = \{y \in \mathbb{R}^2 : h_0(y) \geq 0\}$. For BCBF, we set $\xi = 0$ and compare two metrics: the round metric $g_{ij} = \frac{4}{(1+\|y\|^2)^2}\delta_{ij}$, which is the pullback of the embedded $\mathbb{S}^2$ metric through the stereographic chart, and the flat metric $\delta_{ij}$, which treats the chart coordinates as Euclidean.

\begin{table}[t]
    \centering
    \caption{Run configurations for the $\mathbb{S}^2$ double integrator experiments.}
    \label{tab:s2_run_configs}
    \small
    \setlength{\tabcolsep}{5pt}
    \begin{tabular}{@{}llll@{}}
        \toprule
        Run    & Chart     & Safety (chart, metric) & Action                           \\
        \midrule
        \multicolumn{4}{l}{\textit{Autonomous}: metric dependence and chart switching} \\
        (i)    & N         & ECBF (N)               & --                               \\
        (ii)   & N         & BCBF (N, round)        & --                               \\
        (iii)  & N         & BCBF (N, flat)         & --                               \\
        (iv)   & Switching & ECBF (Switch)          & --                               \\
        (v)    & Switching & BCBF (Switch, round)   & --                               \\
        \addlinespace
        \multicolumn{4}{l}{\textit{Autonomous}: safety recovery}                       \\
        (vi)   & N         & ECBF (N)               & --                               \\
        (vii)  & N         & BCBF (N, round)        & --                               \\
        \addlinespace
        \multicolumn{4}{l}{\textit{Steered}}                                           \\
        (viii) & N         & ECBF (N)               & $0$ (autonomous)                 \\
        (ix)   & N         & ECBF (N)               & $+u_\perp$                       \\
        (x)    & N         & ECBF (N)               & $-u_\perp$                       \\
        (xi)   & N         & ECBF (N)               & $u_{\mathrm{unsafe}}$            \\
        (xii)  & S         & ECBF (S)               & $-u_\perp$                       \\
        \bottomrule
    \end{tabular}
\end{table}

\subsection{Autonomous \methodname{} on $\mathbb{S}^2$}
\label{sec:scenario_1}

First, we validate the theoretical properties of ECBFs and BCBFs in \methodname{} without task manifold actions, shown in Figure~\ref{fig:cbf_combined}(a--b).

\paragraph{Metric dependence and chart switching}
As expected from Remark~\ref{rem:ecbf_metric_independence}, ECBF is metric-independent. The two BCBF runs produce visibly different avoidance trajectories, since BCBF depends on the metric via $\operatorname{grad} h_0$, $\|e_x\|_N$, and $\nabla_{\dot{x}}\tilde{\xi}$; the BCBF also stays farther from the unsafe set boundary due to the additional ``safe velocity tracking error'' penalty (second term in~\eqref{eq:pbds_bcbf_hdot}). We verify chart switching by varying the configuration manifold chart between north and south. This is the safety-augmented analogue of the chart switching experiment in~\cite[Fig.~5]{bylard2021composable}. Run~(iv) matches the fixed-chart reference~(i) on~$\mathbb{S}^2$, confirming that the \methodname{} produce consistent behavior across chart choices and can be deployed on atlases with multiple charts.

\paragraph{Safety recovery}
To test recovery, we run both an ECBF and a BCBF trajectory starting from inside the unsafe set. Both formulations drive the system out of the unsafe region and converge to the goal. The ECBF constrains only $h_0 \geq 0$ and recovers along the most direct path. The BCBF's safe velocity field deviation penalty causes its recovery to maintain a larger clearance from the obstacle boundary.
\begin{figure*}[ht]
    \centering
    \includegraphics[width=0.98\textwidth]{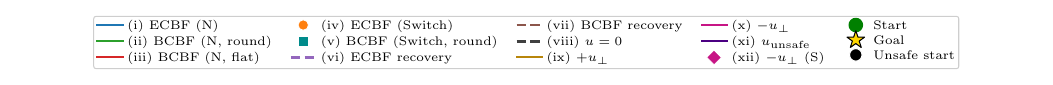}\\[-10pt]
    \subfloat[]{\includegraphics[trim=0 25 0 25, clip, width=0.32\textwidth]{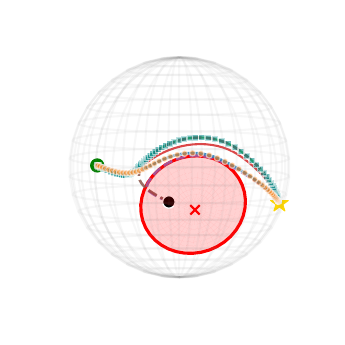}\label{fig:cbf_pbds}}
    \hfill
    \subfloat[]{\includegraphics[trim=0 4 0 4, clip, width=0.32\textwidth]{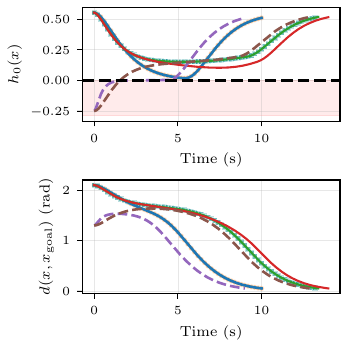}\label{fig:cbf_h0}}
    \hfill
    \subfloat[]{\includegraphics[trim=0 25 0 25, clip, width=0.32\textwidth]{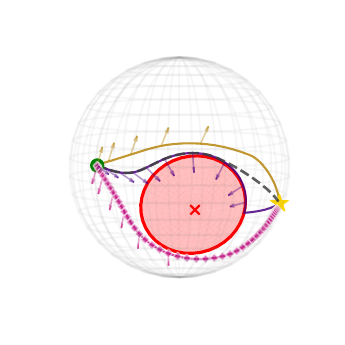}\label{fig:steered_sphere}}
    \caption{Pullback CBF and action interface on $\mathbb{S}^2$; run indices match Table~\ref{tab:s2_run_configs} and the legend. (a)~Autonomous runs~(i)--(vii) on the same scene; recovery runs~(vi)--(vii) start inside the obstacle ($\times$ marker). (b)~$h_0(t)$ (top, shaded region is $h_0<0$) and geodesic distance to goal (bottom) for the same runs. (c)~(viii)--(xii): tangential actions~$\pm u_\perp$ select opposite homotopy classes, $u_{\mathrm{unsafe}}$ is clipped by the CBF, and~(xii) repeats~(x) on the south pole chart to verify chart invariance.}
    \label{fig:cbf_combined}
\end{figure*}

\subsection{Steered \methodname{} on $\mathbb{S}^2$}
\label{sec:scenario_action}

Next, we demonstrate the properties of the task manifold actions (Section~\ref{sec:task_manifold_actions}). We use an ECBF safety constraint and an identity control task map $f_l = \mathrm{id}\colon \mathbb{R}^2 \to \mathbb{R}^2$ on the active chart. We place the obstacle so the system faces two valid avoidance paths on each side of the obstacle. We run five trajectories~(viii)-(xii), all shown in Figure~\ref{fig:cbf_combined}(c):
\begin{enumerate}[label=(\roman*), start=8, leftmargin=*, noitemsep, topsep=2pt]
    \item \emph{Autonomous} ($u = 0$): the system breaks symmetry via a small positional offset and converges to one side.
    \item \emph{$+u_{\perp}$}: a tangential action perpendicular to the start--goal geodesic with $\|u_{\perp}\| = 1$, steering the robot to the same side as the autonomous bias.
    \item \emph{$-u_{\perp}$}: flipping the sign steers the robot to the opposite side, resolving the topological ambiguity that the autonomous dynamics alone cannot.
    \item \emph{$u_{\mathrm{unsafe}}$}: a large action ($\|u\| = 10$) continuously pointing toward the obstacle center. \methodname{} prevents the constraint violation, and safety is maintained ($h_0(t) \geq 0$) despite the adversarial action.
    \item \emph{$-u_{\perp}$ on the south pole chart}: we repeat~(x) on the south pole chart to verify chart invariance under action inputs.
\end{enumerate}

\section{Evaluation: Simulated 7-DOF Robot Arm}
\label{sec:7dof_simulation}
We instantiate the running 7-DOF arm scenario (Example~\ref{ex:7dof}) in MuJoCo~\cite{todorov2012mujoco} to validate the pullback CBF for workspace obstacle avoidance and the action space formulation for behavior modality selection. The arm starts at its home configuration; goals and obstacles vary per experiment as detailed below. Additional details are provided in Appendix~\ref{app:7dof_experiment_details}.

\subsection{Autonomous \methodname{} on the 7-DOF Arm}
\label{sec:7dof_safety}

We validate the pullback ECBF by comparing the full \methodname{} system against an ablation that removes the obstacle avoidance ECBFs while keeping joint limit ECBFs active. We test end effector tracking in two setups:

\paragraph{6-DOF pose ($\mathrm{SE}(3)$) tracking} A spherical obstacle (radius $0.10$\,m) is placed in between the start and goal positions. Figure~\ref{fig:7dof_combined}(a) shows the full system (purple) deflecting around the obstacle, while the ablation (red, dashed) passes through it. Both converge to the goal pose.

\paragraph{Random orientation ($\mathrm{SO}(3)$) tracking}
We sample 50 scenarios that each track a goal orientation (a $30^\circ$-$150^\circ$ rotation about a random axis) with a randomly placed spherical obstacle (radius $8$\,cm). Position tracking is forfeited due to the difficulty of randomly finding a reachable goal pose under the obstacle avoidance constraint. Scenarios with obstacle clearance below $5$\,cm are rejected and resampled. The full system is safe in all 50 runs (min $h_\mathrm{obs} = +0.010$), whereas the ablation  violates the obstacle constraint in 11 of 50 runs (min $h_\mathrm{obs} = -0.080$); all runs reach the goal orientation.

\subsection{Steered \methodname{} on the 7-DOF Arm}
\label{sec:7dof_steered}

We validate the action interface (Section~\ref{sec:task_manifold_actions}) using end effector position ($\mathbb{R}^3$) tracking, which leaves four redundant degrees of freedom for the action input to exploit.
When an obstacle lies in the robot's workspace, the deterministic autonomous system defaults to one avoidance path, yet multiple safe paths exist. The action inputs allow an operator or high-level planner to select among these configurations with minimal input, instantiating the homotopy class selection application described in the running example (Example~\ref{ex:7dof}). We apply a brief constant scalar action on the first joint: $u = \pm 15$ for the first 3.5\,s, then $u = 0$. Flipping the sign of this scalar results in trajectories of distinct homotopy classes: the $+u$ arm (blue) passes on one side while the $-u$ arm (purple) passes on the other (Figure~\ref{fig:7dof_combined}(b)). Both trajectories converge to the goal while preserving safety. This highlights a key practical advantage of the steered \methodname{}: qualitatively different behaviors can be selected using simple, transient inputs, while the autonomous dynamics and safety constraints handle the details of the motion.

\begin{figure}[t]
    \centering
    \subfloat[]{\includegraphics[width=0.48\columnwidth]{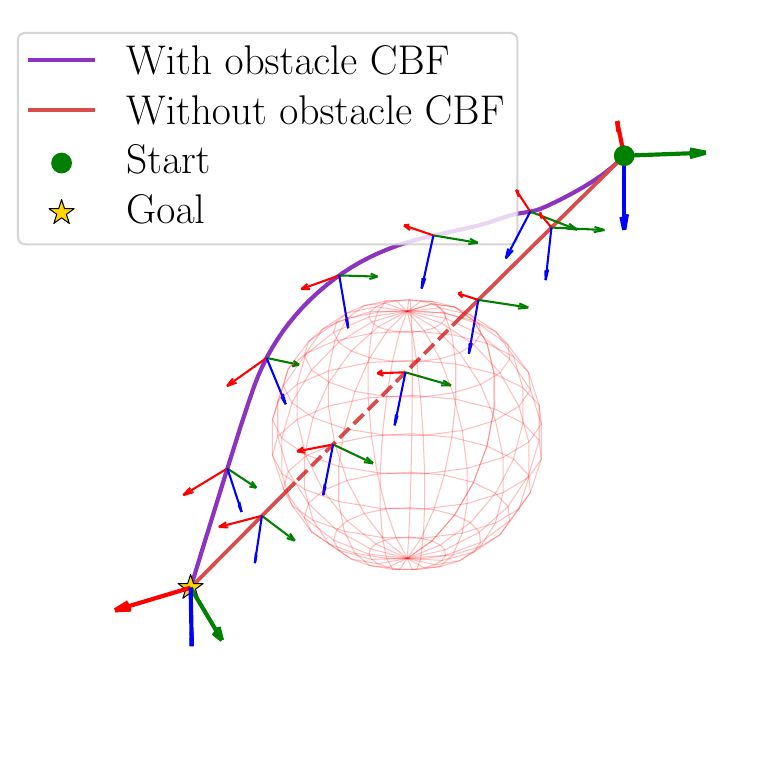}\label{fig:7dof_safety}}
    \hfill
    \subfloat[]{\includegraphics[width=0.48\columnwidth]{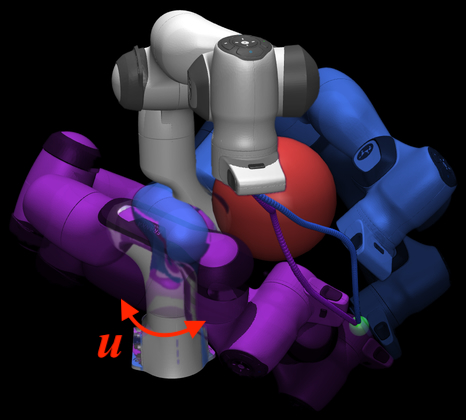}\label{fig:7dof_homotopy_3d_labeled.jpg}}
    \caption{7-DOF arm: workspace safety and steered action. (a)~Obstacle avoidance during 6-DOF pose tracking: the full system (purple) deflects around the obstacle, while the ablation without the obstacle CBF (red, dashed) passes through it (penetration in orange); coordinate frames at equal arc length intervals show orientation converging to the goal. (b)~Behavior selection with action inputs and position-only ($\mathbb{R}^3$) tracking: flipping the sign of a single joint 1 action steers the forearm to opposite sides of the obstacle (blue vs.\ purple), while the end effector converges to the same position goal (green); the initial configuration (gray) and final silhouettes are shown for reference.}
    \label{fig:7dof_combined}
\end{figure}

\section{Evaluation: Hardware Experiments}
\label{sec:hardware_experiments}
We validate \methodname{} on hardware with two dexterous manipulation tasks: autonomous grasping of diverse household objects (Section~\ref{sec:dexterous_grasping}) and in-hand reorientation via finger gaiting (Section~\ref{sec:in_hand_reorientation}).

\subsection{Robot Setup}
\label{sec:hardware_setup}
All hardware experiments are performed with a Franka Emika Panda 7-DOF arm and a Wonik Allegro 16-DOF dexterous hand~\cite{allegro_linux}, for a combined 23-DOF system. The arm is controlled at 20\,Hz via the Deoxys framework~\cite{zhu2023viola}, using the joint impedance controller when PBDS is active, and the joint position controller for the pre-grasp approach. The hand is commanded over ZMQ using a mix of joint-level PD and Cartesian fingertip impedance modes. Runtime perception uses an Intel RealSense D435 camera with Segment Anything~\cite{kirillov2023segment} for segmentation and FoundationPose~\cite{wen2024foundationpose} for 6-DOF pose tracking at 10\,Hz, which updates an object pose in a MuJoCo~\cite{todorov2012mujoco} simulation that the PBDS controller runs against. Object meshes are obtained by scanning with KIRI Engine~\cite{kiri_engine} on a LiDAR-equipped iPhone. Full control, perception, and modeling details are given in Appendix~\ref{app:hw_setup}.

\begin{figure}[!t]
    \centering
    \includegraphics[width=\columnwidth]{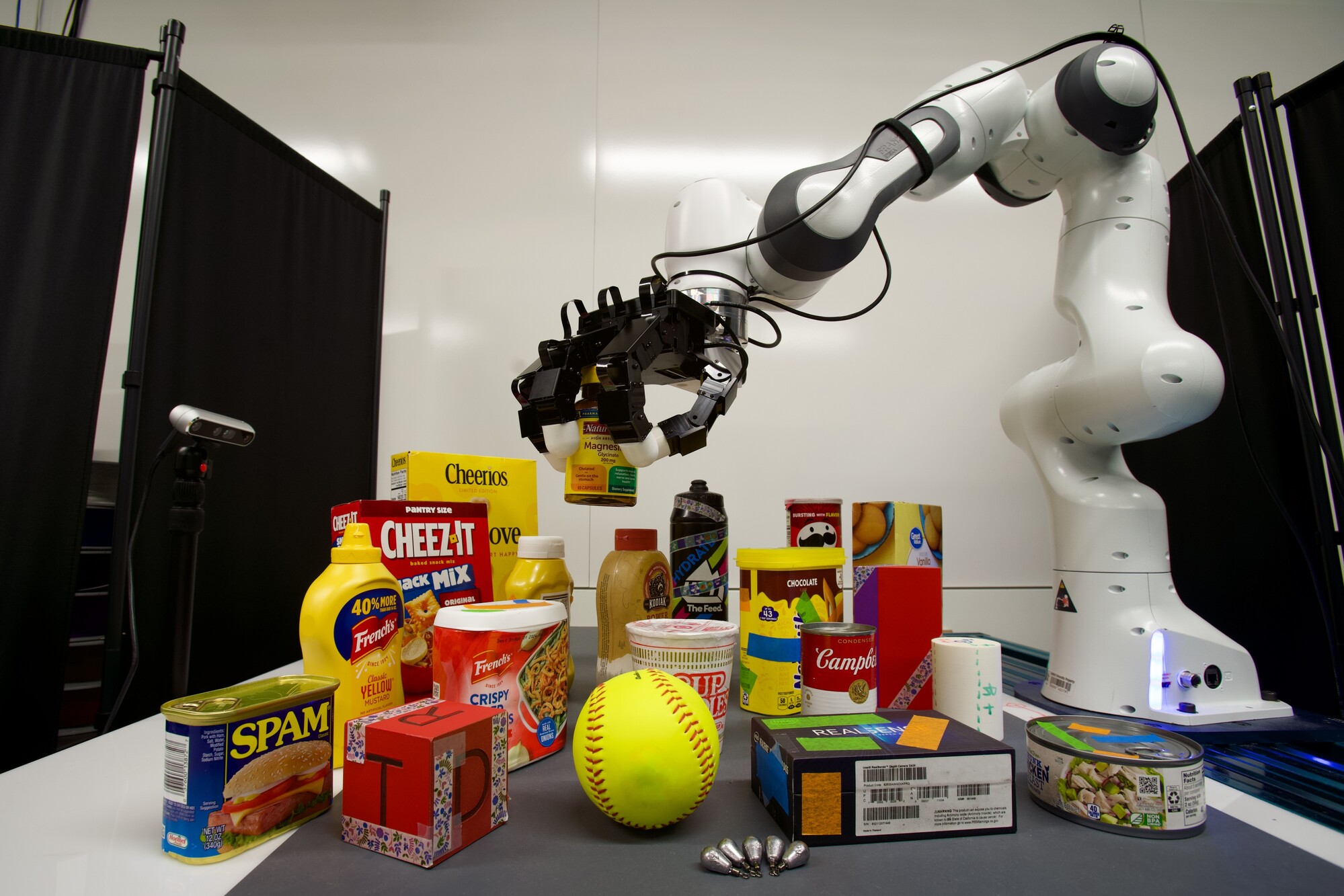}
    \caption{Hardware setup for the dexterous manipulation experiments: a 7-DOF arm equipped with a 16-DOF dexterous hand, a camera for runtime perception, and household objects used for grasping and in-hand reorientation.}
    \label{fig:hardware_setup_img}
\end{figure}

\begin{figure}[!t]
    \centering
    \subfloat{\includegraphics[width=0.23\columnwidth]{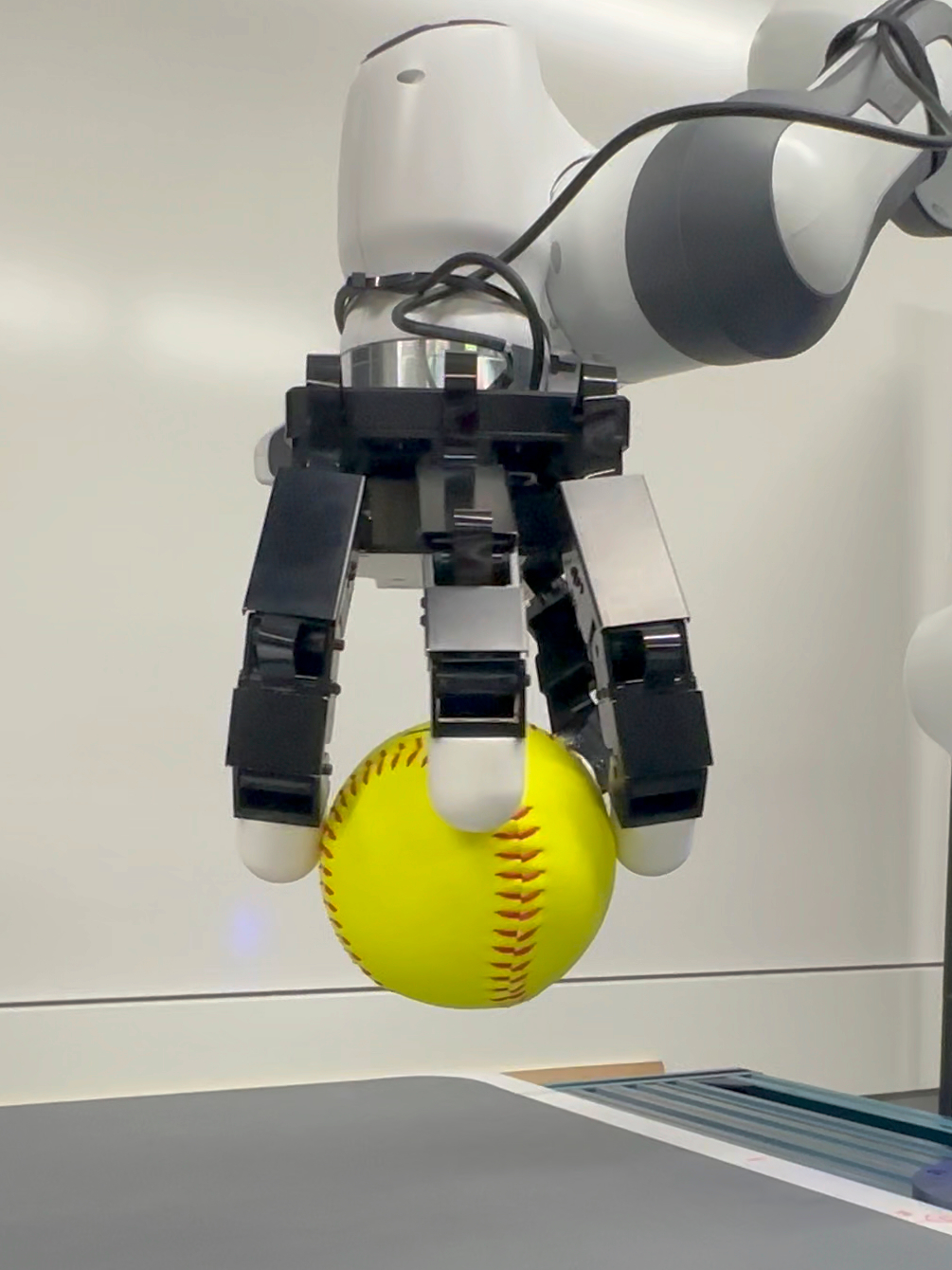}\label{fig:hw_grasp_ball}}
    \hfill
    \subfloat{\includegraphics[width=0.23\columnwidth]{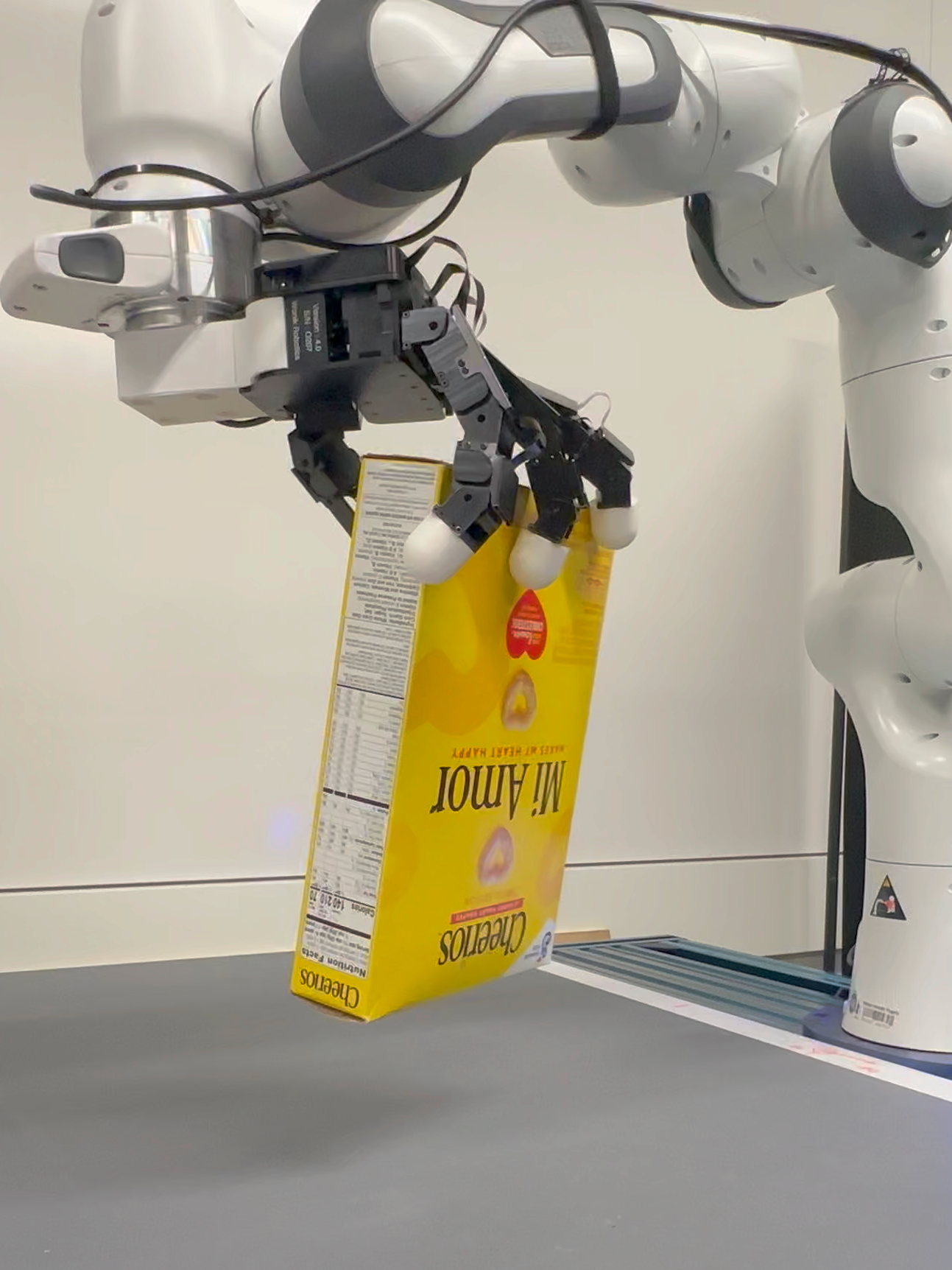}\label{fig:hw_grasp_cereal}}
    \hfill
    \subfloat{\includegraphics[width=0.23\columnwidth]{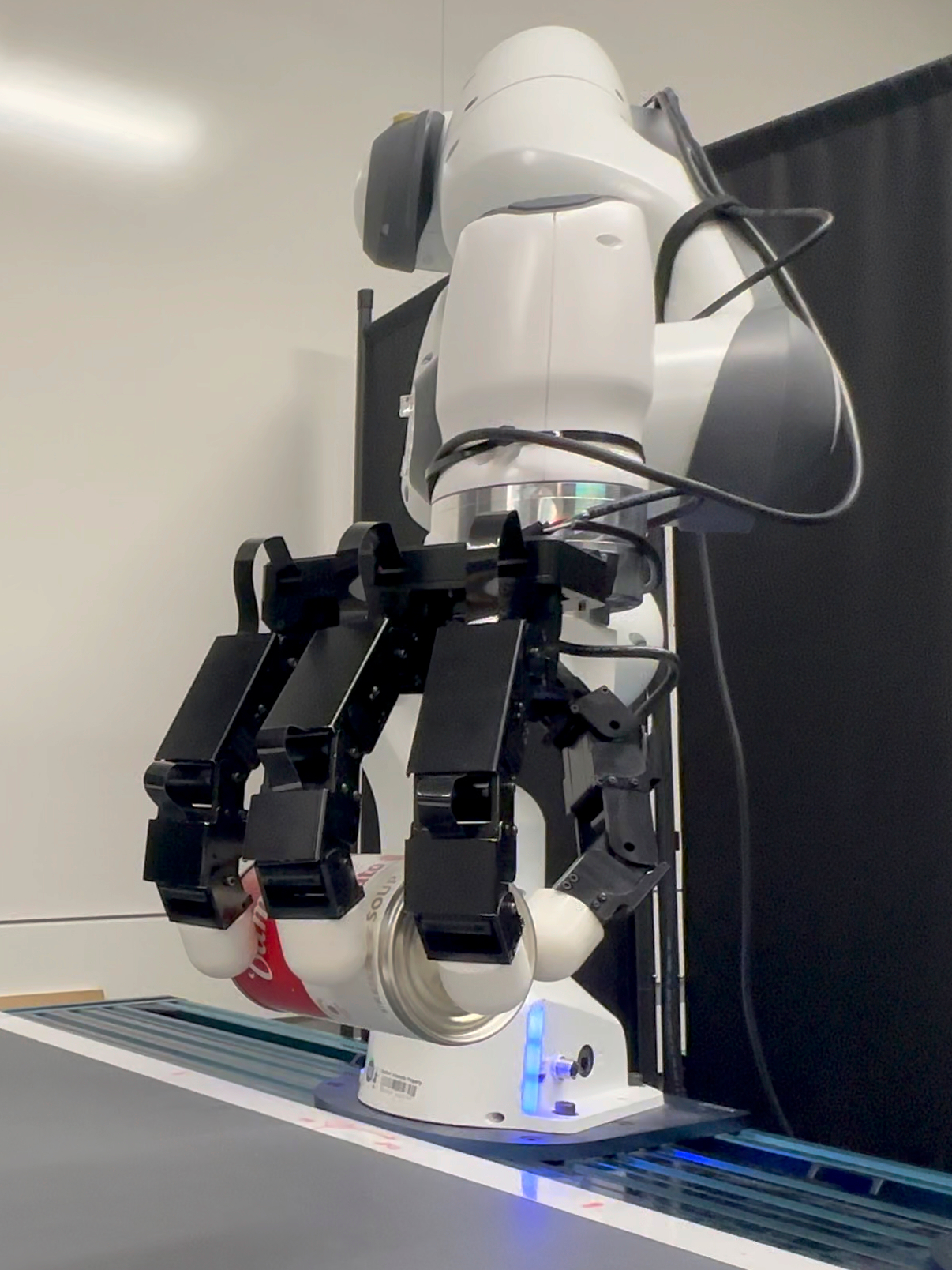}\label{fig:hw_grasp_soup}}
    \hfill
    \subfloat{\includegraphics[width=0.23\columnwidth]{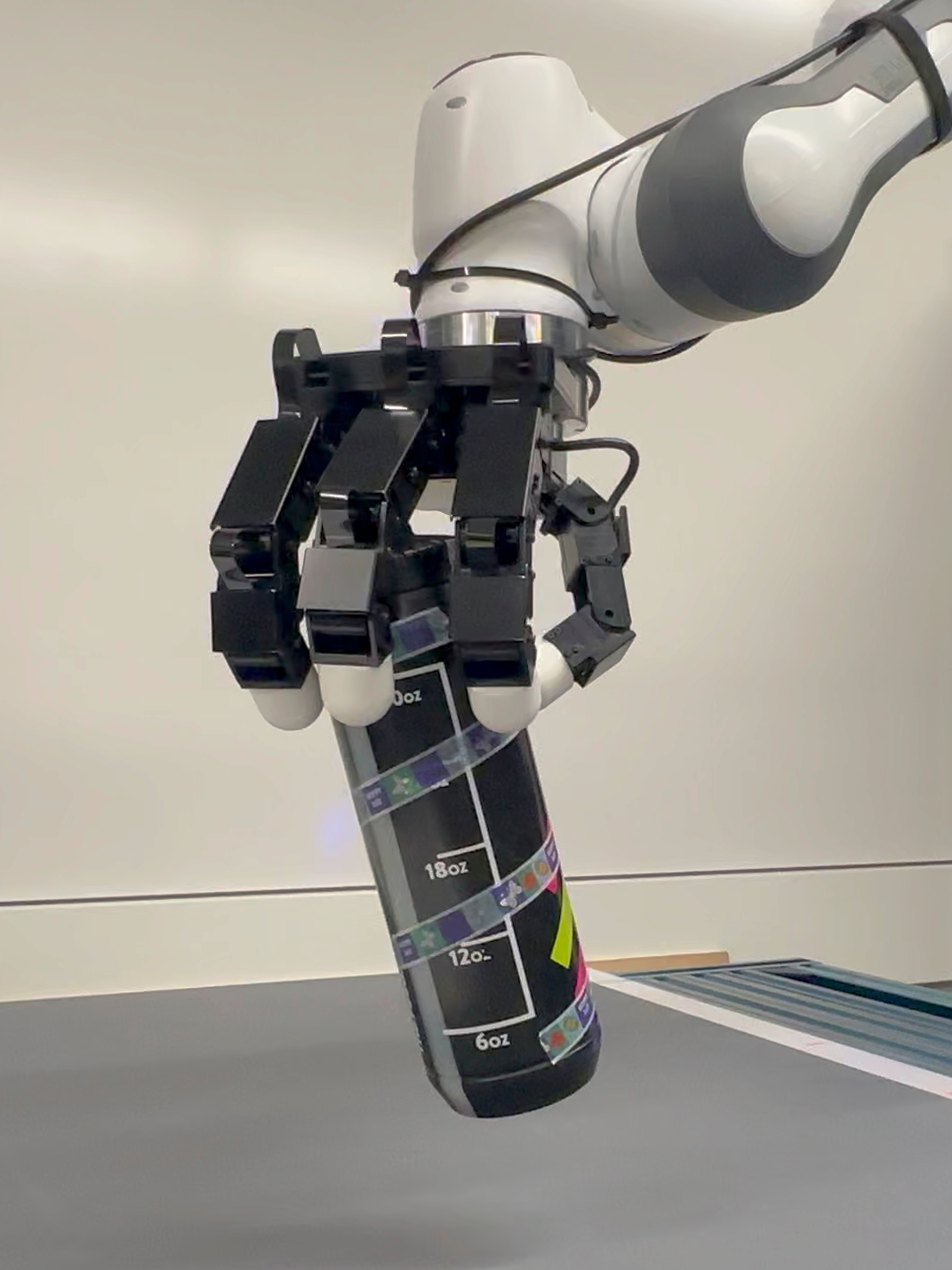}\label{fig:hw_grasp_bottle}}
    \caption{4-finger grasp examples across representative object categories. All four fingers form a force closure grasp under \methodname{}, with the per-finger fingertip-to-object distance and centroid centering action tasks driving the fingertips onto the object surface.}
    \label{fig:hw_grasping_examples}
\end{figure}

\subsection{Manifold Setup and Task Specification}
\label{sec:dex_hw_manifolds}
The configuration manifold $M = \prod_{i=1}^{m}(\sigma^i_-,\, \sigma^i_+)$ is an open subset of $\mathbb{R}^m$ equipped with the flat product metric. The combined arm-hand system used for dexterous grasping has $m = 23$; for in-hand reorientation (Section~\ref{sec:in_hand_reorientation}) the arm is controlled separately and $m = 16$.

The grasping and finger-gaiting behaviors share the same overall structure: a joint space PBDS damping task, a family of pullback ECBF safety tasks, and task manifold action tasks driven by a higher-level planner. The safety tasks comprise (i)~joint limit ECBFs; (ii)~force closure ECBFs based on the $l^*$ metric~\cite{li2023frogger} for one four-finger-grasp variant and four per-finger-excluded variants; (iii)~fingertip lift-distance ECBFs that keep each non-grasping fingertip away from the object surface; (iv)~pairwise finger-finger distance ECBFs that prevent fingertip collisions; (v)~per-fingertip table avoidance ECBFs; and (vi)~link-object distance ECBFs that keep non-fingertip links clear of the object (per-link CBFs for grasping; per-finger distal phalanx ECBFs for in-hand reorientation). The action tasks expose a per-finger controller on the 1D fingertip-to-object distance. For grasping, additional controllers on the 3D palm position and on the 3D average fingertip position. For in-hand reorientation, an additional action task drives the gaited fingertip's 3D position toward a planner-supplied target in the frame of the contact plane spanned by the three in-contact fingertips. In total, $84$ ECBF tasks are instantiated for the grasping policy and $61$ for the in-hand reorientation policy, a subset of which is active at each control step.

A finite-state machine modulates task weights to coordinate the grasp and gaiting phases. It selects the active force closure variant as fingers enter or leave contact, and toggles the per-finger lift-distance CBFs depending on whether each finger is stationary or in transit. The full enumeration of task maps, safety functions, and ECBF gains is given in Appendix~\ref{app:hw_task_spec}.

\subsection{Dexterous Grasping}
\label{sec:dexterous_grasping}
We apply \methodname{} to autonomous grasping of diverse household objects.

\subsubsection{Scene Setup and Test Protocol}
We select 20 household objects (Fig.~\ref{fig:hardware_setup_img}) with mass ranging from $64$ to $613$\,g and longest bounding box dimension ranging from roughly $6$ to $29$\,cm. At the start of each trial, FoundationPose reports the object's 6-DOF pose, and a wrist pose sampler proposes top-down wrist pose candidates that are filtered for reachability, clearance, and aperture fit. The top-ranked candidate is executed in two stages: first, a Deoxys-controlled approach to the wrist pose (via a $15$\,cm waypoint above the pose followed by a direct descent); second, \methodname{} takes over the full $23$-DOF arm-hand system and runs the PBDS QP at each control step. During the \methodname{} phase, two action tasks are composed: per-finger 1D fingertip-to-object distance controllers that drive each fingertip onto the object surface, and a 3D average-fingertip centering controller that aligns the four-fingertip centroid with the object's geometric center, for a total of $4 \times 1 + 3 = 7$ action dimensions. The force closure ECBF (Section~\ref{sec:dex_hw_manifolds}) is activated as the fingertips approach the object, so that the final squeeze converges to a certified force-closed grasp. A trial is recorded as a \emph{success} if the object is successfully lifted with all four fingers in contact, a \emph{partial success} if it is stably lifted with one finger not in contact, and a \emph{failure} otherwise. The candidate filtering pipeline, action gain schedule, and full execution protocol are given in Appendix~\ref{app:grasping_protocol}.

To showcase the flexibility of our framework, we additionally evaluate grasping with one of the four fingers excluded. The active force closure ECBF is switched to the per-finger-excluded variant, so that $l^*$ is computed over the three in-contact fingers, while the excluded finger is held clear of the object by retargeting its fingertip-to-object distance action to a fixed clearance and enabling its fingertip lift-distance ECBF. No other policy modification is required.

\subsubsection{Results}
Figure~\ref{fig:success_weight_d3} summarizes the 4-finger grasping results (per-object trial counts in Table~\ref{tab:grasping_results} of Appendix~\ref{app:grasping_protocol}). Figure~\ref{fig:hw_grasping_examples} shows representative 4-finger executions, and Figure~\ref{fig:hw_3finger_examples} shows the 3-finger variants for each excluded finger.
\begin{figure}[ht]
    \centering
    \setlength{\tabcolsep}{0pt}%
    \begin{tabular}{@{}c@{\hskip 2pt}c@{\hskip 2pt}c@{\hskip 2pt}c@{}}
        \includegraphics[width=0.23\columnwidth]{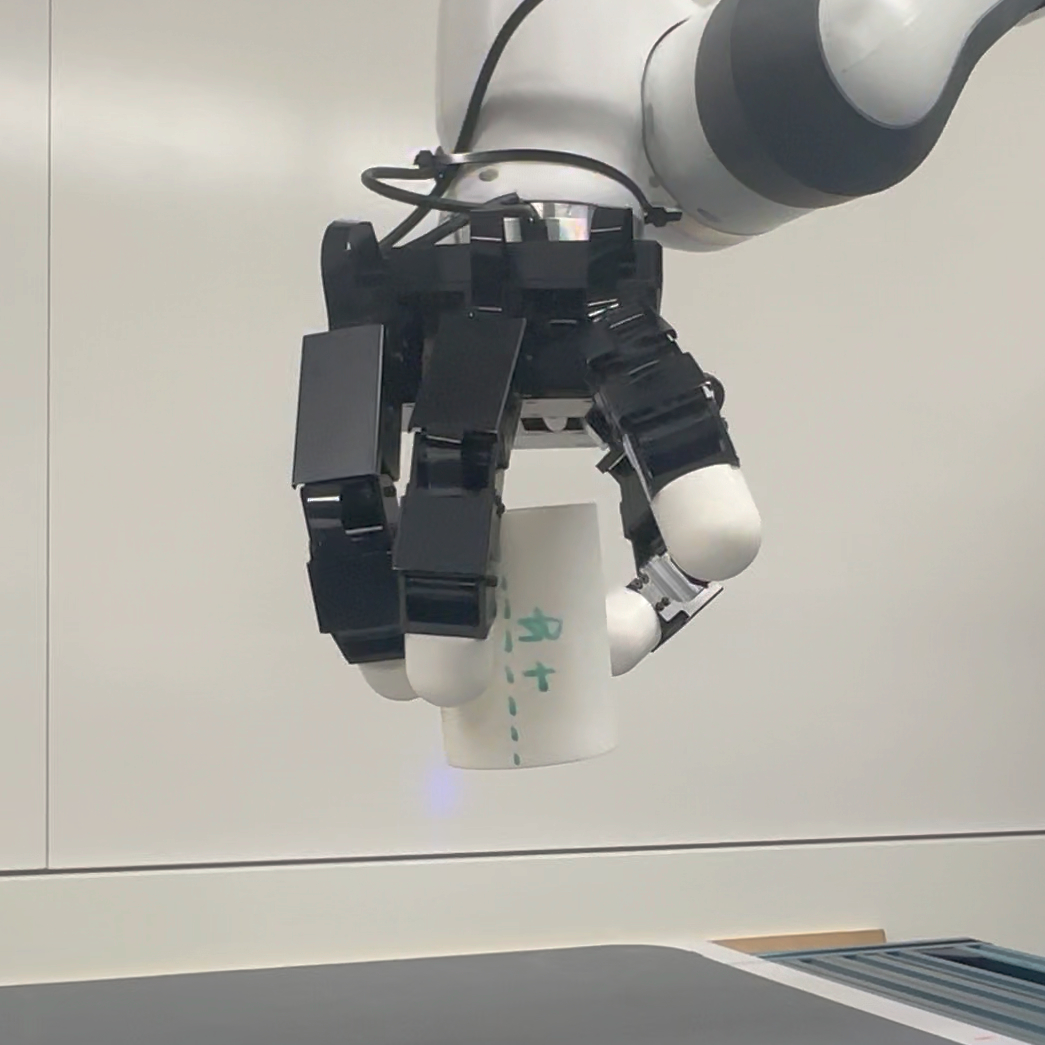}  &
        \includegraphics[width=0.23\columnwidth]{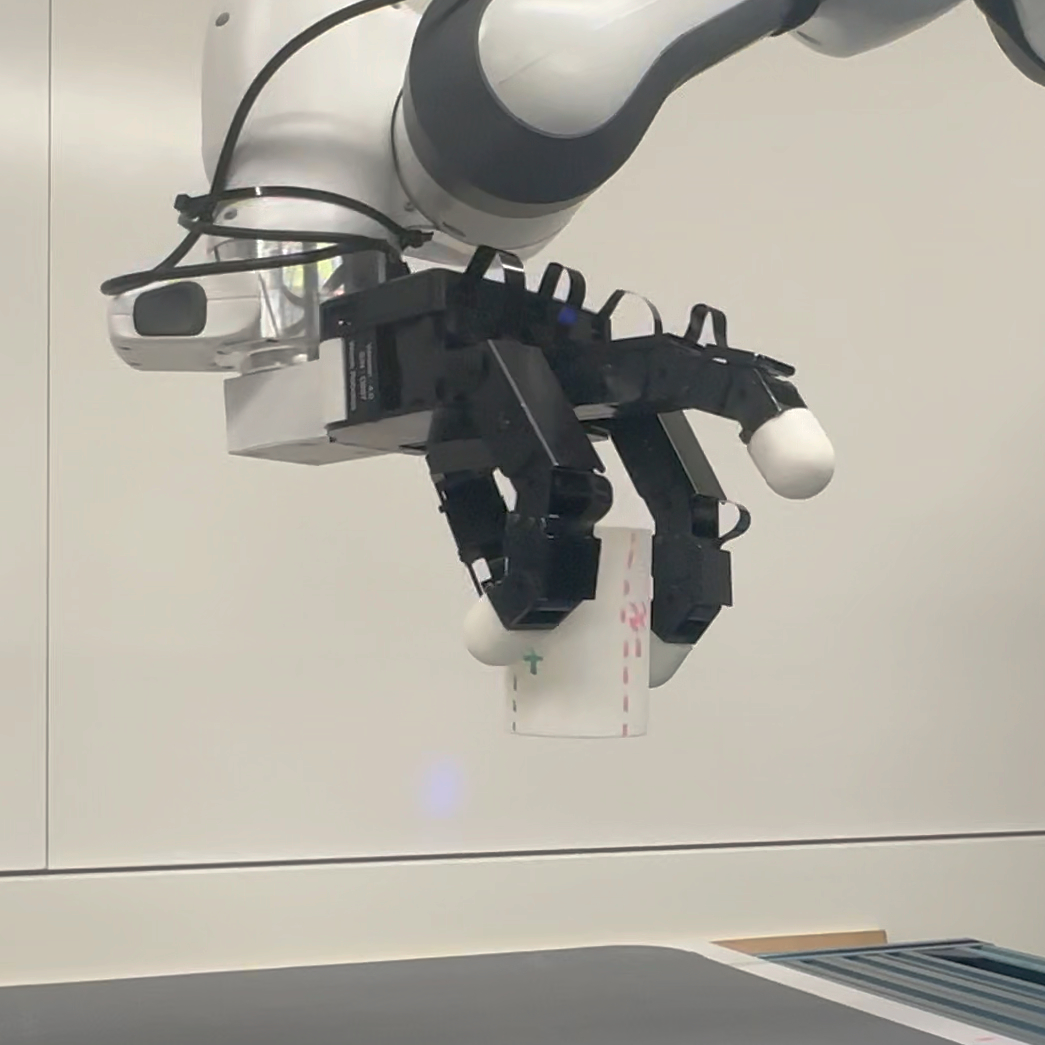} &
        \includegraphics[width=0.23\columnwidth]{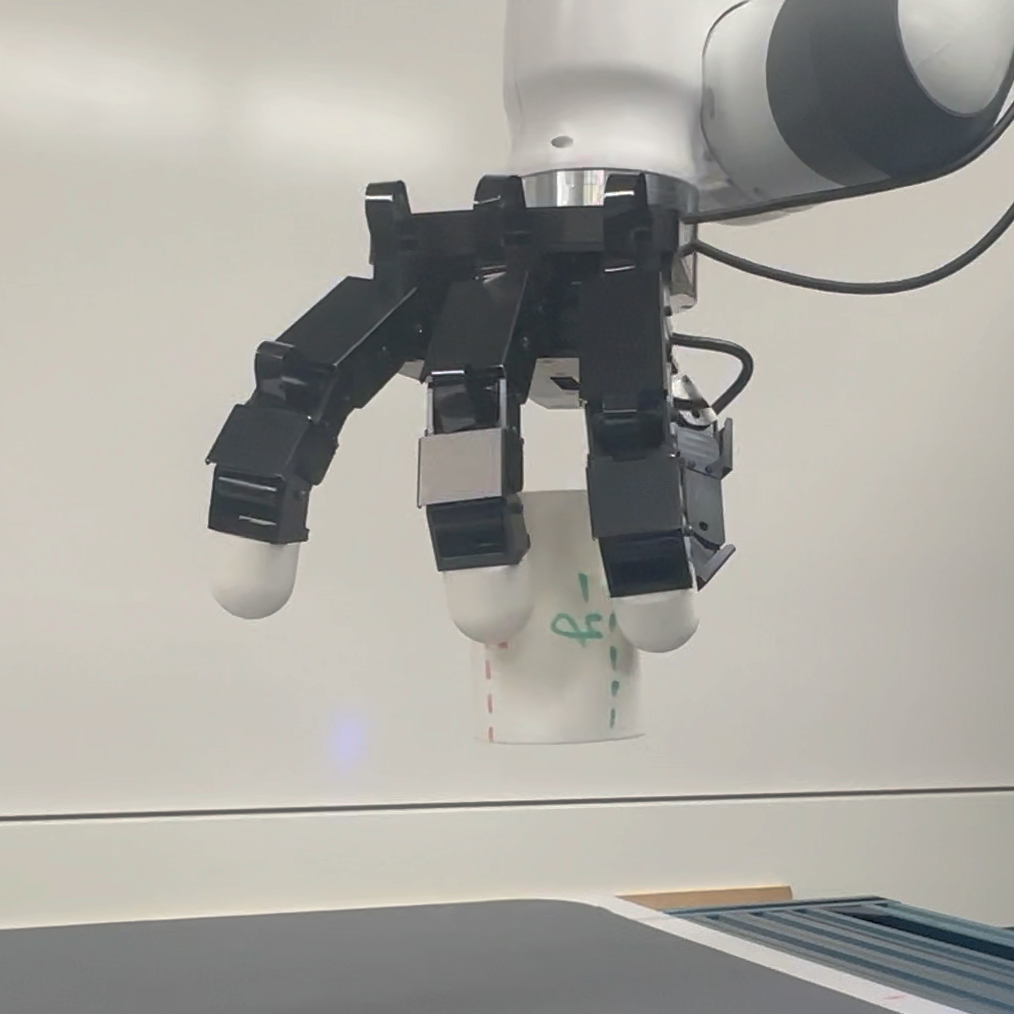}   &
        \includegraphics[width=0.23\columnwidth]{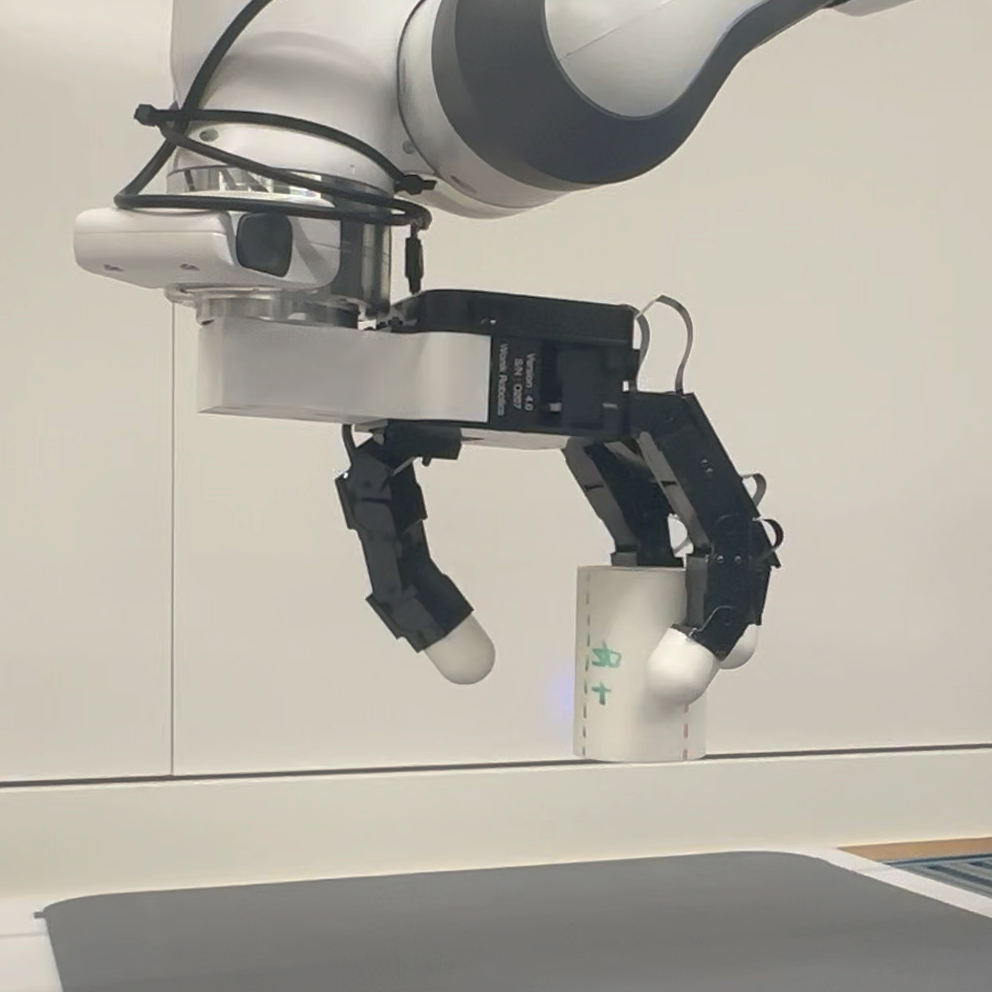}                                                                       \\[2pt]
        \includegraphics[width=0.23\columnwidth]{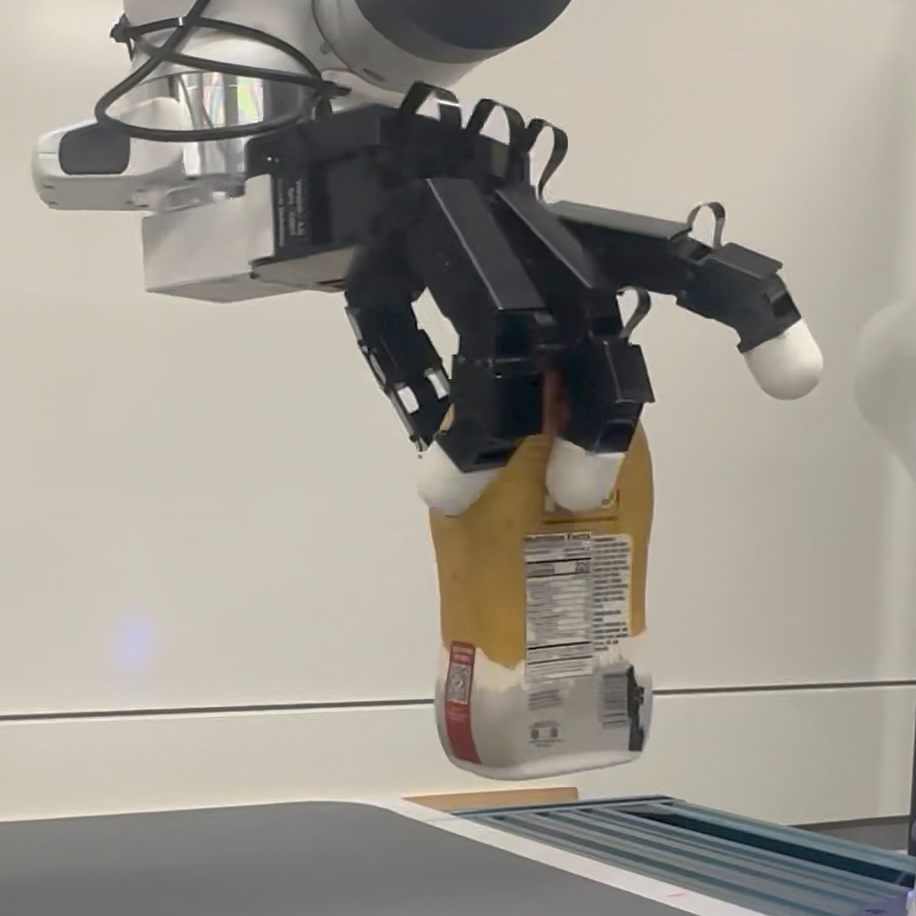}  &
        \includegraphics[width=0.23\columnwidth]{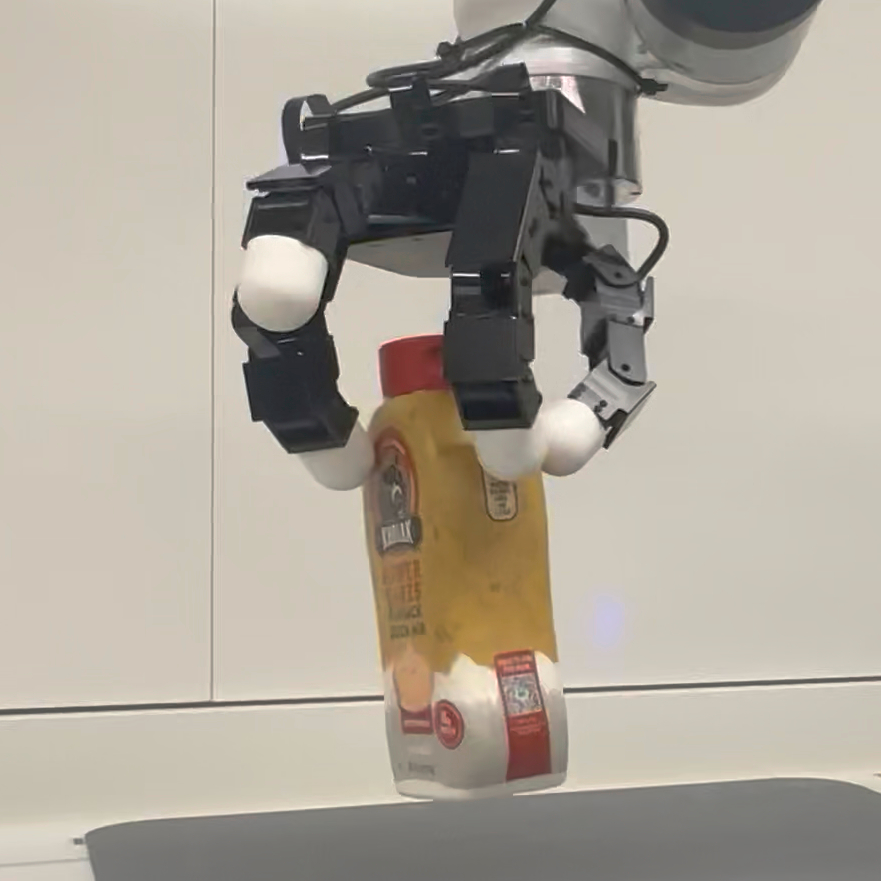} &
        \includegraphics[width=0.23\columnwidth]{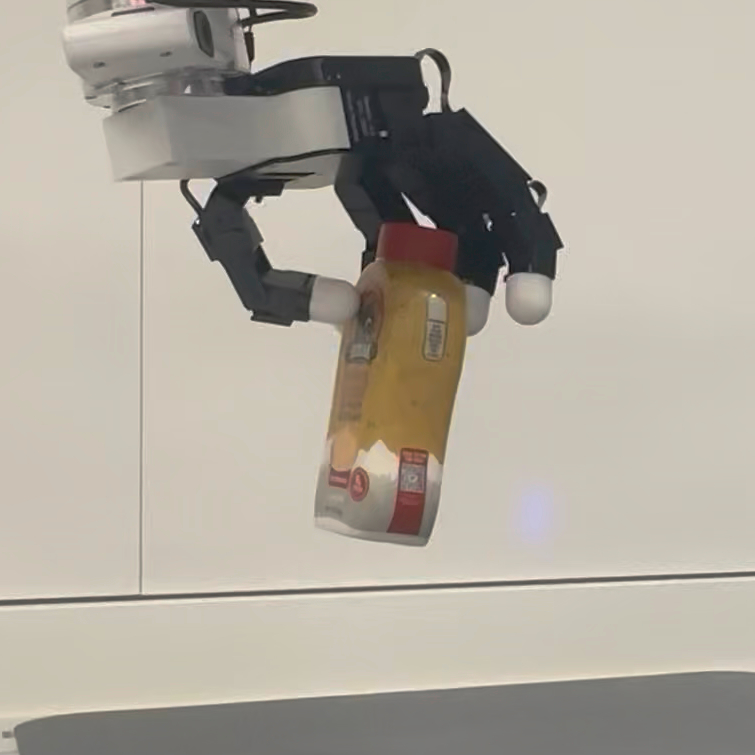}   &
        \includegraphics[width=0.23\columnwidth]{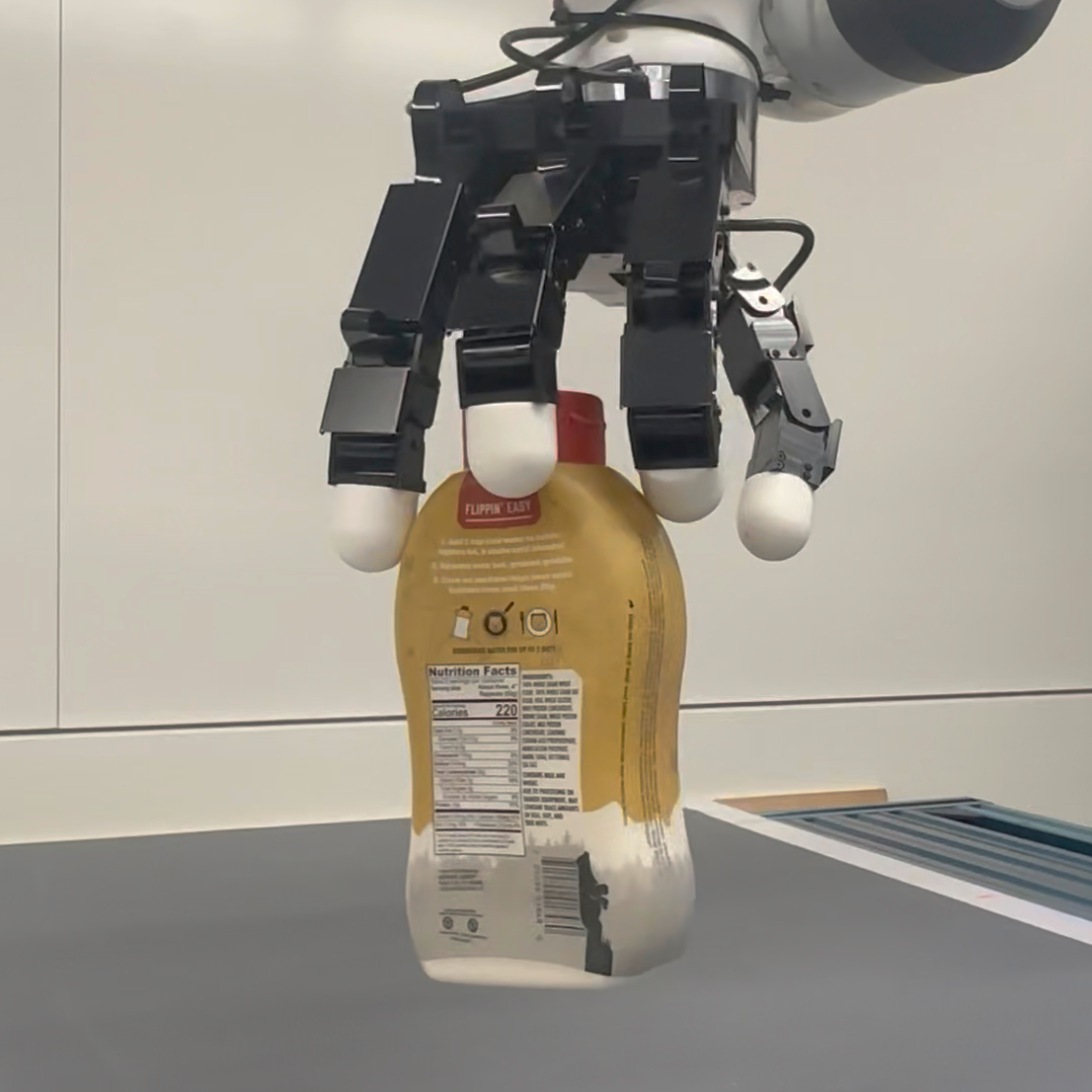}                                                                       \\[2pt]
        \includegraphics[width=0.23\columnwidth]{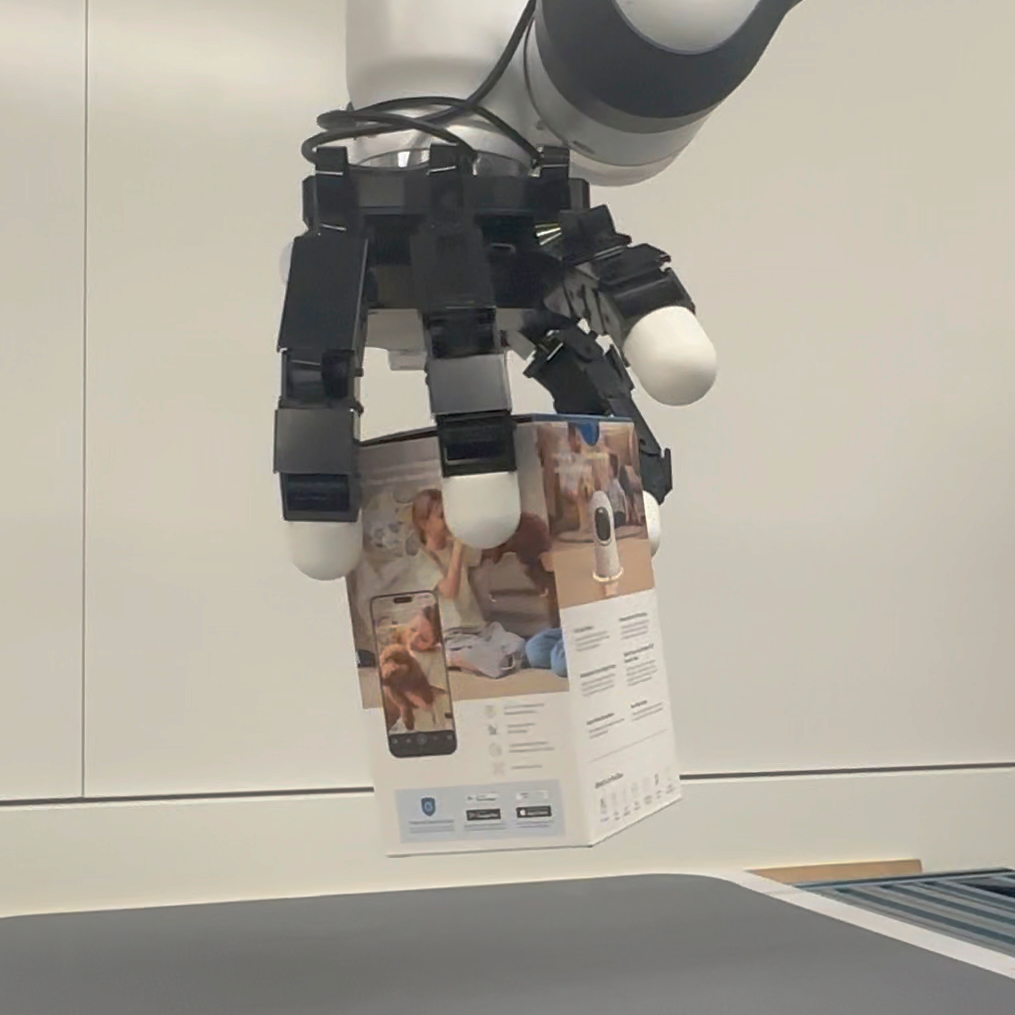}    &
        \includegraphics[width=0.23\columnwidth]{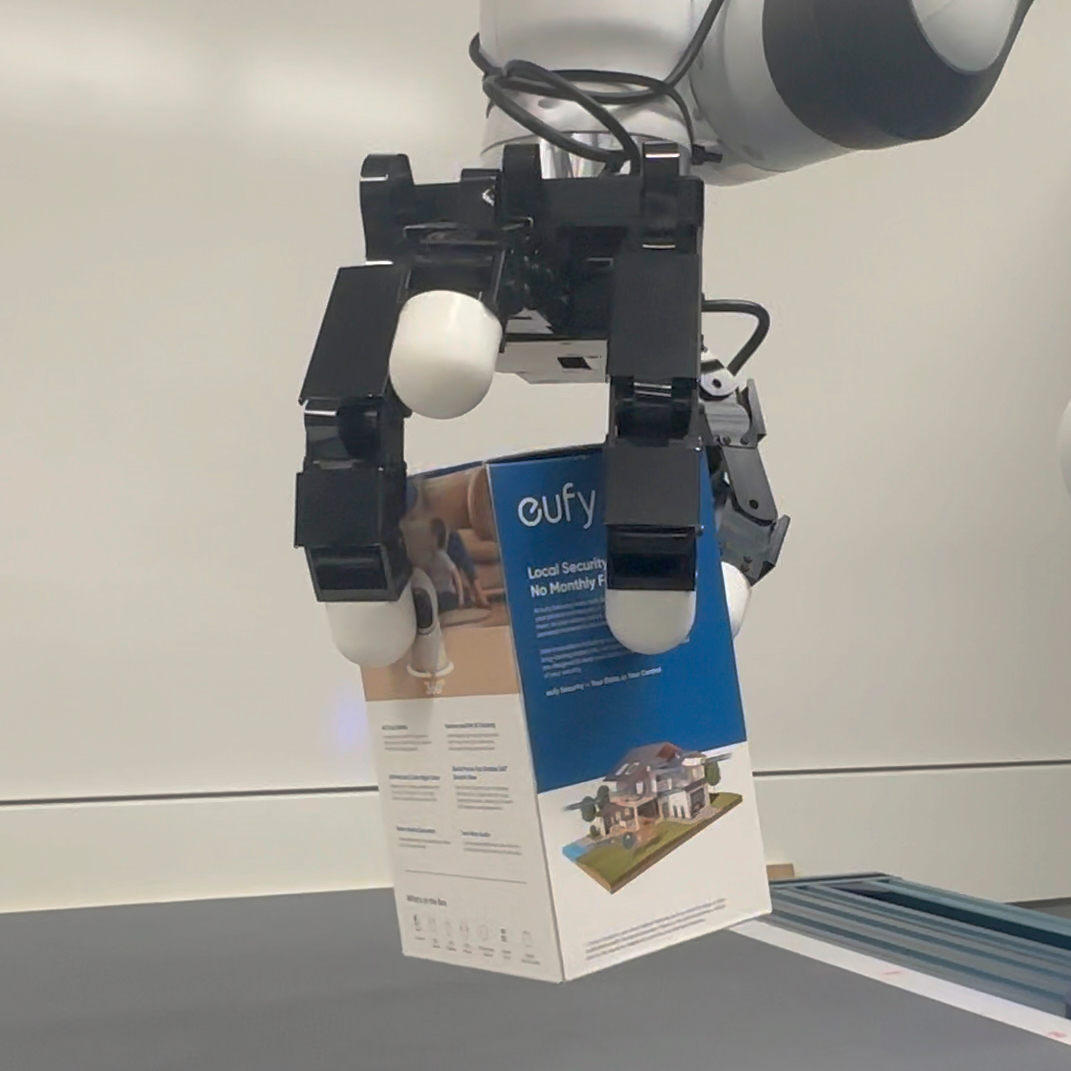}   &
        \includegraphics[width=0.23\columnwidth]{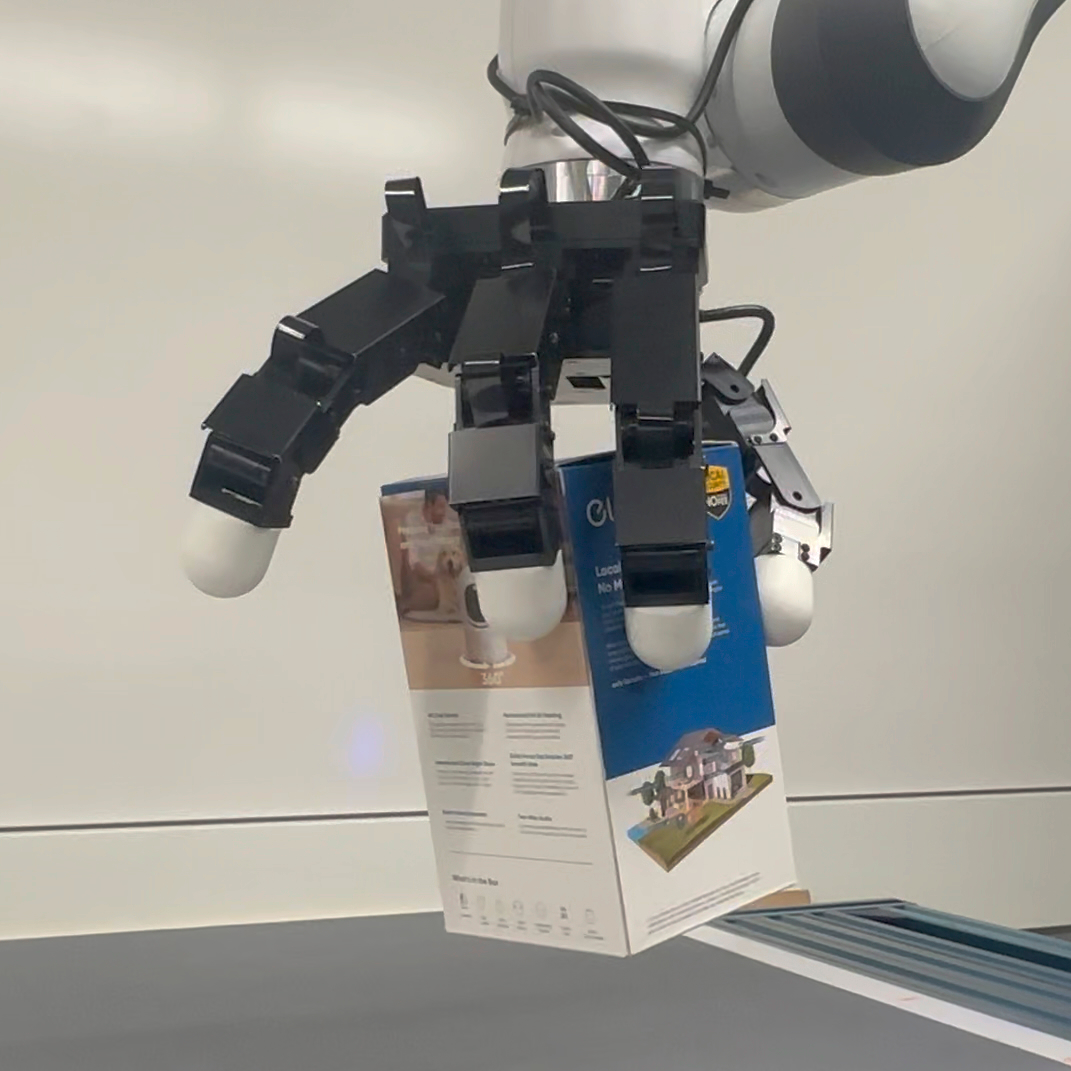}     &
        \includegraphics[width=0.23\columnwidth]{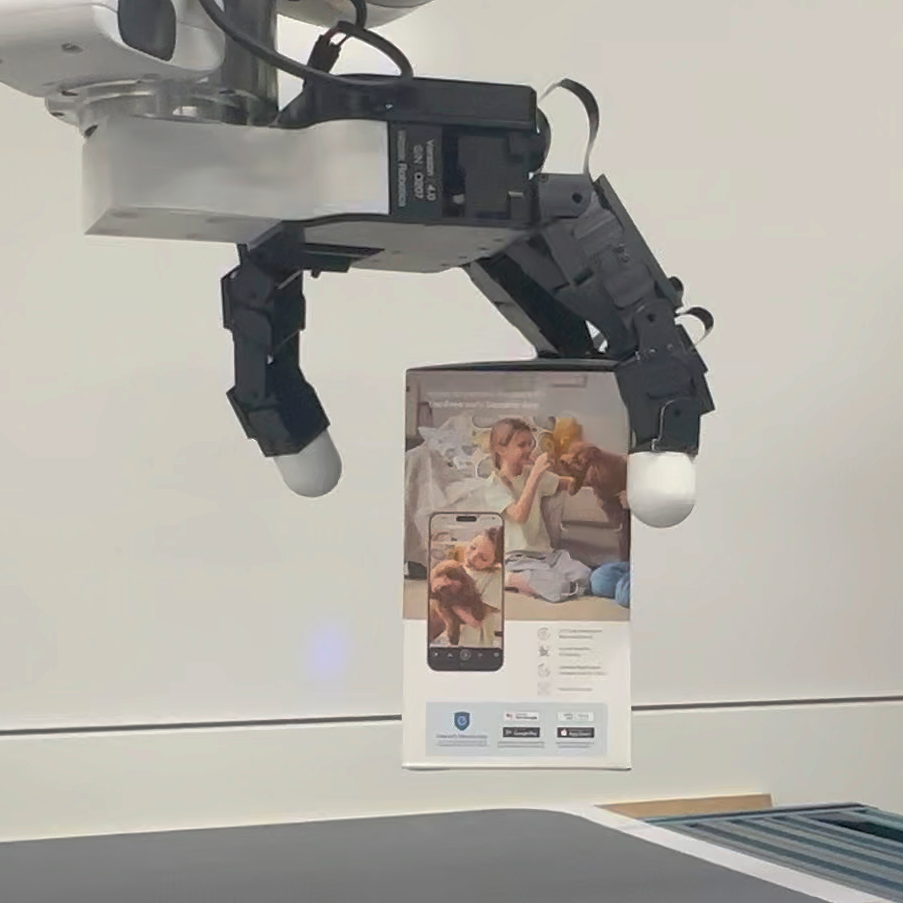}                                                                         \\
        \scriptsize No index                                                                & \scriptsize No middle & \scriptsize No ring & \scriptsize No thumb
    \end{tabular}
    \caption{3-finger grasping variants for the three test objects. Each image is labeled with the excluded finger, which is held clear of the object while the remaining three fingers form a force closure grasp under \methodname{}. Across the three objects and four exclusions, tested at three table locations each ($36$ trials), \methodname{} achieves $34/36$ ($94.4\%$); both failures occur on the wide object (bottom row) with the thumb excluded, where its width forces the remaining index--ring pinch to span near the limits of the hand's reachable workspace. Per-object counts are listed in Appendix Table~\ref{tab:grasping_3finger}.}
    \label{fig:hw_3finger_examples}
\end{figure}

\begin{figure}[ht]
    \centering
    \includegraphics[width=\columnwidth]{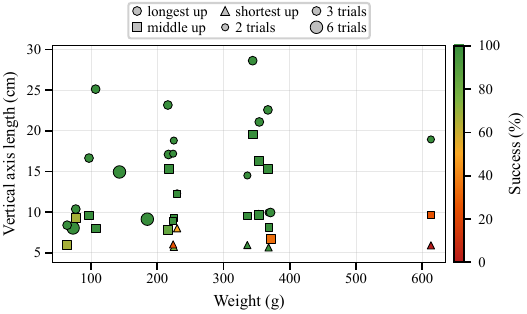}
    \caption{Per-(object, pose) grasp outcomes plotted against object weight and the vertical bounding-box length at the tested pose. Marker color encodes the per-group success score (success $= 1$, partial $= 0.5$, failure $= 0$, averaged over the $2$--$6$ trials in each group, shown on a $0$--$100\%$ scale); shape encodes which bounding-box axis is oriented upward (longest, middle, or shortest); and size is proportional to trial count.}
    \label{fig:success_weight_d3}
\end{figure}

\methodname{} achieves an overall success rate of $111/120$ ($92.5\%$) on the $20$ household objects, with $15$ of $20$ attaining full $6/6$ success. Figure~\ref{fig:success_weight_d3} plots each (object, pose) trial group against object weight and the vertical bounding-box length at that pose. Partial successes and failures concentrate among objects with small vertical extents (roughly $\le 10$\,cm) and high weight. Low vertical extent is less forgiving of pose estimation errors and leaves less table clearance, while high weight requires grasp forces that approach the limits of the low-level impedance controller and the hand's hardware capabilities.

The 3-finger ablation (Figure~\ref{fig:hw_3finger_examples}) succeeds in $34/36$ trials ($94.4\%$): thanks to the autonomous PBDS behavior, a high-level decision policy can reliably choose which finger to exclude despite the complexity of dexterous precision grasping.

\subsection{Palm-Down In-Hand Reorientation (IHR)}
\label{sec:in_hand_reorientation}

We apply \methodname{} to palm-down IHR, where the goal is to rotate a grasped object about the palm normal axis with the palm normal facing down. This is significantly more challenging than the more common ``IHR with palm facing up'' in literature (e.g.~\cite{andrychowicz2020learning,handa2023dextreme,qi2023general,yin2023rotating}); with the palm facing down, the object must be securely grasped throughout the process. To tackle the combinatorial complexity of finger gaiting sequences, we leverage offline planning in simulation to discover feasible reorientation trajectories, which are then deployed on hardware.

\subsubsection{Simulation Pre-Planning}
The main challenge in IHR is choosing which finger to move and where to move it. We address this with a depth-first tree search over finger-object contact states, where each node has all four fingers in contact and each edge corresponds to relocating one finger. The search tree is rooted at an initial grasp obtained by reusing the top-down grasping pipeline of Section~\ref{sec:dexterous_grasping} on a $6$\,cm diameter bottle. A priority queue at each depth orders nodes by the cumulative $z$-axis yaw rotation achieved along the path from the root, while a balance constraint caps the disparity in per-finger move counts so that no single finger is overused.

For each movable finger at a node, $12$ candidate extensions are generated by a predetermined grid of steps along the object surface. Fingers are not allowed to move consecutively and the max gaiting step count difference across fingers is 1. This leaves at most $3$ movable fingers and a maximum branching factor of $36$. To prevent the search from getting stuck with a never-movable finger, a cyclic sequence of ``next movable'' fingers is assigned so that at each tree depth, the ``next movable'' finger must be movable after the candidate extension.  Each candidate is forward-simulated under \methodname{} and a four-phase (\texttt{LIFTING}, \texttt{TRAVERSING}, \texttt{DROPPING}, \texttt{ADJUSTING}) primitive, and an accepted candidate becomes a tree edge. Throughout the primitive, the three-finger grasp is maintained by the corresponding force closure ECBF and pairwise fingertip spacing ECBFs (the object itself may still move). In \texttt{TRAVERSING} and \texttt{DROPPING}, the force closure ECBF excluding the next movable finger is additionally activated. A candidate is accepted when, on reaching \texttt{ADJUSTING}, the moving finger is aligned with its target contact, the hand is at rest, all four fingers are back in contact, the force closure ECBF excluding the next movable finger is non-negative, and the object tilt from the rotation axis is below $10^\circ$. Loss of contact at any in-contact finger during the rollout aborts execution and rejects the candidate. A simplified illustration is shown in Figure~\ref{fig:headline}. More details are given in Appendix~\ref{app:ihr_planner}. Figure~\ref{fig:inhand_trees} shows the resulting search trees.

\setcounter{figure}{8}
\begin{figure*}[!b]
    \centering
    \subfloat[Clockwise rotation, static arm]{%
        \includegraphics[width=0.14\textwidth]{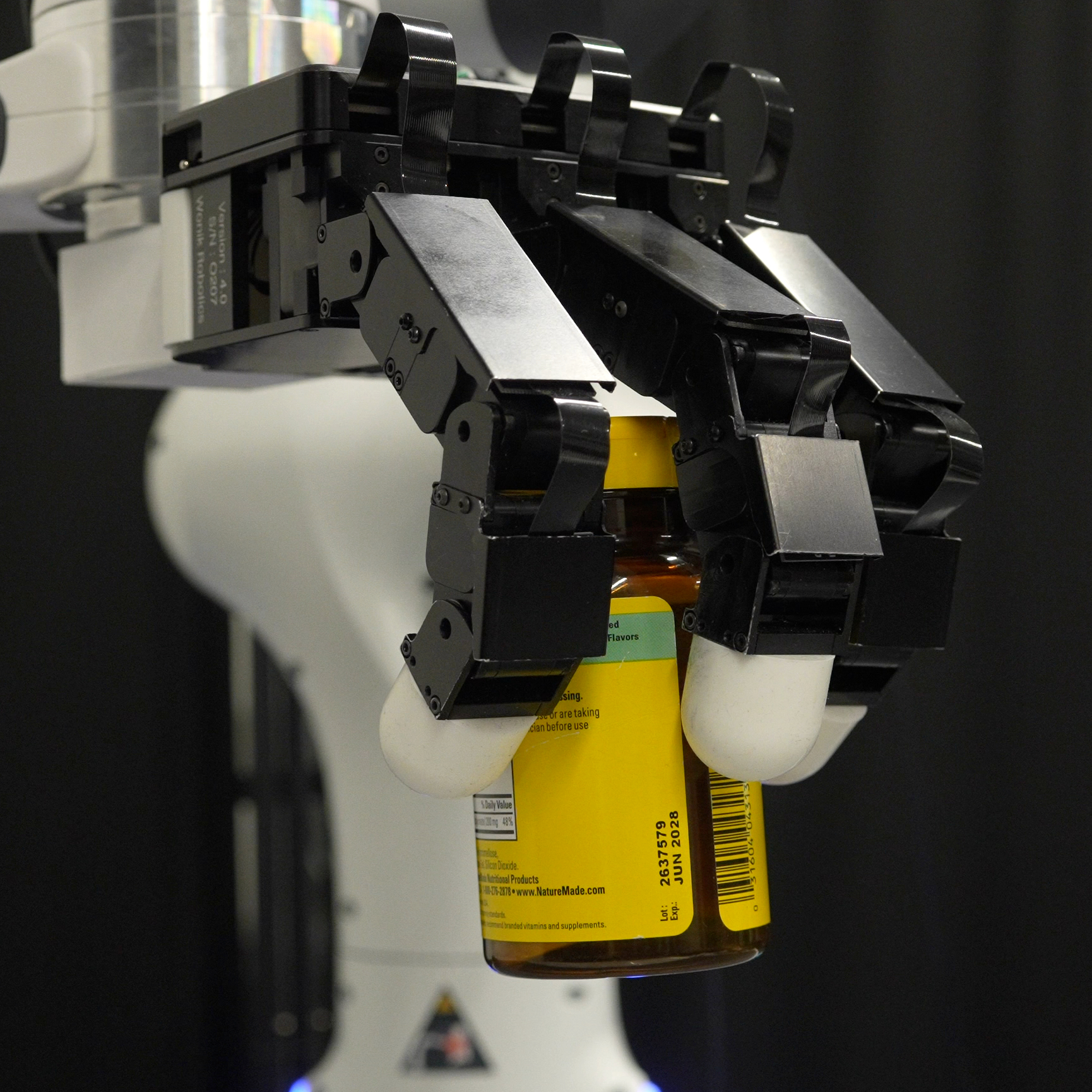}\hspace{2pt}%
        \includegraphics[width=0.14\textwidth]{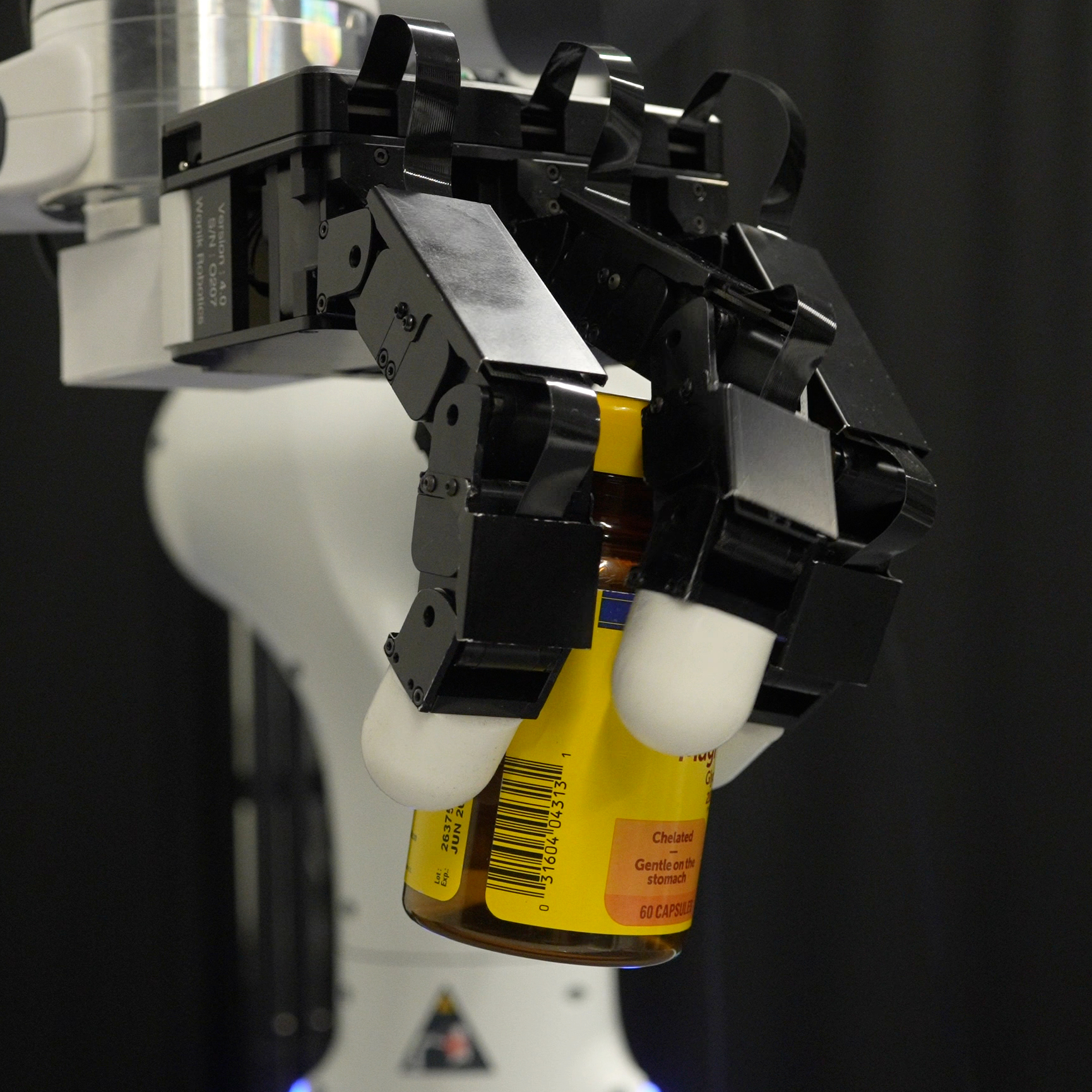}\hspace{2pt}%
        \includegraphics[width=0.14\textwidth]{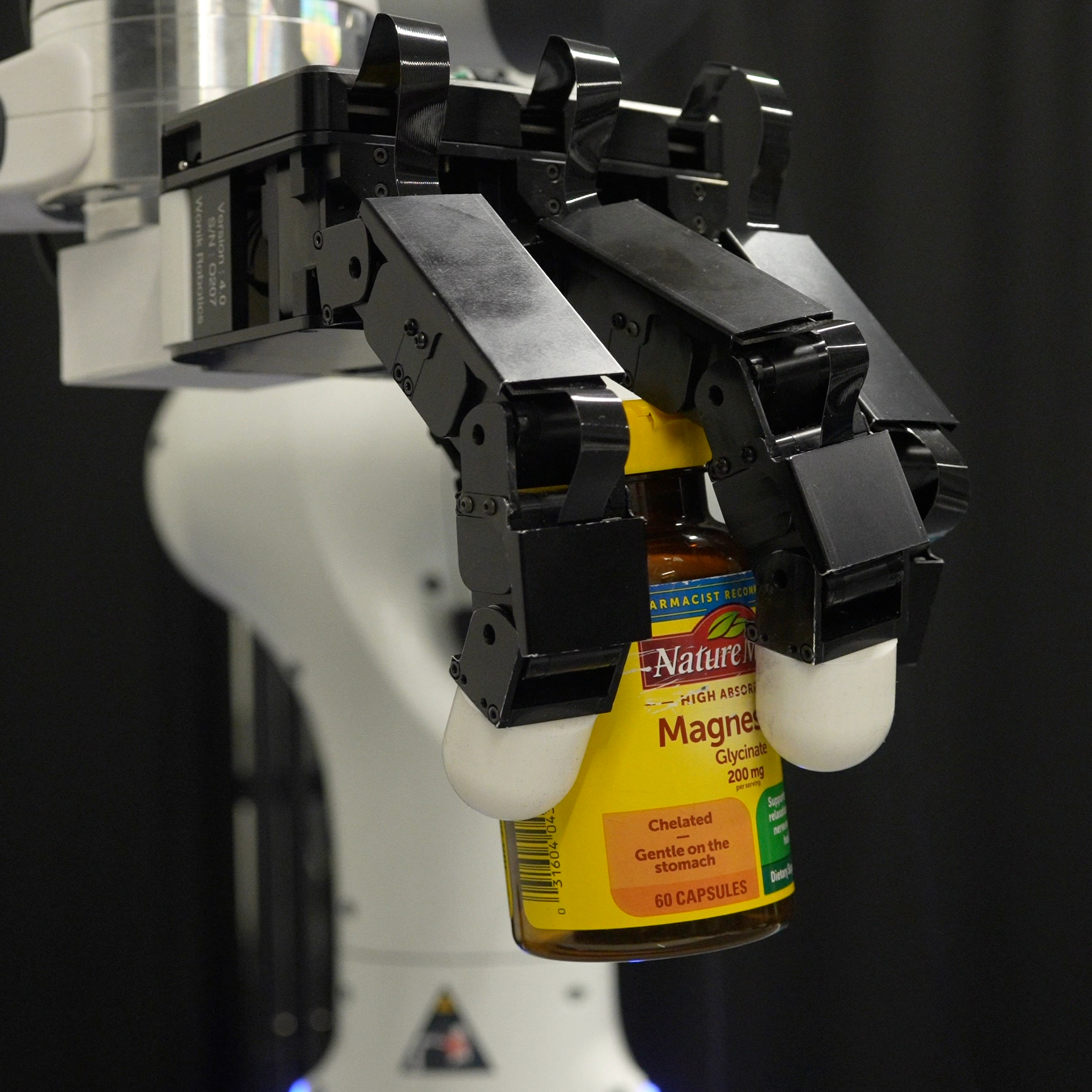}\hspace{2pt}%
        \includegraphics[width=0.14\textwidth]{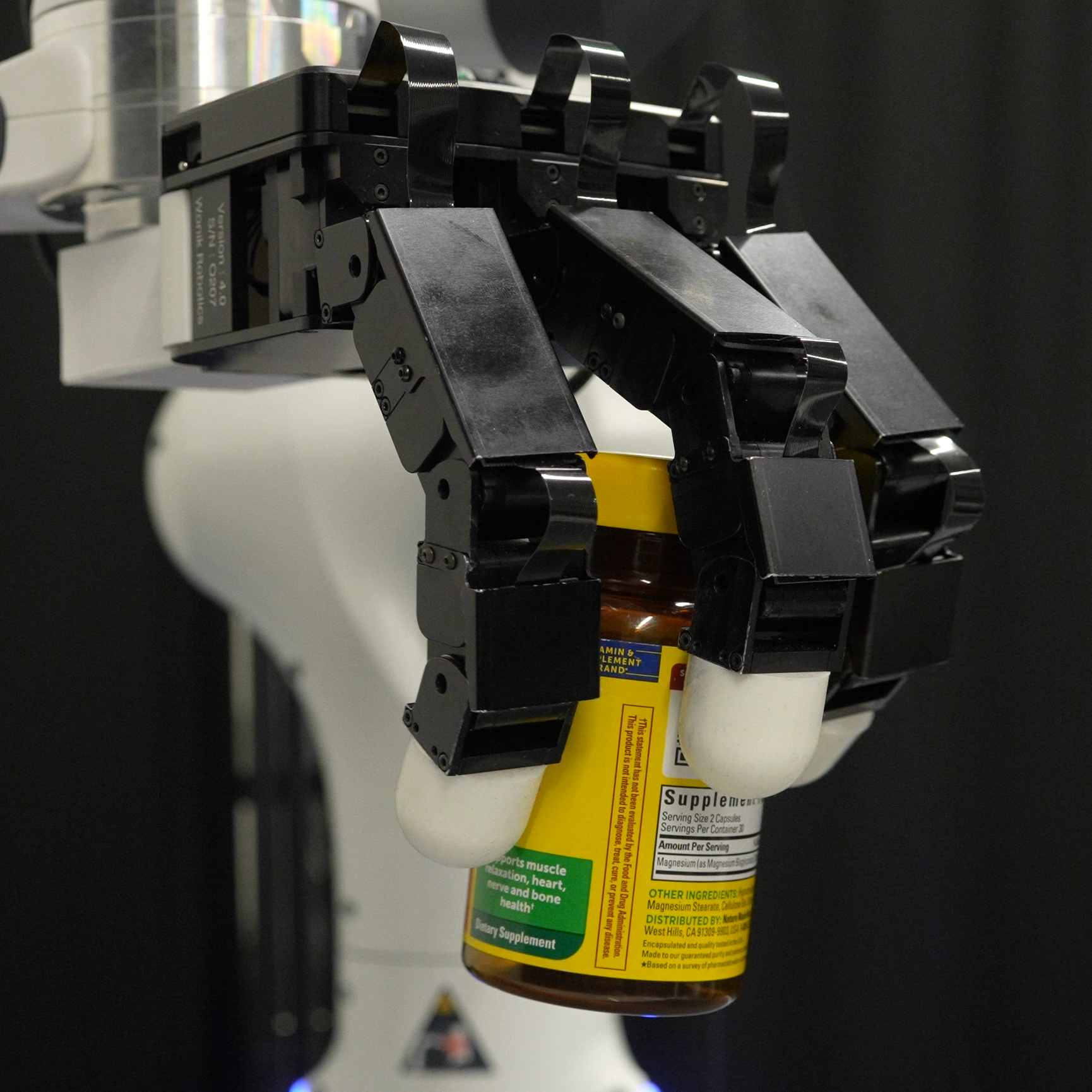}\hspace{2pt}%
        \includegraphics[width=0.14\textwidth]{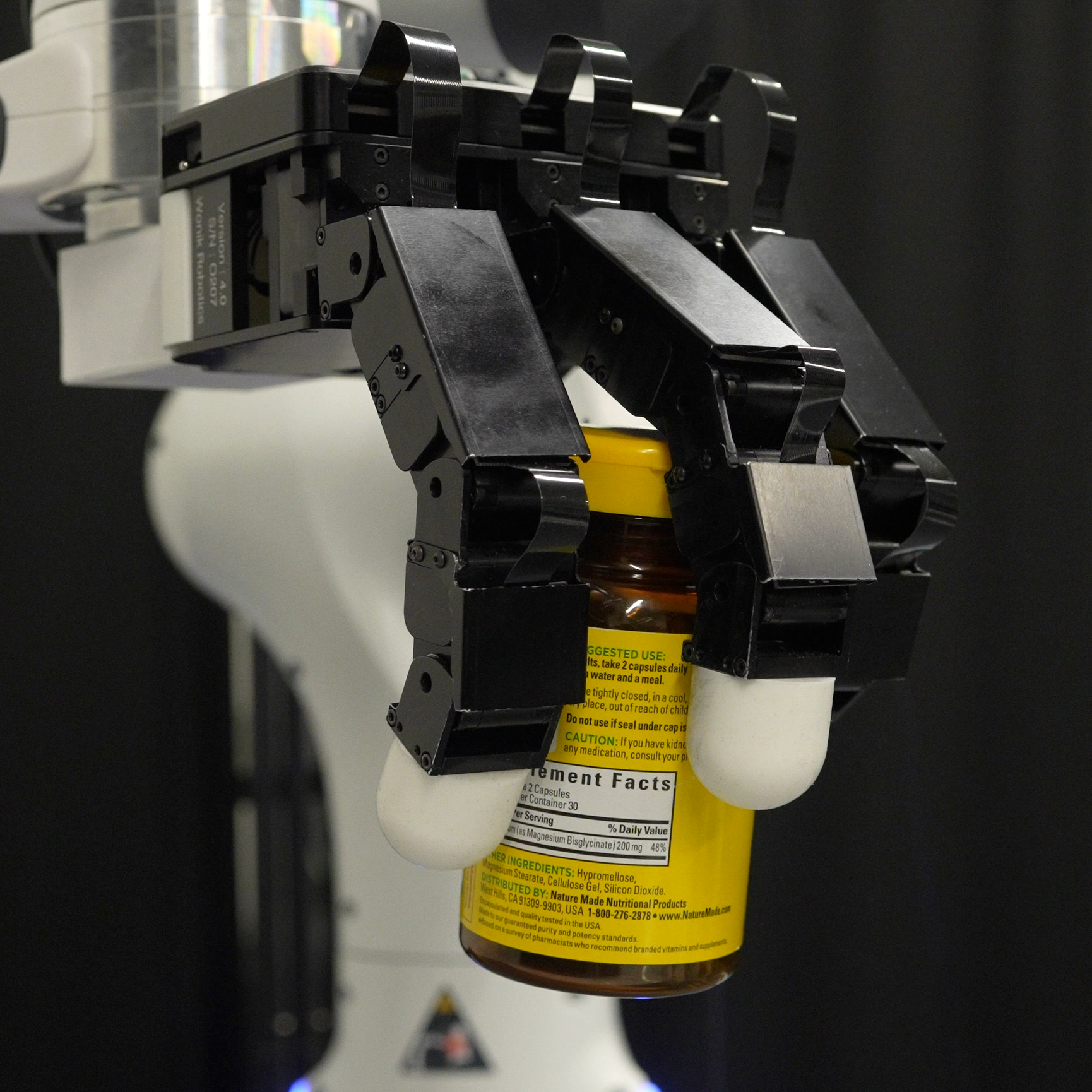}\hspace{2pt}%
        \includegraphics[width=0.14\textwidth]{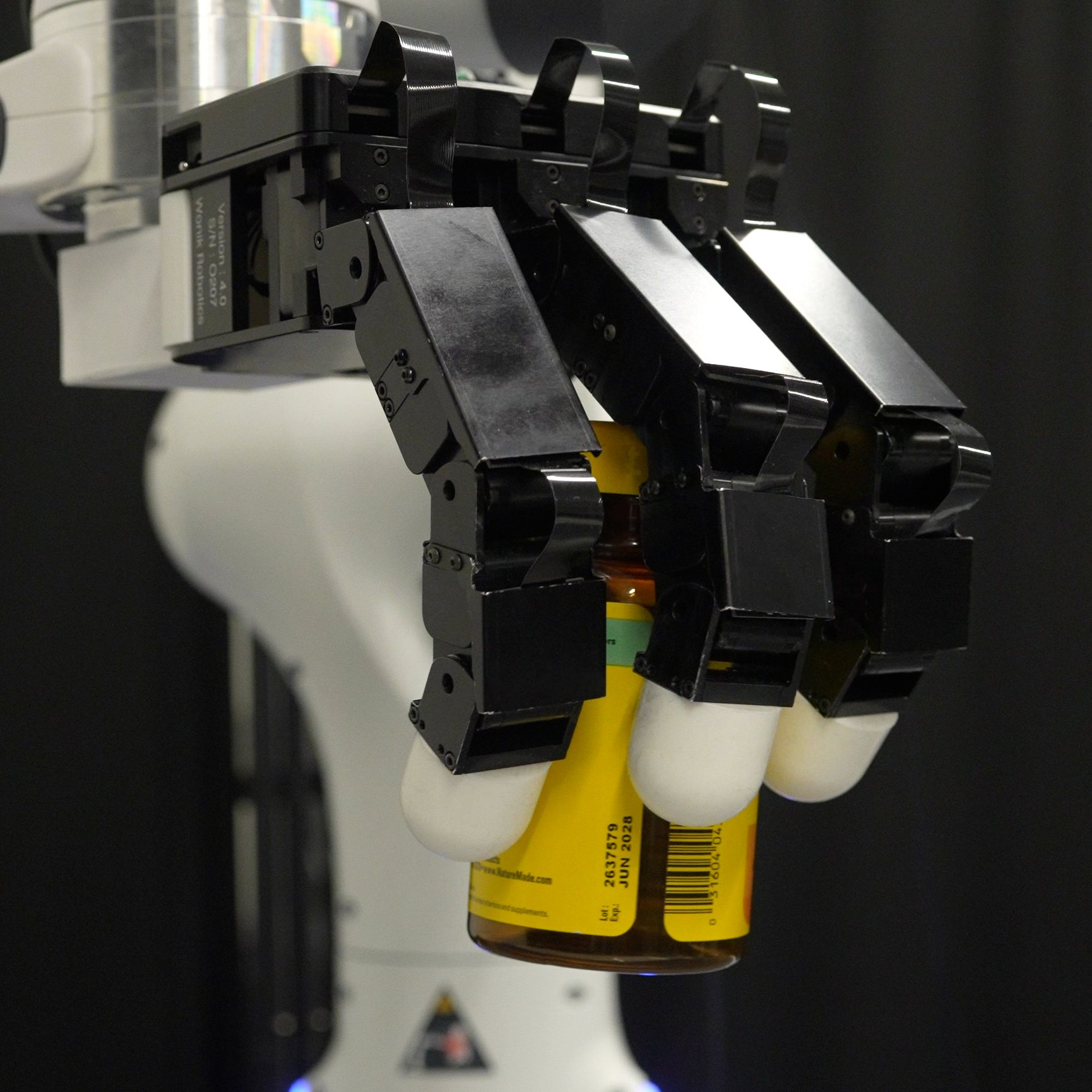}%
        \label{fig:inhand_cw}
    }\\
    \subfloat[Counterclockwise rotation, swinging arm with weighted object]{%
        \includegraphics[width=0.14\textwidth]{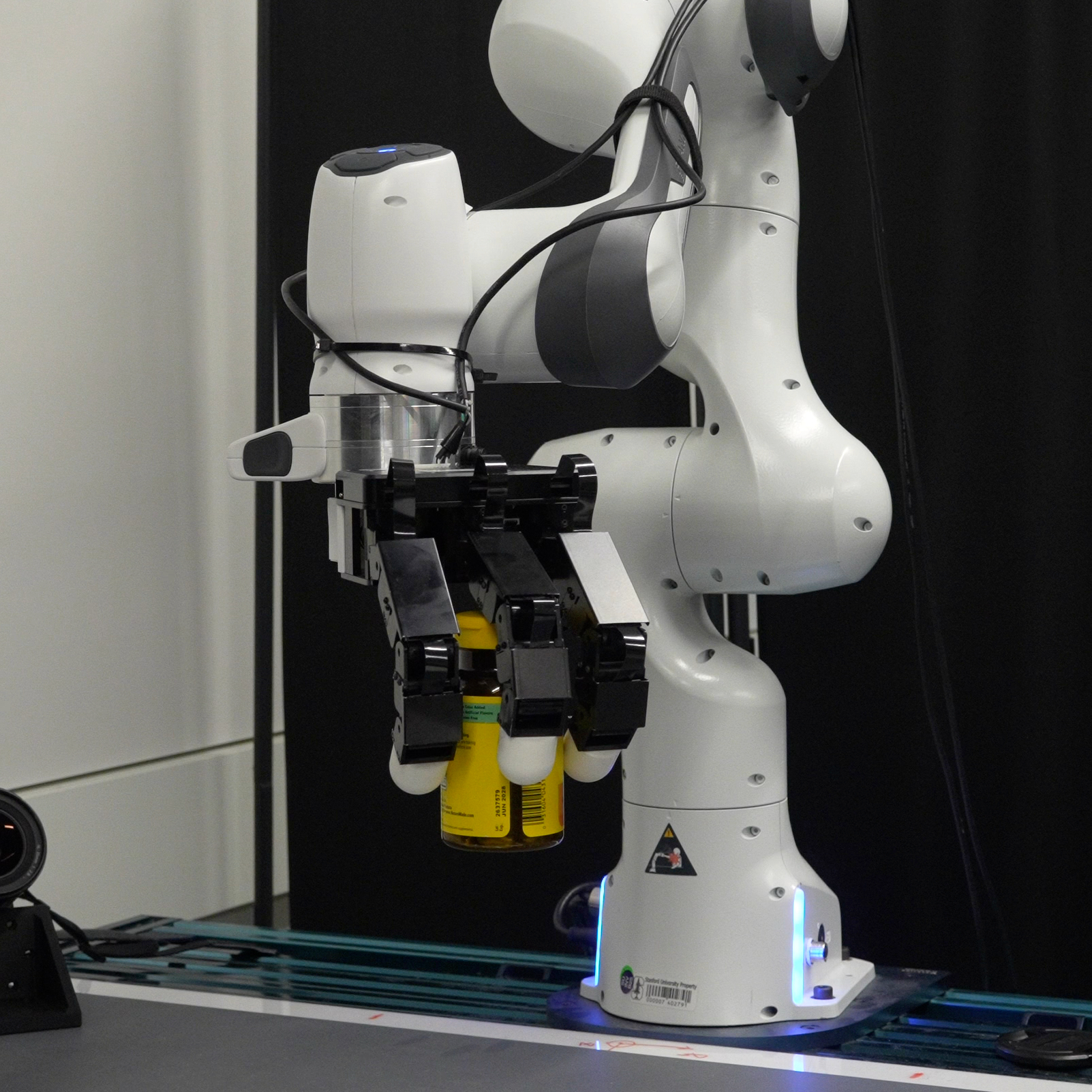}\hspace{2pt}%
        \includegraphics[width=0.14\textwidth]{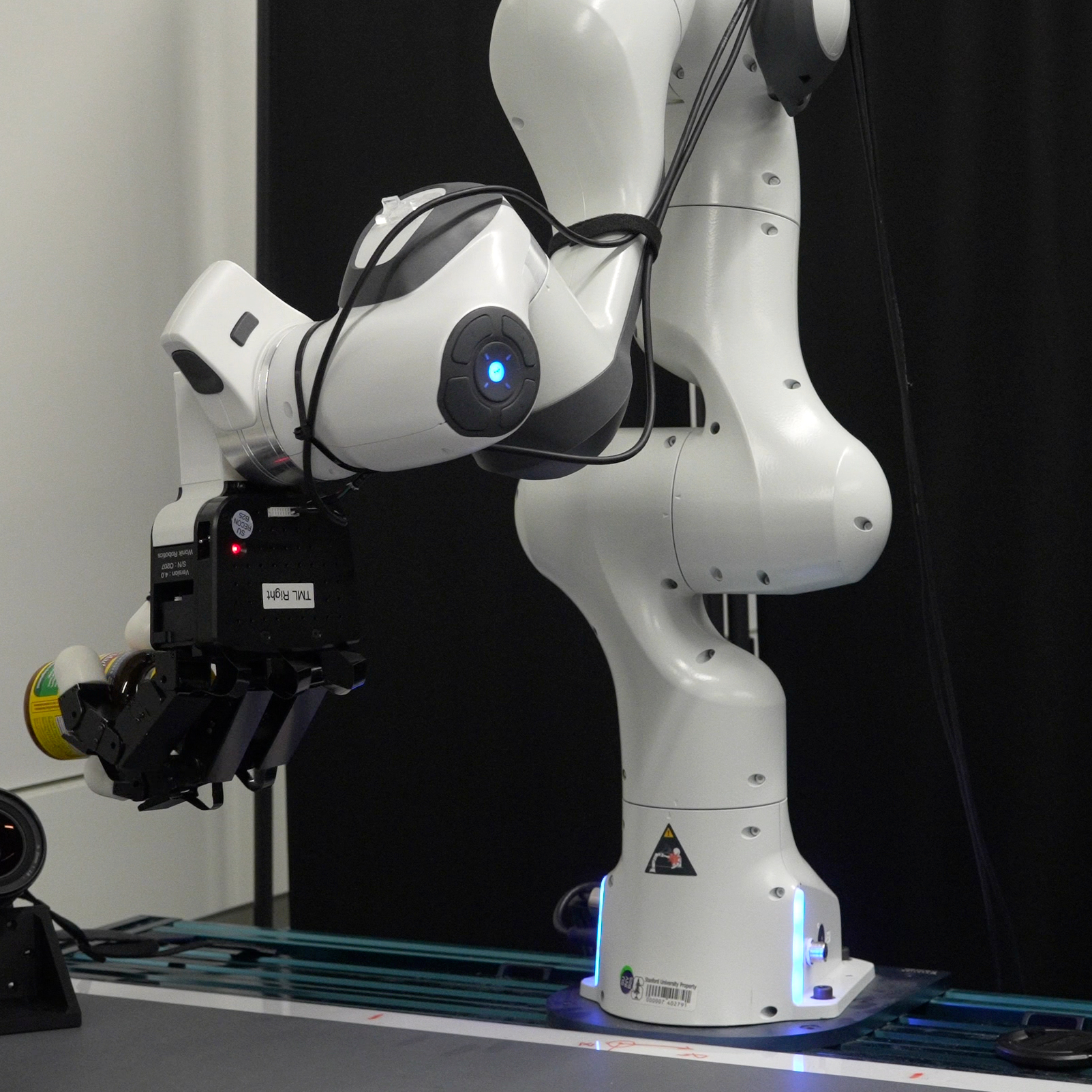}\hspace{2pt}%
        \includegraphics[width=0.14\textwidth]{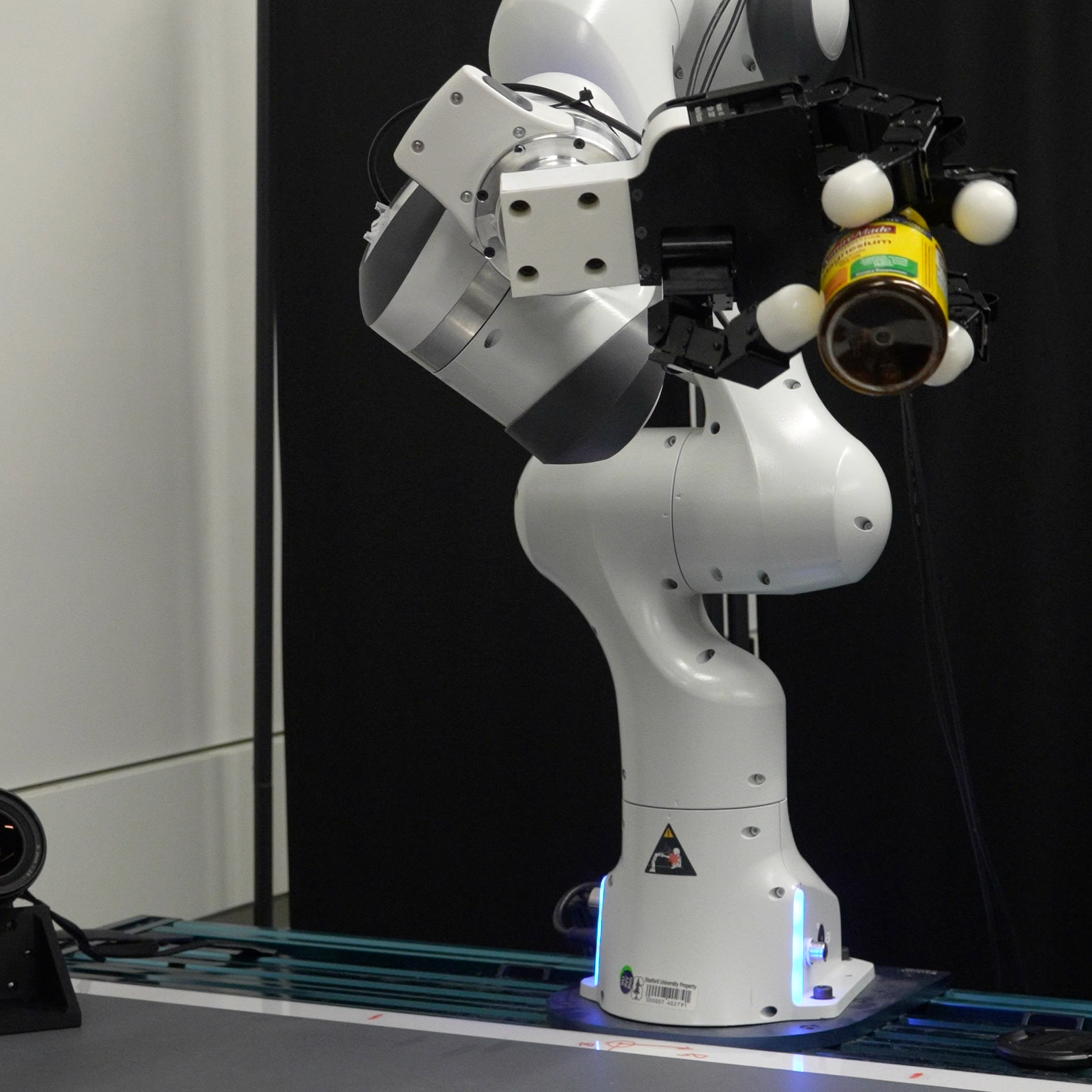}\hspace{2pt}%
        \includegraphics[width=0.14\textwidth]{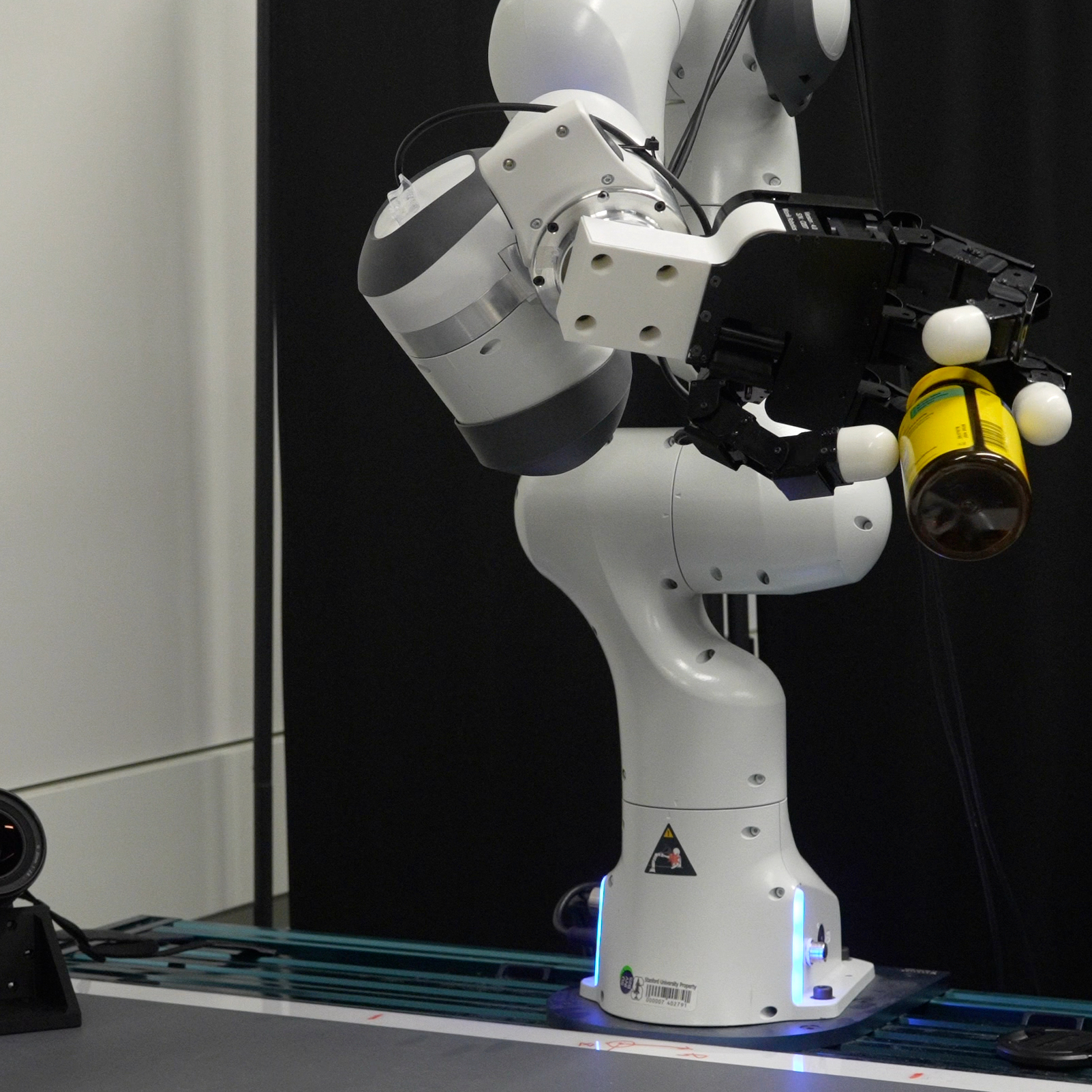}\hspace{2pt}%
        \includegraphics[width=0.14\textwidth]{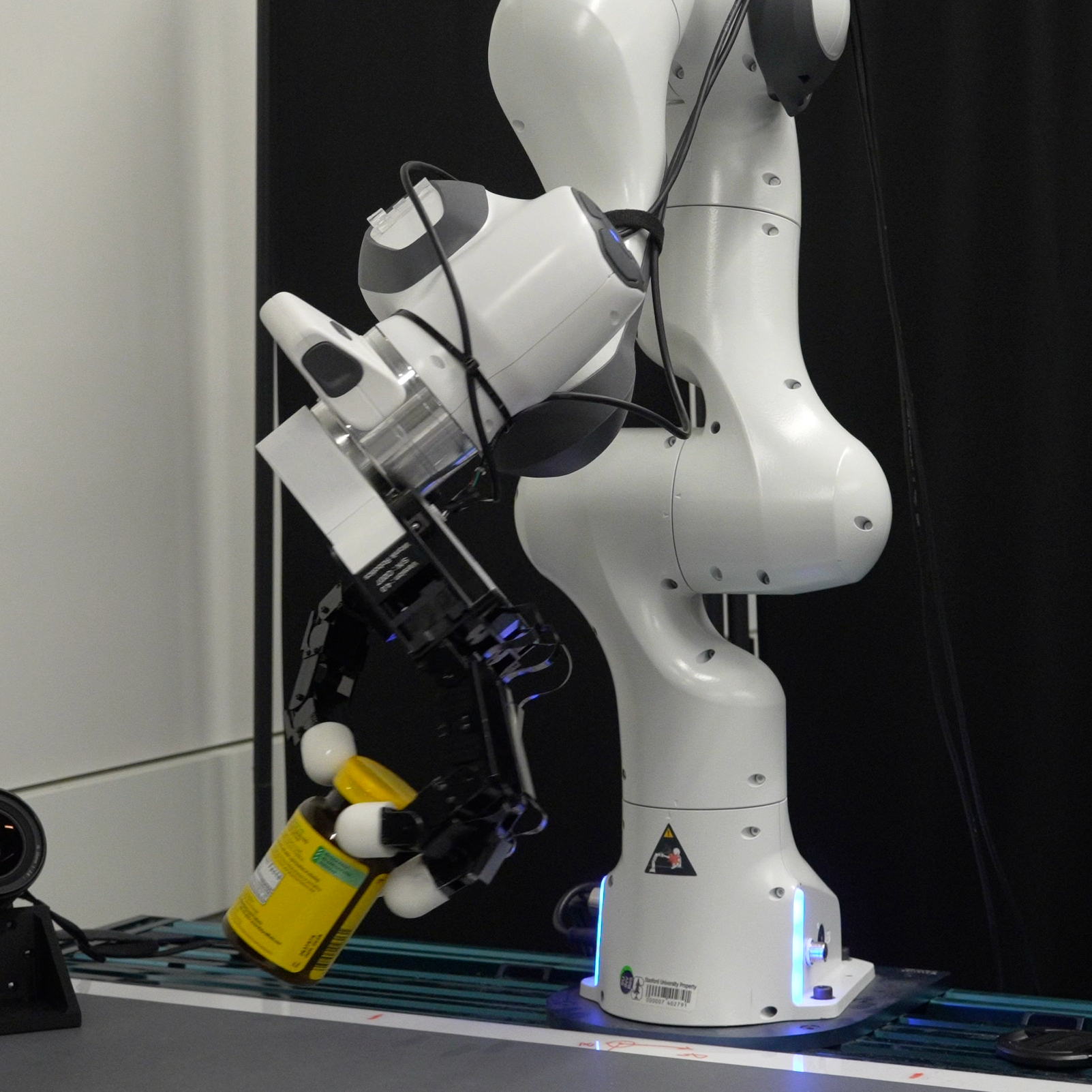}\hspace{2pt}%
        \includegraphics[width=0.14\textwidth]{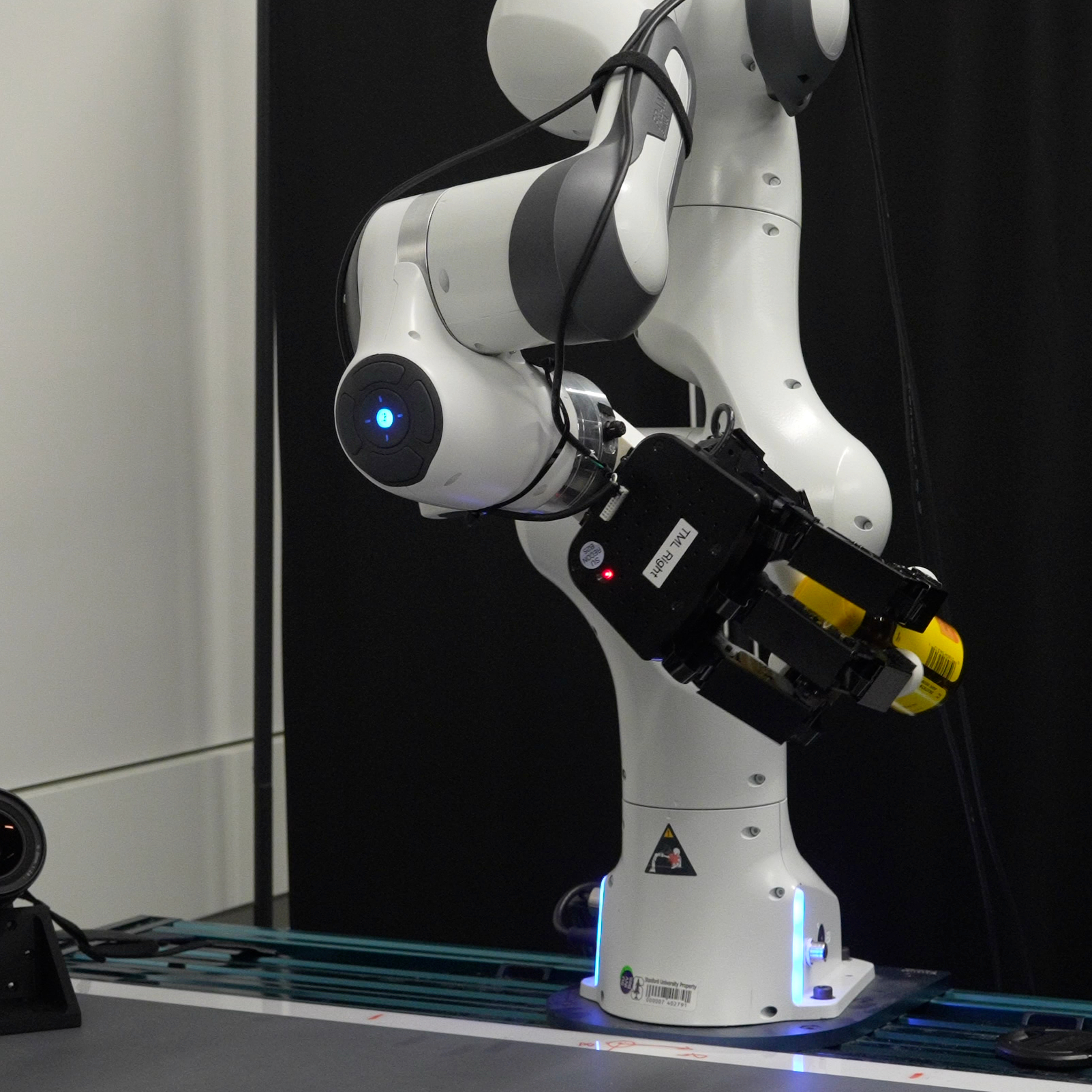}%
        \label{fig:inhand_ccw}
    }
    \caption{In-hand reorientation on hardware. Snapshots are arranged temporally from left to right.(a)~Clockwise rotation with the arm held static and empty bottle. (b)~Counterclockwise rotation with the arm swinging and 2.5\,oz loose weight in the bottle, demonstrating robustness under load and motion.}
    \label{fig:inhand_reorientation}
\end{figure*}

\subsubsection{Scene Setup and Test Protocol}
A human operator places the object in between the finger cage, and the fingers grasp the object using the initial grasping configuration in the motion plan. The joint angles from the planned finger gaiting sequence are then played back open loop. To evaluate robustness of our plans, we test each plan under three conditions: (i)~with the arm held static and no added payload, (ii)~with the arm held static and up to 98 grams of weight incrementally added into the bottle, and (iii)~with the arm swinging and 70 grams of loose weight inside the bottle.
\subsubsection{Results}
Figure~\ref{fig:inhand_reorientation} illustrates examples of the task execution. Using the same open loop motion plan across all three conditions, \methodname{} achieves over $360^\circ$ rotation in both directions. We attribute this robustness to the fact that our IHR plan is certifiably in force closure. This demonstrates that \methodname{} offloads the burden on a higher-level planner (in this case the tree search) and allows a simple strategy to achieve highly constrained and difficult manipulation tasks.

\setcounter{figure}{7}
\begin{figure}[ht]
    \centering
    \subfloat[Clockwise (CW)]{\includegraphics[width=0.48\columnwidth]{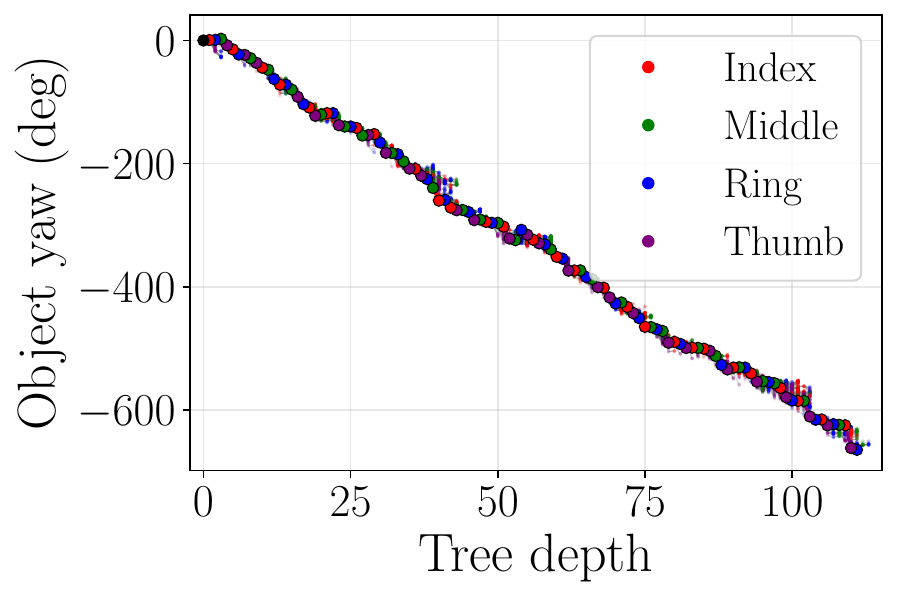}\label{fig:inhand_tree_cw}}
    \hfill
    \subfloat[Counterclockwise (CCW)]{\includegraphics[width=0.48\columnwidth]{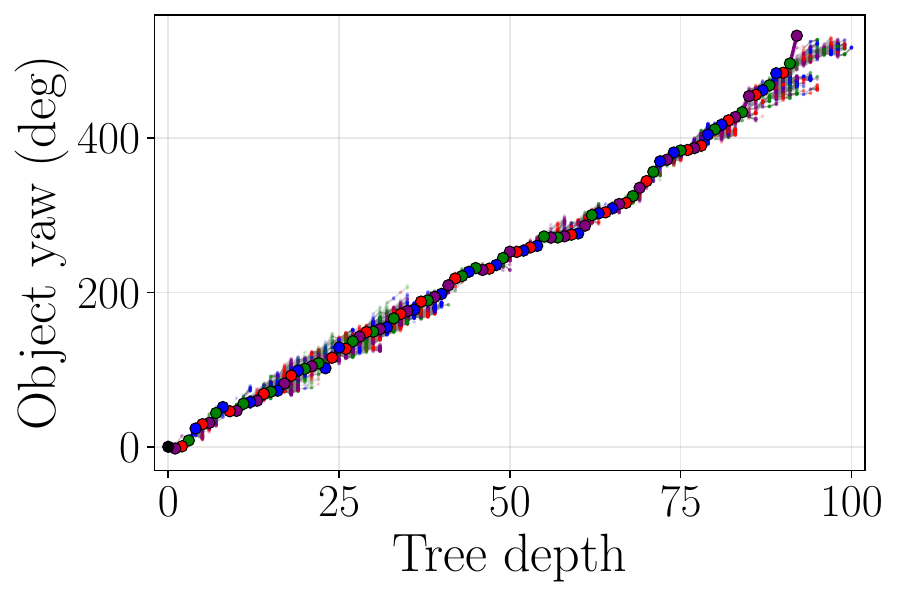}\label{fig:inhand_tree_ccw}}
    \caption{Offline IHR motion plan trees for (a)~clockwise and (b)~counterclockwise target rotations. Each node represents a finger gaiting state, and the highlighted path from the root achieves the largest cumulative yaw.}
    \label{fig:inhand_trees}
\end{figure}
\setcounter{figure}{9}

\section{Discussion}
\label{sec:conclusion}

This paper presents \methodname{} (\methodnamelong{}), a geometric motion generation framework that extends PBDS~\cite{bylard2021composable} with task manifold safety guarantees and an action interface. At each control step, two convex QPs are solved sequentially: the first defines the autonomous safe acceleration, and the second injects the action input around it. The resulting configuration space acceleration is certifiably safe and recovers the autonomous task dynamics when the action input vanishes. We validate its theoretical properties in simulation on an $\mathbb{S}^2$ double integrator and a 7-DOF Franka arm. On a 23-DOF Franka--Allegro hardware system, \methodname{} attains $92.5\%$ grasp success across 20 household objects, and $94.4\%$ for 3-finger grasps through simple action changes enabled by the action interface. Finally, \methodname{} enables the first model-based robust palm-down in-hand reorientation, producing over $360^\circ$ rotation in both directions under varying object weight and arm motion.

\subsection{Limitations}

Like other vector-field policies, \methodname{} solves a local motion generation problem: over long horizons, the system may become trapped in local minima, and escaping them must be handled by the higher-level policy that provides the action input. In addition, \methodname{} is kinematic and therefore requires a lower-level tracker to realize the commanded accelerations, or an impedance controller when compliant contact forces are needed. Finally, the pullback derivations assume surjective submersion task maps and a fully or over-actuated robot; underactuated systems are therefore outside the scope of the present framework.

\subsection{Future work}
On the application side, an important direction is to combine the action interface enabled by \methodname{} with more advanced decision-making policies, such as reinforcement learning or vision-language-action models. On the theoretical side, relaxing the Assumption~\ref{assum:surjective_submersion_task_map} will allow \methodname{} to handle a broader class of objectives, including tasks with singularities or nonsmooth task maps. In addition, incorporating robot dynamics and the contact forces required for manipulation, rather than relying on an acceleration tracker or impedance controller, will allow \methodname{} to reason directly about interaction forces alongside geometric safety.

\printbibliography

\clearpage
\pagenumbering{arabic}
\renewcommand{\thepage}{A\arabic{page}}
\setcounter{page}{1}
\begin{refsection}
        \appendices
        \renewcommand{\thesubsection}{\thesection.\arabic{subsection}}
        \renewcommand{\thesubsectiondis}{\arabic{subsection}}
        \setcounter{figure}{0}
        \setcounter{table}{0}
        \renewcommand{\thefigure}{A\arabic{figure}}
        \renewcommand{\thetable}{A\arabic{table}}
        \section{Proof of Lemma~\ref{lem:task_bcbf_hdot}}
\label{app:proof_task_bcbf_hdot}

\noindent\textbf{Lemma~\ref{lem:task_bcbf_hdot}. }\textit{In the task manifold PBDS setup, the BCBF candidate $h$ given in~\eqref{eq:backstepping_cbf} has time derivative}
\begin{equation*}
    \begin{split}
        \dot{h} &= \langle \operatorname{grad} h_0, \dot{x} \rangle_N + \varepsilon \bigl\langle e_x, \nabla_{\dot{x}} \tilde{\xi} \bigr\rangle_N \\
        &\quad - \varepsilon \bigl\langle e_x, df \left( {}^M\nabla_{\dot{\sigma}}\dot{\sigma} \right) + (\nabla df)(\dot{\sigma}, \dot{\sigma}) \bigr\rangle_N.
    \end{split}
\end{equation*}

\begin{proof}
    Write $e^\alpha = \dot{x}^\alpha - \tilde{\xi}^\alpha$ for the error components. The coordinate expression for the differential of $h$ at $(x,\dot{x}) \in TN$, acting on $(a_x, a_{\dot{x}}) \in T_{(x,\dot{x})}(TN)$, is
    \begin{multline}
        dh \cdot (a_x, a_{\dot{x}})
        = \frac{\partial h_0}{\partial x^\gamma} a_x^\gamma \\
        + \varepsilon\, g_{\alpha\beta}\, e^\beta \!\left( \frac{\partial \tilde{\xi}^\alpha}{\partial x^\gamma} a_x^\gamma - a_{\dot{x}}^\alpha \right)
        - \frac{\varepsilon}{2} \frac{\partial g_{\alpha\beta}}{\partial x^\gamma}\, e^\alpha e^\beta\, a_x^\gamma.
        \label{eq:dh_coordinates_task}
    \end{multline}

    \paragraph*{Step 1: Substitute the dynamics.}
    Along a trajectory, $a_x^\gamma = \dot{x}^\gamma$ and $a_{\dot{x}}^\alpha = \ddot{x}^\alpha$.
    From~\eqref{eq:cov_acc_N}, the coordinate acceleration on $N$ is $\ddot{x}^\alpha = ({}^N\!\nabla_{\dot{x}} \dot{x})^\alpha - {}^N\!\Gamma^\alpha_{\gamma\delta}\, \dot{x}^\gamma \dot{x}^\delta$.
    Using the geometric kinematics~\eqref{eq:geo_kinematics}, we substitute
    \[
        \ddot{x}^\alpha = \bigl(df({}^M\!\nabla_{\dot{\sigma}}\dot{\sigma})\bigr)^\alpha + \bigl((\nabla df)(\dot{\sigma},\dot{\sigma})\bigr)^\alpha - {}^N\!\Gamma^\alpha_{\gamma\delta}\, \dot{x}^\gamma \dot{x}^\delta.
    \]
    The parenthesized term in~\eqref{eq:dh_coordinates_task} becomes
    \begin{multline}
        \frac{\partial \tilde{\xi}^\alpha}{\partial x^\gamma} \dot{x}^\gamma - \ddot{x}^\alpha
        = \frac{\partial \tilde{\xi}^\alpha}{\partial x^\gamma} \dot{x}^\gamma + {}^N\!\Gamma^\alpha_{\gamma\delta}\, \dot{x}^\gamma \dot{x}^\delta \\
        - \bigl(df({}^M\!\nabla_{\dot{\sigma}}\dot{\sigma})\bigr)^\alpha - \bigl((\nabla df)(\dot{\sigma},\dot{\sigma})\bigr)^\alpha,
    \end{multline}
    so~\eqref{eq:dh_coordinates_task} yields
    \begin{multline}
        \dot{h}
        = \frac{\partial h_0}{\partial x^\gamma} \dot{x}^\gamma
        + \varepsilon\, g_{\alpha\beta}\, e^\beta \!\biggl(
        \frac{\partial \tilde{\xi}^\alpha}{\partial x^\gamma} \dot{x}^\gamma + {}^N\!\Gamma^\alpha_{\gamma\delta}\, \dot{x}^\gamma \dot{x}^\delta \\
        - \bigl(df({}^M\!\nabla_{\dot{\sigma}}\dot{\sigma})\bigr)^\alpha - \bigl((\nabla df)(\dot{\sigma},\dot{\sigma})\bigr)^\alpha
        \biggr) \\
        - \frac{\varepsilon}{2} \frac{\partial g_{\alpha\beta}}{\partial x^\gamma}\, e^\alpha e^\beta\, \dot{x}^\gamma.
        \label{eq:hdot_expanded_task}
    \end{multline}

    \paragraph*{Step 2: Decompose $\dot{x} = e + \tilde{\xi}$.}
    Split the second index of ${}^N\!\Gamma^\alpha_{\gamma\delta}\, \dot{x}^\gamma \dot{x}^\delta$ using $\dot{x}^\delta = e^\delta + \tilde{\xi}^\delta$:
    \[
        {}^N\!\Gamma^\alpha_{\gamma\delta}\, \dot{x}^\gamma \dot{x}^\delta
        = {}^N\!\Gamma^\alpha_{\gamma\delta}\, \dot{x}^\gamma \tilde{\xi}^\delta + {}^N\!\Gamma^\alpha_{\gamma\delta}\, \dot{x}^\gamma e^\delta.
    \]
    The first two terms inside the parentheses of~\eqref{eq:hdot_expanded_task} combine with ${}^N\!\Gamma^\alpha_{\gamma\delta}\, \dot{x}^\gamma \tilde{\xi}^\delta$ to form the covariant derivative of $\tilde{\xi}$ in the direction $\dot{x}$:
    \[
        (\nabla_{\dot{x}} \tilde{\xi})^\alpha = \frac{\partial \tilde{\xi}^\alpha}{\partial x^\gamma} \dot{x}^\gamma + {}^N\!\Gamma^\alpha_{\gamma\delta}\, \dot{x}^\gamma \tilde{\xi}^\delta.
    \]
    This leaves a remainder ${}^N\!\Gamma^\alpha_{\gamma\delta}\, \dot{x}^\gamma e^\delta$, so~\eqref{eq:hdot_expanded_task} becomes
    \begin{multline}
        \dot{h}
        = \frac{\partial h_0}{\partial x^\gamma} \dot{x}^\gamma
        + \varepsilon\, g_{\alpha\beta}\, e^\beta \!\biggl(
        (\nabla_{\dot{x}} \tilde{\xi})^\alpha + {}^N\!\Gamma^\alpha_{\gamma\delta}\, \dot{x}^\gamma e^\delta \\
        - \bigl(df({}^M\!\nabla_{\dot{\sigma}}\dot{\sigma})\bigr)^\alpha - \bigl((\nabla df)(\dot{\sigma},\dot{\sigma})\bigr)^\alpha
        \biggr) \\
        - \frac{\varepsilon}{2} \frac{\partial g_{\alpha\beta}}{\partial x^\gamma}\, e^\alpha e^\beta\, \dot{x}^\gamma.
        \label{eq:hdot_with_remainder_task}
    \end{multline}

    \paragraph*{Step 3: Cancellation via metric compatibility.}
    It remains to show that the two terms involving $e^\alpha e^\beta \dot{x}^\gamma$ in~\eqref{eq:hdot_with_remainder_task} cancel.
    By metric compatibility~\cite[Prop.~5.5(c)]{lee2018introduction},
    \[
        \frac{\partial g_{\alpha\beta}}{\partial x^\gamma} = g_{\delta\beta}\, {}^N\!\Gamma^\delta_{\gamma\alpha} + g_{\alpha\delta}\, {}^N\!\Gamma^\delta_{\gamma\beta}.
    \]
    Contracting against $e^\alpha e^\beta \dot{x}^\gamma$ and using the symmetry of $e^\alpha e^\beta$ in $\alpha \leftrightarrow \beta$,
    \[
        \frac{\partial g_{\alpha\beta}}{\partial x^\gamma}\, e^\alpha e^\beta \dot{x}^\gamma
        = 2\, g_{\delta\beta}\, {}^N\!\Gamma^\delta_{\gamma\alpha}\, e^\alpha e^\beta \dot{x}^\gamma,
    \]
    so the $\partial g$ term evaluates to
    \[
        -\frac{\varepsilon}{2} \frac{\partial g_{\alpha\beta}}{\partial x^\gamma}\, e^\alpha e^\beta \dot{x}^\gamma
        = -\varepsilon\, g_{\delta\beta}\, {}^N\!\Gamma^\delta_{\gamma\alpha}\, e^\alpha e^\beta \dot{x}^\gamma.
    \]
    Meanwhile, relabeling $\alpha \leftrightarrow \delta$ in the remainder term from~\eqref{eq:hdot_with_remainder_task} gives
    \[
        \varepsilon\, g_{\alpha\beta}\, {}^N\!\Gamma^\alpha_{\gamma\delta}\, \dot{x}^\gamma e^\delta e^\beta
        = \varepsilon\, g_{\delta\beta}\, {}^N\!\Gamma^\delta_{\gamma\alpha}\, \dot{x}^\gamma e^\alpha e^\beta.
    \]
    These two expressions are equal and opposite, so they cancel.

    \paragraph*{Step 4: Collect.}
    After cancellation, \eqref{eq:hdot_with_remainder_task} reduces to
    \begin{multline*}
        \dot{h}
        = \frac{\partial h_0}{\partial x^\gamma} \dot{x}^\gamma
        + \varepsilon\, g_{\alpha\beta}\, e^\beta \!\biggl(
        (\nabla_{\dot{x}} \tilde{\xi})^\alpha \\
        - \bigl(df({}^M\!\nabla_{\dot{\sigma}}\dot{\sigma})\bigr)^\alpha - \bigl((\nabla df)(\dot{\sigma},\dot{\sigma})\bigr)^\alpha
        \biggr),
    \end{multline*}
    which, upon identifying $\frac{\partial h_0}{\partial x^\gamma} \dot{x}^\gamma = \langle \operatorname{grad} h_0, \dot{x} \rangle_N$ and $g_{\alpha\beta}\, e^\beta\, (\cdot)^\alpha = \langle e_x,\, \cdot\, \rangle_N$, is~\eqref{eq:pbds_bcbf_hdot}.
\end{proof}

\section{Evaluation: Simulated $\mathbb{S}^2$ Double Integrator}
\label{app:s2_experiment_details}
Using a point robot on $\mathbb{S}^2$, we first validate the theoretical properties of \methodname{}, in particular chart invariance, safety guarantees, multi-task composition, and effects of our action space design.

\subsection{System Setup}

\paragraph{System definition.}
Consider a point robot with unit mass traveling on the surface of a unit sphere. The configuration manifold of the robot is $\mathbb{S}^2$, which has a natural embedding $\bar{\varphi}: \mathbb{S}^2\hookrightarrow \mathbb{R}^3$. We consider the atlas $\left\{(U_N, \varphi_N), (U_S, \varphi_S)\right\}$ formed by the north and south pole stereographic projection charts, where:
\begin{itemize}
    \item $U_N = \mathbb{S}^2 \setminus \{(0,0,1)\}$
    \item $\varphi_N : U_N \to \mathbb{R}^2$
    \item $U_S = \mathbb{S}^2 \setminus \{(0,0,-1)\}$
    \item $\varphi_S : U_S \to \mathbb{R}^2$
\end{itemize}
We can define corresponding maps $\bar{\varphi}_N: \mathbb{R}^2\rightarrow \mathbb{R}^3$ and $\bar{\varphi}_S: \mathbb{R}^2\rightarrow \mathbb{R}^3$ from chart coordinates $(y_1, y_2)$ to embedding coordinates $(x_1, x_2, x_3)$:
\begin{equation}
    \bar{\varphi}_N(y_1, y_2) = \left( \frac{2y_1}{\|y\|^2 + 1},\; \frac{2y_2}{\|y\|^2 + 1},\; \frac{\|y\|^2 - 1}{\|y\|^2 + 1} \right),
\end{equation}
where $\|y\|^2 = y_1^2 + y_2^2$. The south pole chart $\varphi_S : U_S \to \mathbb{R}^2$ projects from $S = (0,0,-1)$:
\begin{equation}
    \varphi_S(x_1, x_2, x_3) = \left( \frac{x_1}{1 + x_3},\; \frac{x_2}{1 + x_3} \right),
\end{equation}
with inverse
\begin{equation}
    \bar{\varphi}_S(y_1, y_2) = \left( \frac{2y_1}{\|y\|^2 + 1},\; \frac{2y_2}{\|y\|^2 + 1},\; \frac{1 - \|y\|^2}{\|y\|^2 + 1} \right).
\end{equation}

The metric in both charts is given by $g_{ij} = \frac{4}{(1 + \|y\|^2)^2} \delta_{ij}$, with Christoffel symbols
\begin{equation}
    \Gamma^k_{ij} = \frac{-2}{1 + \|y\|^2} \left( y_i \delta_{jk} + y_j \delta_{ik} - y_k \delta_{ij} \right).
\end{equation}

\paragraph{Configuration Dynamics}
We compare two configuration manifold dynamics, both implemented with fourth-order Runge--Kutta integration schemes:
\begin{enumerate}
    \item \emph{Geometric}: $\ddot{y}^k + \Gamma^k_{ij}\,\dot{y}^i\dot{y}^j = F^k$, using the Christoffel symbols of the round metric.
    \item \emph{Flat}: $\ddot{y}^k = F^k$, ignoring the manifold geometry and treating chart coordinates as~$\mathbb{R}^2$ with $\Gamma = 0$ and $g = \delta_{ij}$.
\end{enumerate}

\paragraph{Task Dynamics}
Beyond the PBDS with embedding task used throughout the main text, the chart invariance experiment in this appendix uses two additional task dynamics variants. We consider three classes of dynamics on~$\mathbb{S}^2$ in total:
\begin{enumerate}
    \item \emph{Direct SMCS} (appendix only): a geodesic attractor potential $\Phi(y) = \tfrac{1}{2}\,d_{\mathbb{S}^2}(y, y_{\mathrm{goal}})^2$ with Riemannian dissipation $\mathcal{F}_D = -b\,g\,\dot{y}$ defines a single SMCS directly on the configuration manifold (no PBDS).
    \item \emph{Trivial PBDS} (appendix only): we define the task map $f = \bar{\varphi}_\alpha\colon \mathbb{R}^2 \to \mathbb{R}^3$ (chart-to-embedding) with \emph{zero} potential and \emph{zero} dissipative force on the embedding. The resulting PBDS task SMCS has no task manifold forces; however, the PBDS closed-form~\eqref{eq:pbds_closed_form} still produces a nonzero coordinate acceleration~$\ddot{\sigma}$ because the Hessian of the task map injects the Christoffel-symbol correction. This isolates the geometric contribution of the task map from any task manifold forcing.
    \item \emph{PBDS with embedding task} (used in main text): following~\cite{bylard2021composable}, we define the task map $f = \bar{\varphi}_\alpha\colon \mathbb{R}^2 \to \mathbb{R}^3$ (chart-to-embedding) and design an attractor potential and dissipative force on the embedding~$\mathbb{R}^3$. The PBDS closed-form~\eqref{eq:pbds_closed_form} yields a coordinate acceleration~$\ddot{\sigma}$.
\end{enumerate}

\paragraph{Safety constraint.}
The robot must stay arclength $r$ from a point $x_o \in \mathbb{S}^2$. Since arclength on the unit sphere equals the angle between embedding vectors, the safety function (Definition~\ref{def:h0}) is
\begin{equation}
    h_0(x) = \arccos(x \cdot x_o) - r, \quad x,\, x_o \in \mathbb{S}^2 \subset \mathbb{R}^3,
    \label{eq:s2_safety_function_app}
\end{equation}
or, in chart coordinates, $h_0(y) = \arccos\!\bigl(\bar{\varphi}_\alpha(y) \cdot x_o\bigr) - r$ with $\alpha \in \{N, S\}$ the active chart. The safe set is $\mathcal{C}_0 = \{x \in \mathbb{S}^2 : h_0(x) \geq 0\}$. We enforce $h_0 \geq 0$ via the ECBF and BCBF formulations of Section~\ref{sec:task_manifold_cbf}; for BCBF, we set $\xi = 0$ and compare two metrics on the chart: the round metric $g_{ij} = \frac{4}{(1+\|y\|^2)^2}\delta_{ij}$ and the flat metric $\delta_{ij}$.

\paragraph{Shared experimental conditions.}
All experiments start from a point on the equator of $\mathbb{S}^2$ (in the overlap of both charts, away from both poles). The chart invariance experiment (Section~\ref{app:scenario_0}) uses nonzero initial velocity so that the Christoffel terms are nonzero from the start. The pullback CBF and action interface experiments (Sections~\ref{app:scenario_1}--\ref{app:scenario_action}) share the same scene: an obstacle of radius $r = 0.5$\,rad centered on the geodesic midpoint between start and goal, with zero initial velocity.

\subsection{Chart Invariance of Geometric vs.\ Flat Configuration Dynamics}
\label{app:scenario_0}

We first demonstrate that on a curved configuration manifold, ignoring the geometry produces chart-dependent and incorrect dynamics, but that the PBDS task map alone recovers the correct geometric behavior even with zero task manifold forces. 

The run configurations for the chart invariance experiments are listed in Table~\ref{tab:s2_chart_invariance_configs}.

\begin{table}[t]
\centering
\caption{Chart invariance run configurations: free motion ($F=0$), nonzero initial velocity. N/S: north/south pole chart.}
\label{tab:s2_chart_invariance_configs}
\small
\setlength{\tabcolsep}{5pt}
\begin{tabular}{@{}llll@{}}
\toprule
Run & Config dynamics & Task dynamics \\
\midrule
(a) & Geometric (N) & Direct SMCS \\
(b) & Flat (N)      & Direct SMCS \\
(c) & Geometric (S) & Direct SMCS \\
(d) & Flat (S)      & Direct SMCS \\
(e) & Flat (N)        & Trivial PBDS \\
(f) & Flat (S)        & Trivial PBDS \\
\bottomrule
\end{tabular}
\end{table}

As shown in Figure~\ref{fig:scenario_0_app}, the geometric trajectories (a) and~(c) follow great circles and agree on $\mathbb{S}^2$ up to numerical precision. The flat trajectories (b) and~(d), which ignore the Christoffel symbols, diverge from each other and from the geodesic. Runs~(e) and~(f) use flat dynamics with the trivial PBDS: zero potential and zero dissipation on the embedding, so the only source of coordinate acceleration is the Hessian of the chart-to-embedding task map.
Despite having flat configuration dynamics, the PBDS task dynamics provide the correct geometric correction, and both (e) and (f) match the geometric runs~(a) and~(c). (e) and (f) demonstrate that the task dynamics governs the PBDS behaviors, with no regard to the configuration manifold dynamics.
\begin{figure}[t]
    \centering
    \includegraphics[width=\columnwidth]{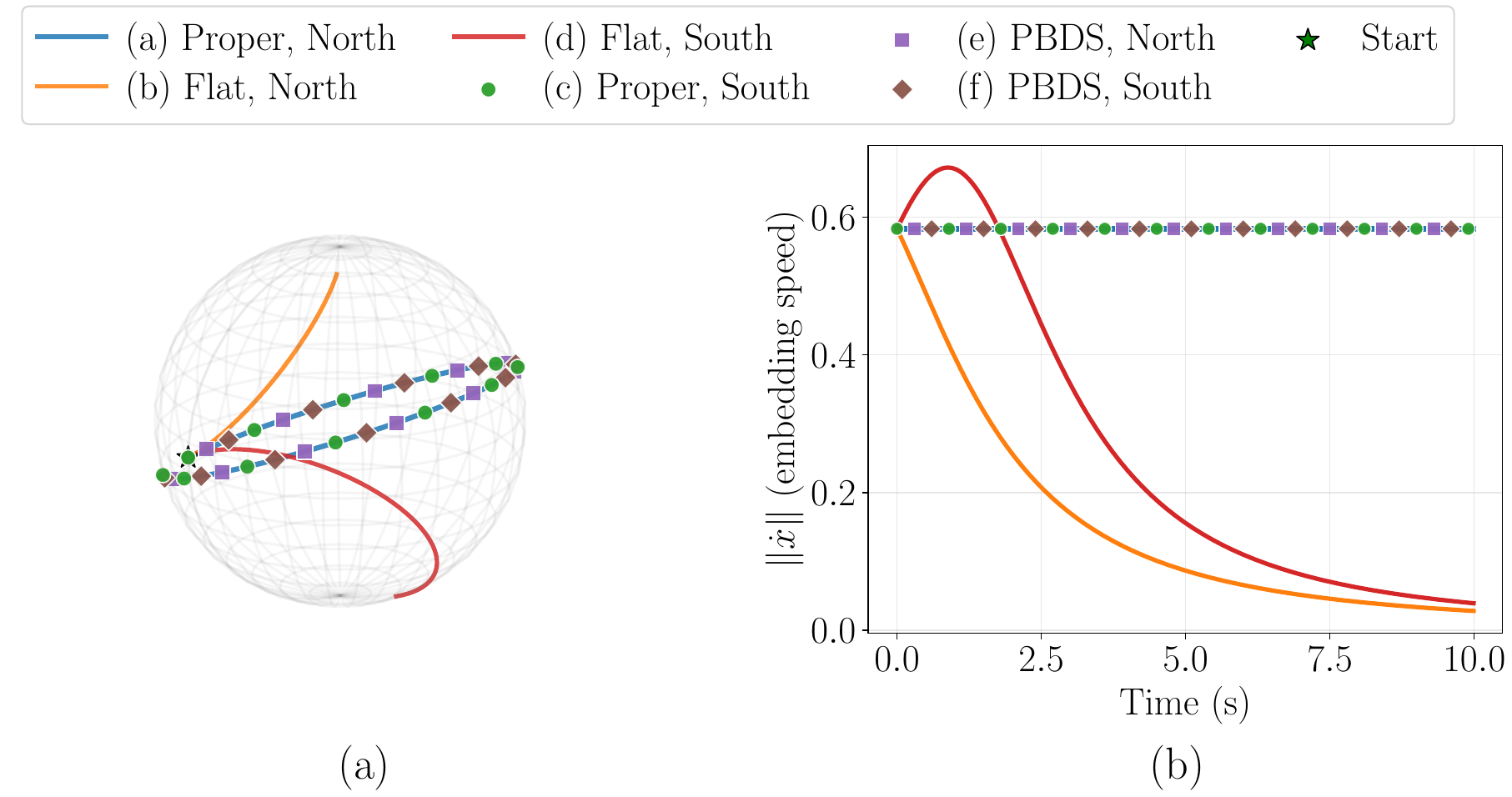}
    \caption{Free motion on $\mathbb{S}^2$. (a)~Geometric (a,\,c) follow great circles and overlap; flat (b,\,d) diverge from the geodesic and from each other. Trivial PBDS (e,\,f) with zero task manifold forces and flat integration overlay the geometric trajectories, showing that the task map's Hessian alone injects the Christoffel correction. Overlapping curves are drawn with staggered marker overlays so each is individually visible. (b)~Geometric and trivial PBDS maintain constant embedding speed (geodesic property); flat does not.}
    \label{fig:scenario_0_app}
\end{figure}

\subsection{Autonomous \methodname{} on $\mathbb{S}^2$}
\label{app:scenario_1}

We validate Autonomous \methodname{} (without task manifold actions) on the $\mathbb{S}^2$ double integrator to test for metric dependence, chart switching, and safety recovery. All experiments use the obstacle avoidance safety function~\eqref{eq:s2_safety_function_app}.

\paragraph{Metric dependence and chart switching.}
The ECBF constraint involves only coordinate partial derivatives of $h_0$ and the coordinate acceleration, with no metric tensors or Christoffel symbols (Remark~\ref{rem:ecbf_metric_independence}), making it metric-independent. The two BCBF runs produce visibly different avoidance trajectories, since the BCBF depends on the metric via $\operatorname{grad} h_0$, $\|e_x\|_N$, and $\nabla_{\dot{x}}\tilde{\xi}$; the BCBF also stays farther from the unsafe set boundary $h_0=0$ due to the required ``safe velocity field.''

Additionally, we verify chart switching by varying the chart used for the configuration space integration. This is the safety-augmented analogue of the chart switching experiment in~\cite[Fig.~5]{bylard2021composable}. The task SMCS on $\mathbb{R}^3$ is unchanged.
Run~(iv) matches the fixed-chart reference~(i) on~$\mathbb{S}^2$, confirming that the PBDS QP solution and safety constraint produce consistent embedding space behavior across chart transitions and that \methodname{} can be deployed on atlases with multiple charts.

\paragraph{Recovery from outside the safe set.}
To test recovery, we run both an ECBF and a BCBF trajectory starting from inside the unsafe set. Both formulations drive the system out of the unsafe region and converge to the goal. However, the ECBF constrains only $h_0 \geq 0$ and recovers along the most direct path, while the BCBF's backstepping barrier $h = h_0 - \tfrac{\varepsilon}{2}\lVert \dot{x} - \tilde{\xi} \rVert_N^2$ additionally penalizes velocity that deviates from the safe velocity field~$\tilde{\xi}$. This velocity-dependent margin causes the BCBF recovery to maintain a larger clearance from the obstacle boundary once $h_0 \geq 0$ is reached, consistent with the greater conservatism of the BCBF observed in the safe-start runs.

\subsection{Steered \methodname{} on $\mathbb{S}^2$}
\label{app:scenario_action}

We demonstrate the action interface (Section~\ref{sec:task_manifold_actions}) on the $\mathbb{S}^2$ double integrator. We use PBDS embedding dynamics with an ECBF safety constraint and an identity control task map $f_l = \mathrm{id}\colon \mathbb{R}^2 \to \mathbb{R}^2$ on the active chart, so the action $u \in \mathbb{R}^2$ is a residual chart-space acceleration. All actions are defined geometrically in the embedding space and mapped to chart coordinates via the Jacobian pseudoinverse at each timestep, making them chart-invariant. We place the obstacle between the start and the goal, so that the autonomous system ($u = 0$) faces two valid avoidance paths on each side of the obstacle. We run five trajectories~(viii)--(xii) (Table~\ref{tab:s2_run_configs}), all shown in Figure~\ref{fig:cbf_combined}(c):
\begin{enumerate}[label=(\roman*), start=8, leftmargin=*]
    \item \emph{Autonomous} ($u = 0$, gray): the system breaks symmetry via a small positional offset and converges to one side of the obstacle.
    \item \emph{$+u_{\perp}$} (blue): a tangential action perpendicular to the start--goal geodesic with $\|u_{\perp}\| = 1$, steering the robot to the same side as the autonomous bias.
    \item \emph{$-u_{\perp}$} (orange): flipping the sign steers the robot to the opposite side, resolving the topological ambiguity that the autonomous dynamics alone cannot.
    \item \emph{$u_{\mathrm{unsafe}}$} (green): a large action ($\|u\| = 10$) continuously pointing toward the obstacle center. \methodname{} prevents the constraint violation, so the actual acceleration deviates from the command near the obstacle boundary. Safety is maintained ($h_0(t) \geq 0$) despite the adversarial action.
    \item \emph{$-u_{\perp}$ on the south pole chart} (dashed purple): to verify chart invariance, we repeat~(x) on the south pole chart. The overlay with the north pole chart run confirms that the steered QP solution is independent of chart choice.
\end{enumerate}

\section{Evaluation: Simulated 7-DOF Robot Arm}
\label{app:7dof_experiment_details}

\subsection{System Setup}

Consider a 7-DOF Franka Emika Panda arm with joint limits, simulated in MuJoCo~\cite{todorov2012mujoco}. In joint coordinates we write $\sigma = (\sigma^1,\ldots,\sigma^7)$, so the configuration manifold
$M = \prod_{i=1}^{7}(\sigma^i_-,\, \sigma^i_+)$ is an open subset of $\mathbb{R}^7$,
equipped with the flat product metric so that ${}^M\!\Gamma = 0$ in joint
coordinates. The behavior comprises three PBDS tasks (items~1--3), two pullback CBF safety tasks (items~4--5), and a action task (item~6):
\begin{enumerate}
    \item An \emph{orientation attractor} task with scalar task map
          $f\colon M \to \mathbb{R}$,
          $f(\sigma) = \lVert \tilde{q}_{\mathrm{ee}}(\sigma) - q_{\mathrm{goal}} \rVert_2$,
          the chord distance in the $\mathbb{R}^4$ quaternion embedding between
          the goal and the end effector unit quaternion, with
          $\tilde{q}_{\mathrm{ee}} = \operatorname{sign}(q_{\mathrm{ee}}^\top q_{\mathrm{goal}})\,q_{\mathrm{ee}}$
          resolving the antipodal $\pm q$ ambiguity in the universal double cover $\mathbb{S}^3 \to \mathrm{SO}(3)$. The task uses the
          quadratic potential $\Phi(d) = d^2$ and the Euclidean metric on the
          scalar output; a companion damping task applies linear dissipation
          with identity map and Euclidean metric on the $\mathbb{R}^4$
          quaternion.

    \item A \emph{position attractor} task (when active) with task map
          $f\colon M \to \mathbb{R}^3$, with a quadratic
          potential attracting toward a goal position, with the Euclidean metric;
          a companion damping task applies linear dissipation with identity map
          and Euclidean metric on $\mathbb{R}^3$.

    \item A \emph{joint space damping} task with $f = \mathrm{id}\colon M \to \mathbb{R}^7$.

    \item \emph{Joint limit ECBF} tasks with task maps
          $f_k\colon M \to \mathbb{R}$, $k = 1,\ldots,7$, given by coordinate
          projections, each with a safety function
          $h_{0,k}$ enforced via the pullback ECBF (Section~\ref{sec:task_manifold_cbf}).

    \item \emph{Workspace obstacle ECBF} tasks (one per collision body) with per-body safety functions
          $h_{0,j}(\sigma) = d_{\mathrm{geom}}(g_j,\, g_{\mathrm{obs}})$,
          where $d_{\mathrm{geom}}$ is the signed distance between a robot link collision body~$g_j$ and the obstacle sphere~$g_{\mathrm{obs}}$, computed via MuJoCo's \texttt{mj\_geomDistance}.

    \item A \emph{steered action} with scalar control task map
          $f\colon M \to \mathbb{R}\colon \sigma \mapsto \sigma^1$, projecting onto
          joint 1, so the action $u \in \mathbb{R}$ is a residual
          joint 1 acceleration.
\end{enumerate}
All experiments use the same shared scenario unless otherwise noted: the arm starts at its home configuration and must reach a goal end effector orientation obtained from the home orientation by an XYZ Euler rotation of $(45^\circ, 60^\circ, 90^\circ)$, with a spherical obstacle (radius $0.10$\,m) placed in the sweep path of links~4--6. Each run proceeds until the system converges (task error and velocity below tolerance) or reaches a maximum duration of $15$\,s ($20$\,s for the run with action input).

\section{Hardware Setup and Task Specification Details}
\label{app:hw_setup}
This appendix expands on the hardware setup (Section~\ref{sec:hardware_setup}) and the manifold/task specification (Section~\ref{sec:dex_hw_manifolds}) used for the dexterous grasping and in-hand reorientation experiments.

\subsection{Control and Perception}
The Franka Panda arm is controlled at 20\,Hz via the Deoxys framework~\cite{zhu2023viola}: the pre-grasp approach uses the Deoxys joint position controller, and during PBDS execution the streamed joint targets are tracked by the Deoxys joint impedance controller with a min-jerk position interpolator. The Allegro hand is commanded over ZMQ with two control modes: joint-level PD control for free-space positioning (e.g., opening to a pre-grasp configuration) and Cartesian fingertip impedance control with gravity compensation for compliant grasping, with force closure regulated upstream by the PBDS controller.

For the grasping task, runtime perception is provided by an Intel RealSense D435 camera. Object segmentation uses the Segment Anything Model (SAM)~\cite{kirillov2023segment}, and 6-DOF pose tracking is provided by FoundationPose~\cite{wen2024foundationpose}, streamed at 10\,Hz to the controller via a Redis pub/sub interface with exponential smoothing to reduce pose jitter. During grasp execution, streaming perception continuously updates the object pose in the MuJoCo~\cite{todorov2012mujoco} simulation at each control step, enabling the PBDS controller to react to object motion.

For object modeling, we scan each object using KIRI Engine~\cite{kiri_engine} with a LiDAR-equipped iPhone 14 Pro.

\subsection{Task Specification}
\label{app:hw_task_spec}
The grasping and finger-gaiting behaviors share the same overall structure but differ in a small number of policy-specific CBFs and in several gain choices, noted inline below.
\begin{enumerate}
    \item A \emph{joint space damping} task with task map
          $f_D = \mathrm{id}\colon M \to \mathbb{R}^m$, equipped with a
          dissipative force $\mathcal{F}_D = -B\,\dot\sigma$ and the flat metric. For grasping, the damping matrix is diagonal with separate entries, $B_{\mathrm{arm}} = 20\,I_7$ and $B_{\mathrm{hand}} = 0.1\,I_{16}$. For in-hand reorientation, the arm is fixed and $B_{\mathrm{hand}} = 2.0\,I_{16}$.

    \item \emph{Joint limit ECBF} tasks with task maps
          $f_k\colon M \to \mathbb{R}$, $k = 1,\ldots,m$, given by the coordinate
          projections $f_k(\sigma) = \sigma^k$, each encoding the upper and lower joint limits. For grasping, the arm joints use softened gains $(\kappa_1, \kappa_2) = (2, 5)$ (their larger ranges and tighter acceleration limits make $(10, 50)$ QP-infeasible), while the hand joints use $(10, 50)$; for in-hand reorientation, all $16$ hand joints use $(10, 50)$.

    \item \emph{Force closure ECBF} tasks with task map
          $f_{\mathrm{fc},S}\colon M \to \mathbb{R}$ that computes the scaled
          $l^*$ metric~\cite{li2023frogger} over a subset $S \subseteq \{1,2,3,4\}$ of fingertip contacts:
          \[
              h_{0,S}(\sigma) = l^*_S(\sigma) \cdot |S| \cdot n_b - l^*_{\min},
          \]
          where $n_b = 3$ is the number of basis wrenches per contact under the hard-contact friction model with coefficient $\mu = 0.3$, $l^*_{\min} = 0.01$, and $(\kappa_1, \kappa_2) = (2, 4)$.
          Both the full-grasp variant $S = \{1,2,3,4\}$ and the four per-finger-excluded variants $S = \{1,2,3,4\} \setminus \{i\}$ are instantiated; the FSM enables whichever is consistent with the currently contacting fingers.

    \item \emph{Fingertip lift-distance ECBF} tasks with task maps
          $f_{\mathrm{lift},i}\colon M \to \mathbb{R}$, $i = 1,\ldots,4$, given by the signed distance of fingertip~$i$ to the object surface. The safety function is $h_{0,i} = d_i - d_{\mathrm{lift}}$ with $d_{\mathrm{lift}} = 10$\,mm, keeping each stationary fingertip above the surface during the transit of another finger. Gains are $(\kappa_1, \kappa_2) = (12, 7)$ for grasping and $(10, 10)$ for in-hand.

    \item \emph{Finger--finger distance ECBF} tasks with task maps
          $f_{ij}\colon M \to \mathbb{R}$, for each pair $(i,j)$ of distinct fingers, given by the Euclidean distance $d_{ij}$ between fingertip centers. Two barriers are defined per pair: a quadratic \emph{spacing} barrier $h_0 = -(d_{ij} - d_{\mathrm{ref},ij})^2 + \epsilon^2$ with $\epsilon = 1.5$\,mm that keeps each pair within a tolerance of its reference spacing $d_{\mathrm{ref},ij}$, and a linear \emph{minimum-separation} barrier $h_0 = d_{ij} - d_{\min}$ with $d_{\min} = 40$\,mm that prevents fingertip collisions. Gains are $(\kappa_1, \kappa_2) = (12, 7)$ for grasping and $(10, 10)$ for in-hand.

    \item \emph{Table avoidance ECBF} tasks with task maps
          $f_{\mathrm{tab},i}\colon M \to \mathbb{R}$, $i = 1,\ldots,4$, given by the world frame height of each fingertip. The safety function is $h_{0,i} = z_i - z_{\mathrm{table}} - p$ with $p = 10$\,mm fingertip padding, enforced with $(\kappa_1, \kappa_2) = (10, 10)$.

    \item \emph{Link--object ECBF} tasks (grasping only) penalizing interference between non-fingertip hand geometry and the object, comprising:
          a per-finger \emph{fingertip penetration} CBF with $h_{0,i} = d_i + 3\,\mathrm{mm}$ (up to $3$\,mm tip ingress tolerated) and gains $(12, 7)$;
          a \emph{palm} CBF with $h_0 = d_{\mathrm{palm}} - 5\,\mathrm{mm}$ and gains $(12, 7)$;
          per-finger \emph{proximal phalanx} CBFs with $h_{0,i} = d_{\mathrm{prox},i}$ and gains $(4, 5)$; and
          per-finger \emph{medial phalanx} CBFs with $h_{0,i} = d_{\mathrm{med},i} + 5\,\mathrm{mm}$ and gains $(12, 7)$. The sign of the additive constant indicates how much ingress is permitted before the CBF activates.

    \item \emph{Distal link--object ECBF} tasks (in-hand only) with task maps
          $f_{\mathrm{dist},i}\colon M \to \mathbb{R}$, $i = 1,\ldots,4$, given by the signed distance between the distal phalanx of finger~$i$ and the object surface, with $h_{0,i} = d_{\mathrm{dist},i} + 20\,\mathrm{mm}$ (allowing up to $20$\,mm ingress) and soft gains $(\kappa_1, \kappa_2) = (1, 2)$. This replaces the per-link CBFs of item~7, which are deactivated for the finger-gaiting policy.

    \item \emph{Fingertip to object distance actions}. For each fingertip we expose a control task map $f_{\mathrm{lift},i}$ carrying a PD controller on the scheduled fingertip-to-object distance target. The grasping policy additionally exposes an \emph{average-fingertip-position} action task (combined arm--hand only) that aligns the centroid of the four fingertips with the object's geometric center, and a \emph{palm position} action task on the $\mathbb{R}^3$ end effector position. When all action inputs are zero, the system reduces to the autonomous PBDS defined by items~1--8.
\end{enumerate}
For in-hand reorientation, the gaited fingertip's position is controlled in the frame of the trifinger plane spanned by the three in-contact fingertips. For each excluded finger $i$, we define a $10$-dimensional \emph{trifinger plane manifold} $P_{\setminus i}$, whose coordinates are the centroid ($3$D) and orientation quaternion ($4$D) of the plane spanned by the three non-lifted fingertips, together with the position of finger~$i$ relative to that plane ($3$D). The gaited-fingertip action task drives the last $3$ coordinates of $P_{\setminus i}$ toward a planner-supplied target via PD control.

Task weights are modulated dynamically by a finite-state machine (FSM) that coordinates the grasp and gaiting phases (approach, contact, lift, traverse, drop). The FSM selects the active force closure variant $f_{\mathrm{fc},S}$ as fingers enter or leave contact, and toggles $f_{\mathrm{lift},i}$ depending on whether finger~$i$ is stationary or in transit.

\section{Dexterous Grasping Trial Protocol}
\label{app:grasping_protocol}
This appendix expands on the hardware grasping experiments in Section~\ref{sec:dexterous_grasping}.
Each 4-finger grasping trial proceeds as follows:
\begin{enumerate}
    \item The stable rest pose of each object is determined and permuted across 3 table locations.
    \item FoundationPose~\cite{wen2024foundationpose} estimates the object's 6-DOF pose via the Redis streaming interface.
    \item Candidate wrist poses are generated by sampling 18 yaw angles (every $20^\circ$) around the vertical axis. For each yaw, a top-down pose (palm facing downward) is computed above the object's bounding-box center. Three standoff distances (0, 7.5, 15\,mm) along the approach direction are evaluated, yielding up to 54 candidates per object.
    \item Each candidate is filtered through a five-stage pipeline:
          (i)~\emph{object width check}: the object extent along the thumb-to-index aperture axis must not exceed $1.5$ times the open-hand span;
          (ii)~\emph{arm inverse kinematics (IK)}: a MuJoCo-based IK solver verifies reachability with a 2\,cm position tolerance, $5^\circ$ rotation tolerance, and a 5\% joint limit margin;
          (iii)~\emph{palm--object collision}: the palm axis-aligned bounding box is checked against the object with a 1\,cm clearance margin;
          (iv)~\emph{hand--object collision}: full hand--object contact detection via MuJoCo rejects poses with penetration; and
          (v)~\emph{aperture verification}: the object center must lie between the thumb and opposing fingers, and in front of the palm.
          Valid candidates are ranked by a quality metric that favors tight width fit and small standoff distance, and the top-ranked candidate is selected for hardware execution.
    \item Using the latest perception update, candidate generation and ranking are repeated at the start of each hardware attempt, so the chosen pose reflects the current object estimate.
    \item On hardware, the arm first moves to its home configuration and the hand opens to a pre-grasp pose (using a two-phase open: fingertip impedance control for coarse motion, then joint PD for fine positioning). The arm then executes a two-stage approach: first moving to a waypoint 15\,cm above the grasp pose, then descending to the final wrist pose via joint-position control through Deoxys~\cite{zhu2023viola}.
    \item The simulation and hardware joint configurations are synchronized, after which \methodname{} takes over the full $23$-DOF arm--hand system in MuJoCo~\cite{todorov2012mujoco}. Two action tasks are composed at each control step:
          (i)~per-finger \emph{fingertip-to-object distance} tasks that PD-servo each fingertip toward the object surface with a clipped acceleration command, and
          (ii)~an \emph{average-fingertip centering} task on the wrist that drives the centroid of the four fingertips toward the object's mesh geometric center, using adaptive gains that taper as the fingertips close in and that freeze the wrist once all four fingertips are within $1$\,cm of the surface.
          The force closure ECBF (Appendix~\ref{app:hw_task_spec}) is enabled once all active fingertips are within $5$\,cm of the object, so that the final squeeze converges to a certified force-closed grasp. The action tasks and ECBFs are resolved by the PBDS QP into a joint-acceleration command that is integrated through inverse dynamics, with the arm portion streamed back to Deoxys and the hand portion streamed to the Allegro. Object pose is continuously updated from the streaming perception pipeline; once the hand is near the object, visual occlusions become severe, so we hold the last reliable pose observation for the remainder of the grasp execution. A trial is recorded as a \emph{success} if the object is stably lifted with all four fingers in contact, a \emph{partial success} if the object is stably lifted but one finger is not in contact with the object surface, and a \emph{failure} otherwise.
\end{enumerate}

\paragraph{3-finger variant.}
For the 3-finger ablation, the active force closure CBF is switched to the per-finger-excluded variant $f_{\mathrm{fc}, S \setminus \{i\}}$ (Appendix~\ref{app:hw_task_spec}, item~3), so that $l^*$ is computed over the three in-contact fingers. The excluded finger is held clear of the object by (i)~retargeting its fingertip-to-object distance action to a fixed $50$\,mm clearance under the same PD law, and (ii)~enabling its fingertip lift-distance CBF (Appendix~\ref{app:hw_task_spec}, item~4) to enforce the $10$\,mm minimum fingertip-to-object clearance. Besides changing the scalar action target and the active force closure variant, no other policy modification is made.

\begin{table}[ht]
    \centering
    \footnotesize
    \renewcommand{\arraystretch}{0.95}
    \caption{4-finger grasping results across 20 household objects. Each object is tested over 6 trials (2 per table location).}
    \label{tab:grasping_results}
    \begin{tabular}{lc@{\hskip 2em}lc}
        \toprule
        \textbf{Object}                      & \textbf{Succ.}            & \textbf{Object} & \textbf{Succ.}         \\
        \midrule
        Ball                                 & 6                         & Cylinder        & 5                      \\
        Bottle                               & 6                         & Flapjack        & 6$^\dagger$            \\
        Camera (large)                       & 6                         & Meat            & 6                      \\
        Camera (small)                       & 6                         & Mustard         & 5$^\dagger$            \\
        Cereal                               & 6                         & Onion           & 5                      \\
        Cracker                              & 6                         & Prism           & 6                      \\
        Chicken                              & 6                         & Ramen           & 5                      \\
        Chips                                & 6$^\dagger$               & Soup            & 4                      \\
        Cocoa                                & 6                         & Wafer           & 6                      \\
        Cube                                 & 6                         & YCB Mustard     & 3$^\dagger$            \\
        \midrule
        \multicolumn{3}{l}{\textbf{Overall}} & \textbf{111/120 (92.5\%)}                                            \\
        \bottomrule
        \multicolumn{4}{l}{\footnotesize $^\dagger$Includes 1 partial success (finger not in contact with object).} \\
    \end{tabular}
\end{table}

\begin{table}[ht]
    \centering
    \footnotesize
    \renewcommand{\arraystretch}{0.95}
    \caption{3-finger grasping results with one finger excluded, tested across 3 table locations per configuration. Folded into Figure~\ref{fig:hw_3finger_examples} in the main text; reproduced here for the per-object breakdown.}
    \label{tab:grasping_3finger}
    \begin{tabular}{lcccc|c}
        \toprule
        \textbf{Object} & \textbf{Thumb} & \textbf{Index} & \textbf{Middle} & \textbf{Ring} & \textbf{Total}       \\
        \midrule
        Cylinder        & 3              & 3              & 3               & 3             & 12                   \\
        Flapjack        & 3              & 3              & 3               & 3             & 12                   \\
        Camera (large)  & 1              & 3              & 3               & 3             & 10                   \\
        \midrule
        \textbf{Total}  & \textbf{7}     & \textbf{9}     & \textbf{9}      & \textbf{9}    & \textbf{34 (94.4\%)} \\
        \bottomrule
    \end{tabular}
\end{table}

\section{In-Hand Reorientation Planner and State Machine}
\label{app:ihr_planner}
This appendix expands on the offline simulation planner and the lift-and-drop state machine described in Section~\ref{sec:in_hand_reorientation}.

\subsection{Depth-First Search Planner}
Each node in the search tree represents a state with all $4$ fingers in contact with the object, and each edge represents moving one finger to a different contact location. The search is rooted at an initial grasp computed by reusing the grasping pipeline of Section~\ref{sec:dexterous_grasping}: we apply the 4-finger grasp computation to a $60$\,mm-diameter, $85$\,mm-tall cylinder, with the hand reoriented so that the palm faces downward over the object. The resulting hand configuration is opened slightly so that each fingertip is $5$\,mm from the object surface.

Among nodes at the same depth, the priority queue orders by cumulative $z$-axis rotation progress,
\begin{equation}
    \Delta\psi_{\text{total}} = \sum_{k=1}^{N_s} \Delta\psi_k,
\end{equation}
where $\Delta\psi_k$ is the incremental yaw change at step $k$ along the path from the root; this accumulation properly handles rotations exceeding $\pm 180^\circ$. To prevent over-reliance on a single finger, we additionally constrain the move counts across fingers to remain balanced,
\begin{equation}
    \max_{i,j} |c_i - c_j| \leq \Delta c_{\max},
\end{equation}
where $c_i$ is the number of times finger $i$ has been moved; empirically $\Delta c_{\max} = 1$ yields the best rotation progress. The search terminates once a maximum node count is reached, returning the path to the node with highest rotation progress.

To prevent the search from getting stuck with a never-movable finger, the planner maintains a cyclic sequence of ``next movable'' fingers, requiring that at each tree depth, the ``next movable'' finger be movable in any accepted candidate's resulting state. The cycle is fixed by hand side and target rotation direction: for the right hand and counterclockwise rotation we use INDEX$\to$MIDDLE$\to$RING$\to$THUMB; for clockwise rotation we use RING$\to$MIDDLE$\to$INDEX$\to$THUMB. At a node of depth $d$, the next movable finger $f^{\text{next}}_d$ is $\textsf{order}[(s_0 + d) \bmod 4]$, where $s_0$ is the index in $\textsf{order}$ of the first movable finger at the root. The specific cyclic order is not load-bearing; any cycle that visits all four fingers serves the same purpose, namely keeping every finger periodically responsible for gating acceptance.

At each node, the planner identifies which fingers may be lifted without violating force closure, and additionally excludes the finger moved at the parent node so that no finger is moved on two consecutive steps. For each remaining movable finger $i$, candidate actions are generated by sampling:
\begin{enumerate}
    \item Step sizes of $\delta \in \{10, 20, 30, 40\}$\,mm on the world XY plane along the object surface, and
    \item Aperture angles $\phi \in \{30^\circ, 50^\circ, 70^\circ, 90^\circ\}$ controlling the finger's approach direction in the plane perpendicular to the object axis.
\end{enumerate}

\subsection{Lift-and-Drop State Machine}
Each candidate action is forward-simulated using \methodname{} and evaluated by a four-phase state machine that governs a single lift-and-drop primitive: \texttt{LIFTING}, \texttt{TRAVERSING}, \texttt{DROPPING}, and \texttt{ADJUSTING}. In \texttt{LIFTING}, the moving finger is pulled clear of the object; once the fingertip has cleared by a small lift margin, the machine advances to \texttt{TRAVERSING}, during which the finger is steered toward the target contact location. Upon aligning with the target in the plane orthogonal to the approach direction, the machine enters \texttt{DROPPING}, which closes the finger onto the object until re-contact is detected, and finally \texttt{ADJUSTING}, which waits for the hand to come to rest. Throughout all four phases, the three in-contact fingers are held in place by the PBDS controller through the force closure ECBF on the three-finger support set, which keeps them force-closed with the object, and pairwise fingertip spacing ECBFs over pairs of in-contact fingers preserve their relative geometry. During \texttt{TRAVERSING} and \texttt{DROPPING}, the force closure ECBF excluding the next movable finger $f^{\text{next}}_d$ (defined above) is additionally pre-enabled, so that the candidate's resulting state retains movability for that finger.

State transitions are triggered as follows:
\begin{itemize}
    \item \texttt{LIFTING} $\rightarrow$ \texttt{TRAVERSING}: the moving finger's fingertip-to-object distance exceeds $0.007$\,m, which is $70\%$ of the nominal $0.01$\,m lift clearance.
    \item \texttt{TRAVERSING} $\rightarrow$ \texttt{DROPPING}: the component of the fingertip tracking error orthogonal to the approach direction falls below $0.01$\,m.
    \item \texttt{DROPPING} $\rightarrow$ \texttt{ADJUSTING}: the re-contacting fingertip-to-object distance drops below the contact threshold of $0.002$\,m.
\end{itemize}

A transition is accepted only if the machine reaches \texttt{ADJUSTING} and, in that state, simultaneously satisfies all of the following: the moving finger is aligned with its target contact (orthogonal tracking error below $0.01$\,m), the hand has come to rest (joint-velocity norm below $0.2$\,rad/s), all four fingers are in contact with the object (each finger satisfying both the fingertip and fingertip-center object distances below $0.002$\,m), and the force closure ECBF excluding $f^{\text{next}}_d$ remains non-negative. We additionally require the object tilt at this resulting state to stay below $\theta_{\text{tilt}}^{\max} = 10^\circ$, evaluated from the object orientation relative to the search root. Any loss of contact at an in-contact finger during the rollout aborts execution and rejects the transition; failure to meet the \texttt{ADJUSTING} success conditions, or a tilt exceeding $\theta_{\text{tilt}}^{\max}$ at the resulting state, likewise rejects the transition.

        \printbibliography[title={Appendix References}]
\end{refsection}

\end{document}